\documentclass{article}
\PassOptionsToPackage{numbers, compress}{natbib}

 \usepackage[main, final]{neurips_2026}
\usepackage[utf8]{inputenc} 
\usepackage[T1]{fontenc}    
\usepackage{hyperref}       
\usepackage{url}            
\usepackage{booktabs}       
\usepackage{amsfonts}       
\usepackage{nicefrac}       
\usepackage{microtype}      
\usepackage{xcolor}      
\usepackage{amsmath,amsfonts,bm}

\def\eqref#1{equation~\ref{#1}}
\def\1{\bm{1}}

\def\vv{{\bm{v}}}

\def\mA{{\bm{A}}}
\def\mB{{\bm{B}}}

\def\mF{{\bm{F}}}

\def\mH{{\bm{H}}}
\def\mI{{\bm{I}}}

\def\mL{{\bm{L}}}
\def\mM{{\bm{M}}}

\def\mP{{\bm{P}}}
\def\mQ{{\bm{Q}}}
\def\mR{{\bm{R}}}

\def\mT{{\bm{T}}}
\def\mU{{\bm{U}}}
\def\mV{{\bm{V}}}
\def\mW{{\bm{W}}}
\def\mX{{\bm{X}}}
\def\mY{{\bm{Y}}}
\def\mZ{{\bm{Z}}}

\DeclareMathAlphabet{\mathsfit}{\encodingdefault}{\sfdefault}{m}{sl}
\SetMathAlphabet{\mathsfit}{bold}{\encodingdefault}{\sfdefault}{bx}{n}

\def\gG{{\mathcal{G}}}

\newcommand{\R}{\mathbb{R}}

\definecolor{first}{HTML}{2CA02C}  % green
\definecolor{second}{HTML}{FF7F0E} % orange
\definecolor{third}{HTML}{FFD700}  % yellow

\definecolor{alizarin}{RGB}{227,38,54}
\definecolor{ultramarine}{RGB}{18,10,143}
\definecolor{Amaranth}{rgb}{0.9, 0.17, 0.31}
\definecolor{tumgreen}{RGB}{162,173,0}
\definecolor{tumblue}{RGB}{0, 101, 189}

\hypersetup{
  linkcolor  = tumblue,
  citecolor  = ultramarine,
  urlcolor   = ultramarine,
  colorlinks = true
}

\newcommand{\ip}[2]{\left\langle#1,#2\right\rangle}
\newcommand{\abs}[1]{\left|#1\right|}
\newcommand{\norm}[1]{\left\|#1\right\|}

\newcommand{\RR}{\mathbb{R}}
\newcommand{\CC}{\mathbb{C}}

\usepackage{hyperref}
\usepackage{url}

\usepackage{titletoc}

\usepackage{bbm}
\usepackage{titlesec}
\usepackage{amsfonts}
\usepackage{hyperref}
\usepackage{mathtools}
\usepackage{tikz}
\usetikzlibrary{matrix}
\usepackage{verbatim}
\usepackage{mathtools}
\usepackage{caption,booktabs}
\usepackage{tabularx,booktabs}
\usepackage{bbm}
\usepackage{wrapfig,lipsum,booktabs}
\usepackage{capt-of}
\usepackage[original]{pict2e}
\usepackage{diagbox}
\usepackage{cancel}
\usepackage{natbib}
\usepackage{amsthm} 
\usepackage{paracol}
\usepackage{multirow}
\usepackage[table]{xcolor}

\newtheorem{theorem}{Theorem}%[section]
\newtheorem{definition}[theorem]{Definition}

\newtheorem{lemma}[theorem]{Lemma}

\newtheorem{example}[theorem]{Example}
\newtheorem{claim}[theorem]{Claim}

\newcounter{RonCounter}

\newcounter{YWCounter}

\newcommand{\fix}[1]{\textcolor{black}{#1}} 
\newcommand{\update}[1]{\textcolor{black}{#1}} 

\title{Beyond Oversquashing: Understanding Signal Propagation in GNNs Via Observables}

\author{
Eden Nagar\\
Technion - Israel Institute of Technology\\
Faculty of Mathematics
\And
Ya-Wei Eileen Lin\\
Technical University of Munich\\ Munich Center for Machine Learning 
\And
Ron Levie\\
Technion - Israel Institute of Technology\\
Faculty of Mathematics
}

\begin{document}

\setlength{\abovedisplayskip}{3pt}
\setlength{\belowdisplayskip}{3pt}
\frenchspacing

\maketitle

\begin{abstract}
Graph Neural Networks (GNNs) perform computations on graphs by routing the  signal between graph regions using a graph shift operator or a message passing scheme. Often, the propagation of the signal leads to a loss of information, where the signal tends to diffuse across the graph instead of being deliberately routed between regions of interest. Two notions that depict this phenomenon are oversmoothing and oversquashing. In this paper, we propose an alternative approach for modeling signal propagation, inspired by quantum mechanics, using the notion of observables. Specifically, we model the place in the graph where the signal lies,  how much the signal is concentrated there, and how much of the signal is propagated towards a location of interest when applying a GNN. Using these new concepts, we prove that standard spectral GNNs have poor signal propagation capabilities. We then propose a new type of spectral GNN, termed \emph{Schrödinger GNN}, which we show has a superior capacity to route the signal across the graph. 
\end{abstract}

\section{Introduction}
\label{sec:intro}

Graph Neural Networks (GNNs) \citep{scarselli2009graph, kipf2017semisupervisedclassificationgraphconvolutional} have emerged as powerful tools, enabling breakthrough applications across diverse domains \citep{wu2020comprehensive, zhou2020graph}, including molecular science \citep{duvenaud2015convolutional, kearnes2016molecular, schutt2017schnet}, physics simulations \citep{battaglia2016interaction, sanchez2020learning}, social network analysis \citep{hamilton2017inductive, perozzi2014deepwalk}, and recommendation systems \citep{wang2019neural, he2020lightgcn}.  
A GNN is a layered architecture that takes a graph with node features, often referred to as the signal, and returns some output, e.g., another signal on the same graph.
The hidden states of the signal across the layers can be interpreted as a gradual flow or propagation of the node features, since the GNN computes the signal at the next layer using local operations on the previous layer.

Often, to solve a problem on graphs, the GNN should be able to direct the propagation of the signal from certain regions of the graph to others. For example, the function of an enzyme is often understood through the notion of allosteric regulation \citep{nerin2024machine, tian2023passer}: activation in one site of the enzyme (the receptor) changes the dynamics of the molecule, leading to some change in another site, called the active site. To be able to predict such a behavior using a GNN, the GNN should be able to propagate the signal about the binding site, which captures structural properties of the receptor, to the distant active site.  
However, one limitation of typical GNNs is that the signal gets diffused in all directions the more layers are used in the network, rather than being propagated, or routed, in a coherent way between regions in the graph.  
This limits the applicability of typical GNNs when a deliberate routing of the signal is required to solve the task. Two standard notions that are commonly regarded as quantifying this phenomenon are \emph{oversmoothing}  \citep{li2018deeper, oono2021graphneuralnetworksexponentially, 
zhao2020pairnorm, 
rong2020dropedgedeepgraphconvolutional, chen2020simpledeepgraphconvolutional, GiraldoEtAl2023Tradeoff} and \emph{oversquashing}  \citep{ AlonYahav2021Bottleneck,topping2022understandingoversquashingbottlenecksgraphs,BlackEtAl2023ER, di2023over, GravinaEtAl2025SWAN, banerjee2022oversquashinggnnslensinformation, sverdlov2024fsw, mishayev2025short}%. 

However, the first notion, oversmoothing, which is often quantified via the Dirichlet energy  \citep{shuman2013emerging, sandryhaila2013discretesignalprocessinggraphs}, describes how quickly the signal varies, or oscillates, across the whole graph, not how much the signal can be kept concentrated, or coherent, when propagating it from one region to another. The second phenomenon, oversquashing,  describes the phenomenon where long range information is compressed through topological bottlenecks. Hence, analyses of oversquashing are typically based on various definitions quantifying bottlenecks, e.g., Cheeger constant, graph curvature,  
 and effective resistance   \citep{topping2022understandingoversquashingbottlenecksgraphs,BlackEtAl2023ER,banerjee2022oversquashing}. 
 In the same papers, oversquashing is quantified via 
 the magnitudes of the partial derivatives of the GNN, at node $m$ with respect to node $n$, for all $m,n$. Namely, large derivatives  indicate that all edges affect all other edges significantly enough. The magnitudes of these derivatives are typically linked to quantities of bottlenecks via special inequalities. Rewiring algorithms are designed to maximize the overall flow along the graph (i.e.,  to minimize topological bottlenecks), and hence, by the special inequalities, to also increase the GNNs derivatives.  

\paragraph{Our contribution.}  
 Here, we identify the following limitation in oversquashing analyses. As noted above, often to solve the task, specific nodes should send their information to other specific nodes. However, by focusing on the graph structure and trying to increase the communication between all pairs of nodes via rewiring, regardless of the task, the signal still diffuses across the graph by the GNN. Increasing all pairwise communications does not promote the ability of the GNN  to single out specific pairs of locations of interest on the graph, allowing a coherent transmission of information between them, while having small GNN gradients for the rest of the nodes.
 
Our goal is to analyze information flow in GNNs beyond oversquashing. We aim to directly study how coherent the signal stays when it is routed between regions of the graph. For this, we propose  an alternative way to model and probe different aspects of the content of the signal and its flow. Specifically, we model (i) the location in the graph where the signal lies, (ii) how much the content of the signal is concentrated about this location, and, (iii) how much of the signal is propagated from one location of the graph to another when applying a GNN. 
 These three concepts are defined via the notion of observables and their mean and variance, similarly to the approach in quantum mechanics.  
 Our focus is not trying to amplify \emph{all} pairwise node communications, but rather to study the ability of GNNs to route the signal from each location to a respective \emph{specific destination}.  Here, a good flow of information is characterized by the ability of the GNN to choose these destinations in view of the task. We note that measuring signal content using observables was also investigated in
 % done in the past in the context of
the signal processing literature  \citep{levie2014adjoint, levie2018uncertaintyprinciplesoptimallysparse, levie2021waveletplanchereltheoryapplication, Halvdansson_2023}.

We then prove that standard spectral GNNs have poor signal propagation capabilities: they keep the location of the content of the signal unchanged, and only increase the spread of the signal about this location. 
 We hence propose a novel spectral GNN, called \emph{Schr\"odinger GNN}, which has provably good signal flow properties. Namely, with \emph{Schr\"odinger filters}, we can direct the propagation of the signal in any desired direction in the graph.   
 
Schr\"odinger GNNs are based on two main components: a unitary graph shift operator (GSO), and complex modulated signals. The unitary GSO is analogous to the  Schr\"odinger operator in classical quantum mechanics, and specifically, in the free particle dynamics. It assures that the content of the signal is transformed in a geometry preserving way, rather than being diffused. Moreover,  Schr\"odinger GNNs consider some of the input feature channels as encoding an abstract notion of ambient location in the graph.  We call these features \emph{formal location}. The rest of the feature channels are called \emph{the signal}. The idea is to be able to shift the signal across the formal location, in any desired direction. For illustration, in a social network, we might want to shift the \emph{income} signal along the \emph{age} direction, to allow comparing salaries of different age groups.  
To quantify the propagation properties of signals, we consider an observable corresponding to each formal location feature, namely, an operator that measures the formal location of signals. Moreover, to guarantee that the formal location of signals shifts when applying GNNs, we form in the signal complex oscillations along the direction of each formal location. 
We show that this leads roughly to a constant speed of the formal location of signals when applying linear Schr\"odinger filters. 

We empirically validate our theory on graph classification and regression benchmarks. 

\section{Measuring signal localization and propagation}

We now introduce the quantities used throughout the paper to measure where signals are localized.

\subsection{General notations}

For $N\in\mathbb{N}$ we denote $[N]=\{1, \ldots,N\}$. 
We treat a vector $a=(a_1,\ldots,a_n)\in\mathbb{C}^N$ as functions $a:[N]\rightarrow\mathbb{C}$, where $a(n)=a_n$ for $n \in [N]$. 
For $z\in \CC$, we denote complex conjugation by $\overline{z}$, and real and imaginary parts by ${\rm Re}(z)$ and ${\rm Im}(z)$.

Consider a graph $\gG=(V,E)$ with the vertex set $V=[N]$ and the edge set $E\subset [N]^2$.
 We denote by $\mathcal{N}(v)$ the neighborhood of vertex $v\in V$. 
We consider only undirected weighted graphs, and denote the adjacency matrix by $A=(a_{n,m})_{n,m}\in\mathbb{R}^{N\times N}$, where $A = A^\top$ and $a_{n,m} = 0$ whenever $\{n,m\}\notin E$.  A \emph{graph-signal} is a pair $(\gG,f)$ where $\gG$ is a graph and $f=(f_k(n))_{k\in[K],n\in[N]}\in\RR^{N\times K}$ is a 2D array called a \emph{signal}, with each $f_k\in\RR^N$. Here, we see $f$ also as a mapping $f:V\rightarrow\mathbb{C}^K$, and see each $f_k$ as a mapping $f_k:V\rightarrow\mathbb{C}$ called a \emph{channel}. By convention, nodes or locations are substituted using brackets $f(n)\in\RR^K$ and channels are written with subscript $f_k\in\RR^N$. Accordingly, the value of channel number $k$ of $f$ at node $n$ is denoted by $f_k(n)$. 
% Note that the raw input node attributes are typically real-valued, but in our framework we modify them to become complex-valued.
% 
A graph shift operator (GSO) is any operator that encodes the graph's structure, e.g., the adjacency matrix or any graph Laplacian.  
The inner product of two  signals  $f, g\in\mathbb{C}^{N\times 1}$ is $\langle f, g\rangle=\sum_{v\in V}f(v)\overline{g(v)}$, and  norm of  $f$ is $\|f\|_2^2=\langle f,f\rangle$.
For a matrix $\mM$, its operator norm is $\|\mM\|_{\mathrm{op}}=\sup_{\|x\|_2=1}\|\mM x\|_2$. For  $f\in\mathbb{C}^{N}$, we denote by $\operatorname{diag}(f)$ the diagonal matrix with diagonal elements $\operatorname{diag}(f)_{n,n}=f_n$.
The commutator of two square matrices  is  $[\mM,\mY]=\mM\mY-\mY\mM$.

\subsection{Observables and the signal routing measure}

In a general Hilbert space $\mathcal{H}$ of signals, an \emph{observable} is a self-adjoint operator $M$ in $\mathcal{H}$, i.e.,  $M^* = M$. By the spectral theorem, any self-adjoint operator in a finite dimensional space can be written as $M = \sum_j \lambda_j P_j$  where $\{\lambda_j\}_j$ are real eigenvalues and $\{P_j\}_j$ are the orthogonal eigenprojections. This decomposition motivates treating a self-adjoint operator as an \emph{observable} of a \emph{physical quantity} \cite{nielsen2010quantum}. Namely, we interpret the eigenvalues as values that the physical quantity can attain, and $P_j$ as projections upon spaces of signals that have $\lambda_j$ as the value of their physical quantity.  For example, the diagonal operator $D:\mathbb{C}^N\rightarrow\mathbb{C}^N$ defined by $(Dg)(j)=jg(j)$ can be thought of as a \emph{location observable} on the discrete line $[N]$. 
Here, the eigenvectors $e_j=(0,\ldots,0,1,0,\ldots,0)$ (with $1$ only at the $j$-th entry) are thought of as pure states (or pure signal) with location exactly $\lambda_j=j$. 
Any signal $g\in \mathbb{C}^N$  is a linear combination of the \emph{pure location states} $\{e_j\}_j$, i.e., $g=\sum_j g(j) e_j$ with $g(j)\in\mathbb{C}$.
When the state $g$ is normalized to $\|g\|_2=1$,  we can interpret $|g(j)|^2$ as  the weight, or probability, of $g$ being at location $j$. 
While $g$  does not have one exact location, we can define its \emph{mean location} as $\mathcal{E}_D(g)=\sum_j |g(j)|^2j$, 
and its location variance as $\mathcal{V}_D(g)=\sum_j |g(j)|^2(j-\mathcal{E}_D(g))^2$.
Using operator notations, these two quantities can be written as $\mathcal{E}_D(g)=\ip{D g}{g}$ and $\mathcal{V}_D(g)=\norm{(D-\mathcal{E}_D(g) I)g}_2^2$, where $I$ is the identity operator in $\mathbb{C}^N$.

This discussion motivates the general construction of observables in quantum mechanics. 
For a self-adjoint operator $M$ and normalized state $g \in \mathcal{H}$, the \emph{expected value} (or \emph{mean}) of $M$ with respect to $g$ is defined to be $\mathcal{E}_M(g) := \langle Mg, g \rangle$. 
Note that when $\mathcal{H}=\mathbb{C}^N$, we have $\mathcal{E}_M(g) = \sum_i \lambda_i \langle P_ig, g \rangle$, which is interpreted similarly to the above example of location observable. 
The \emph{variance} of $M$ with respect to  $g$ is defined to be
\begin{equation}
\mathcal{V}_M(g) := \norm{(M-\mathcal{E}_M(g) I)g}_2^2= \ip{(M-\mathcal{E}_M(g) I)^2g}{g} =\mathcal{E}_{M^2}(g) - \mathcal{E}_M(g)^2.
\end{equation}

In addition to the classical notions of mean and variance, we propose quantifying how well a signal is transmitted towards a target value of the physical quantity. Consider a scenario where we have an initial signal $g^{(0)}$, and we would like to transmit this signal to be concentrated about some value $r\in\RR$ with respect to some observable $M$. For that, suppose that we operate on $g^{(0)}$, e.g., with a GNN, and transform it to $g^{(t)}$.
The following definition quantifies how well $g^{(t)}$ achieves this goal.

\begin{definition}[Signal routing measure]
\label{def:PA}
For an observable $M$, normalized \emph{initial} signal $g^{(0)}$ and \emph{final} signal $g^{(t)}$, and a  target value $r \in \mathbb{R}$, the  \emph{signal routing measure} is defined to be
\begin{equation}
\label{eq:PA}
\mathcal{P}_{M}(g^{(0)},g^{(t)},r) = \frac{\langle(M-Ir)^2g^{(t)},g^{(t)}\rangle}{\mathcal{V}_{M}(g^{(0)})}.  
\end{equation}
\end{definition}
In the setting of Definition \ref{def:PA}, the observable $M$ represents a physical quantity of interest. 
% models some physical quantity.
The term  $\langle(M-Ir)^2g^{(t)},g^{(t)}\rangle$ quantifies how much the values of the physical quantity of  $g^{(t)}$ are concentrated about $r$, and the denominator normalizes this with respect to how well the physical quantity of the initial state $g^{(0)}$ is concentrated.  
It is easy to verify the identity
\begin{equation}
\label{thm:energy_flow_measure}
\mathcal{P}_{M}(g^{(0)},g^{(t)},r)= \frac{\mathcal{V}_{M}(g^{(t)}) + (r - \mathcal{E}_{M}(g^{(t)}))^2}{\mathcal{V}_{M}(g^{(0)})}.
\end{equation}
Hence, to minimize the routing measure, one should construct an operation that transforms $g^{(0)}$ to some $g^{(t)}$, keeping the variance of $g^{(t)}$ small (relatively to the variance of $g^{(0)}$), while making the expected value of $g^{(0)}$ as close as possible to the target value $r$.

\section{Signal propagation in Schr\"odinger graph signal processing}\label{sec:Schrodinger_gsp}

Next, we introduce Schr\"odinger graph signal processing, and analyze signal propagation under it.

\subsection{Feature location observables}

Consider a graph-signal $(\gG,q)$ with $q=(q_1,\ldots,q_M):V\rightarrow\mathbb{C}^M$.
We treat some of the  channels of $q$ as the signal and some as some abstract notion of locations.
Namely, for some $1<J<M$ we call $g=(q_1,\ldots,q_J)$ the \emph{signal}, and call $f=(q_{J+1},\ldots,q_M)$ the \emph{feature locations}. Denote $K=M-J$ and $f=(f_1,\ldots,f_K)$. As we show later, working with complex-valued signals is important for routing signals between graph regions. Hence, we consider $g:V\rightarrow\mathbb{C}^J$ with $\norm{g_j}_2=1$, and consider real-valued feature locations $f:V\rightarrow\mathbb{R}^K$, which need not be normalized. Define the  \emph{feature location observables} $X_{f_k}=\text{diag}(f_k)$, for $k\in[K]$. By the fact that $f_k$ is real-valued, $X_{f_k}$ is self-adjoint. Now, $\mathcal{E}_{X_{f_k}}(g_j)=\sum_{n\in[N]} f_k(n)\abs{g_j(n)}^2$ 
is interpreted as the $f_k$-value about which the energy of $g_j$ is centered, and $\mathcal{V}_{X_{f_k}}(g_j)$ is the spread of the energy of $g_j$ about this center.

\subsection{Partial derivatives and the second order feature derivative GSO}
Our construction 
 uses a special definition of a GSO based on  derivatives. 
\begin{definition}[$f_k$-partial derivative]\label{def:fk_partial}
Given a feature location $f_k: V \to \mathbb{R}$, we define the \emph{$f_k$-partial derivative}  $\nabla_{f_k}\in \mathbb{C}^{N\times N}$ as follows. For $n,m\in V$
\begin{equation}
(\nabla_{f_k})_{n,m}=a_{n,m}({f_k}(n)-{f_k}(m)).
\end{equation}
\end{definition}

Note that $\nabla_{f_k}$ is skew-symmetric (i.e.,  $\nabla_{f_k}^*=-\nabla_{f_k}$), and hence $\nabla_{f_k}^2$ is self-adjoint.
\begin{definition}[Schr\"odinger Laplacian]
Given $K$ feature locations  $f=(f_1,\ldots,f_K)$, the corresponding \emph{Schr\"odinger Laplacian} is defined to be
\begin{equation}
\Delta_f = -\sum_{k\in [K]}\nabla_{f_k}^2.
\end{equation}
\end{definition}
The Schr\"odinger Laplacian is self-adjoint as a sum of bounded self-adjoint operators. This makes the following operator unitary.

\begin{definition}[Schr\"odinger operator]
Given feature locations $f:V\rightarrow\mathbb{R}^K$ and \emph{time} $t\in\mathbb{R}$, the corresponding \emph{Schr\"odinger operator} is defined to be
$\mathcal{S}[t,f] = e^{-it\Delta_f}$.
\end{definition}

As we define later in the section, \emph{Schr\"odinger graph signal processing} is based on filtering signals using Schr\"odinger operators as GSOs.
In this paper, we develop the theory for Schr\"odinger operators based on Schr\"odinger Laplacians, as these special GSOs lead to theoretical guarantees. However, the Schr\"odinger signal processing methodology can be extended to be based on general GSOs, like standard graph Laplacians, though with less theoretical guarantees.

Let us draw an analogy to the classical theory. In the free particle Schr\"odinger equation, we consider the space $\mathbb{R}^3$ as the ``graph,'' consider the coordinates $x,y,z$ as the locations, and $\partial_x,\partial_y,\partial_z$ as the partial derivatives. Here, $\Delta_{x,y,z}=-\partial_x^2-\partial_y^2-\partial_z^2$ is the classical Laplace operator. Given a wave function $g^{(0)}:\mathbb{R}^3\rightarrow\mathbb{C}$ representing a particle at time $0$, $g^{(t)}=\mathcal{S}[t,(x,y,z)]g^{(0)}$ is the particle at time $t$. 
In our case, given a signal $g^{(0)}$ on the graph, thought of as the state at time 0, we denote $g^{(t)}=\mathcal{S}[t,f]g^{(0)}$, and interpret $g^{(t)}$  as the signal at time $t$.

\subsection{Analyzing signal propagation via splitting}

Note that typical signals  are not  localized about one feature location. For example, the grayscale signal of an image is typically supported across all $x,y$ locations. Hence, the expected location and location variance are not meaningful localization notions for such signals (see Figure~\ref{fig:subsignals} for illustration). Still, we can  conceptually apply a localization analysis with observables as follows. We decompose the signal $g$ into a sum of chunks $g=g^1,\ldots,g^L$, e.g., by multiplying the signal by windows in the formal locations $g^l=w^l(f_k)g$, where $w^1,\ldots,w^L:\mathbb{R}\rightarrow\mathbb{R}$ form a partition of unity. Here we assume that each $w^l$ is well localized about one location value. Then, each chunk  $g^l$ has a meaningful mean location, and we can track how Schr\"odinger operators propagate this location. Moreover, by tracking how much the Schr\"odinger operator increases the variance of the chunk, we interpret how much the content of the signal in this chunk is diffused, scatters, or dispersed, when propagating it. Note that this analysis makes sense by the linearity of the Schr\"odinger operator. In our methodology, we do not decompose $g$ to chunks when applying a GNN. Rather, this decomposition is  for conceptualizing the signal propagation, and also for diagnosis of trained GNNs (see Section \ref{sec:experiments}).

\begin{figure}[h]
\centering
\includegraphics[width=1\textwidth]{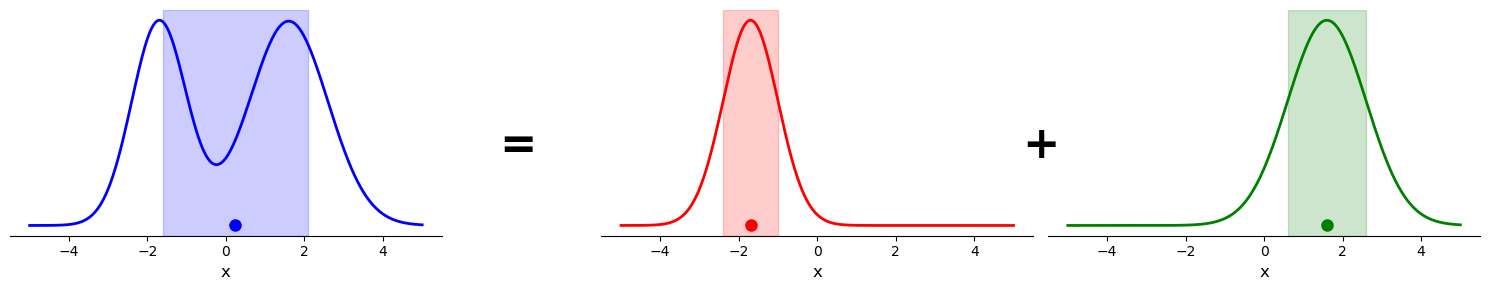}
\caption{Windowed localization of a graph signal. 
Decomposition of a signal $g$ to $g^0+g^1$.
The horizontal axis is the real feature-location coordinate.
Expected feature locations are marked by a dot, and the variance is signified by a color band.}
\label{fig:subsignals}
\end{figure}

\subsection{Dynamics of 1D signals via feature momentum}
\label{Dynamics of 1D Signals via Feature Momentum}

In the classical theory, the partial derivatives are called the \emph{momentum observables}, and the mean $\mathcal{E}_{i\partial_x}(g)$ is interpreted as the expected momentum, or speed, of  particle $g$. Analogously, we interpret the $f_k$-partial derivative $i\nabla_{f_k}$ as observables of momentum or velocity along $f_k$. This interpretation is made precise by developing dynamical equations of signals under the Schrödinger operator. 
In the following discussion, we focus on a single-channel signal $g=g_1$ and a single feature location $f=f_1$. We first show that the expected momentum of a signal is constant under Schr\"odinger dynamics.

\begin{theorem}[Constant expected momentum]
\label{thm:expected_momentum_invariant}
Let $g^{(0)}:V\rightarrow\mathbb{C}$ be a normalized signal and $f:V\rightarrow\mathbb{R}$ a feature location. Then, for every $t\in\mathbb{R}$, 
\begin{equation}
\mathcal{E}_{i\nabla_{f}}(g^{(t)})=\mathcal{E}_{i\nabla_{f}}(g^{(0)}).
\end{equation}
\end{theorem}

We then show that the rate of change of the expected location is equal to some smoothed version of the expected momentum. For that, we first define smoothing with respect to feature directions.

\begin{definition}[$f$-smoothing operator]\label{def:f_smoothing}
Let $f$ be a feature location. The \emph{$f$-smoothing operator} $W_f:\mathbb{C}^N\rightarrow\mathbb{C}^N$ is defined for every signal $g\in\mathbb{C}^N$ and vertex $v\in V$ by
\begin{equation}
(W_{f} g)(v)= \sum_{w \in\mathcal{N}(v)} a_{v,w} (f(w) - f(v))^2g(w).
\end{equation}
\end{definition}

By definition, the $f$-smoothing operator mixes the values of the signal $g$ only along edges where the feature $f$ changes. It is hence interpreted as smoothing along the $f$ direction.

\begin{theorem}
[Expected feature location derivative under Schr\"odinger dynamics]
\label{thm:derivative_expected_feature}
Let $g^{(0)}:V\rightarrow\mathbb{C}$ be a normalized signal and $f:V\rightarrow\mathbb{R}$ a feature location. Let $g^{(t)}=\mathcal{S}[t,f]g^{(0)}$ for every $t\in\mathbb{R}$. Then, 
\begin{equation}
\label{eq:dynamics_smooth}
\frac{\partial}{\partial t}\mathcal{E}_{X_f}(g^{(t)})=2 \mathrm{Re}\left(\langle i\nabla_f g^{(t)}, W_f g^{(t)}\rangle\right).  
\end{equation}

\end{theorem}

The right-hand-side of Eq.~(\ref{eq:dynamics_smooth}) is interpreted as a smoothed version of the expected momentum $\mathcal{E}_{i\nabla_f}(g^{(t)})=\ip{i\nabla_f  g^{(t)}}{g^{(t)}}$. Hence, Theorem \ref{thm:derivative_expected_feature} states that the rate of change of the expected location is equal to a smoothed expected momentum. In Appendix~\ref{app:proof_of_thm:derivative_expected_feature}, we show that for smooth enough signals, the rate of change of the expected location is close to the exact expected momentum. Since the expected momentum is constant, the theorem  suggests that the rate of change of the expected location is roughly constant, as long as the signal stays smooth enough. This analysis hence justifies calling $i\nabla_f$ the momentum, or velocity, observable.

Theorem \ref{thm:derivative_expected_feature} is analogous to the classical case, where the rate of change of the expected location of a free particle is equal to its expected momentum, which is constant. See Appendix~\ref{sec:Schrodinger_classic} for more details. 
An equation for the dynamics of the variance is given in Appendix \ref{Dynamics of the Variance}.

\subsection{Achieving translations via feature modulation}
\label{Achieving Translations via Feature Modulation}

We wish to be able to translate the expected feature location of signals using Schr\"odinger operators. In typical graph data, all features are real. However, as we show next, for real value signals, the expected momentum is always zero.  

\begin{example}
    Let $g$ be a real-valued signal and $f$ a (real) feature location. The corresponding expected momentum is $\mathcal{E}_{i\nabla_{f}}(g) =\ip{i\nabla_{f}g}{g} = i\ip{\nabla_{f}g}{g}$. Since $f$ and $g$ are real valued, so is $\ip{\nabla_{f}g}{g}$. Hence, $\mathcal{E}_{i\nabla_{f}}(g)$ is imaginary. However, it must also be real as an expected value of a self-adjoint observable. Indeed, $\ip{i\nabla_{f}g}{g}=\ip{g}{i\nabla_{f}g}=\overline{\ip{i\nabla_{f}g}{g}}$. Hence, we must have  $\mathcal{E}_{i\nabla_{f}}(g)=0$. 
\end{example}

Hence, given a real-valued signal, to be able to route it between feature regions, we must first modify it to be complex-valued. We do this via the feature modulation operator.
\begin{definition}
[Feature modulation]
Given a feature location $h:V\rightarrow\mathbb{R}$ and  a  \emph{phase} $\theta\in\mathbb{R}$, the corresponding \emph{feature modulation operator} is defined to be $D[\theta h]:={\rm diag}(e^{i\theta h})$, where $e^{i\theta h}$ is the vector with entry $(e^{i\theta h})(v)=e^{i\theta h(v)}$ for each node $v\in V$.
\end{definition}
Next, we show that modulating a real-valued signal gives nonzero expected momentum in general.
\begin{theorem}[Expected momentum of modulated signal]
\label{thm:non_zero_expected_momentum_modulated_signals}
Given a normalized signal $g:V\rightarrow\mathbb{R}$,  feature locations $f,h:V\rightarrow\mathbb{R}$, and a  phase $\theta\in\mathbb{R}$, the expected momentum of  $D[\theta h]g$  satisfies
\begin{equation}
\label{eq:mod_exp}
\mathcal{E}_{i\nabla_f}(D[\theta h]g) = \sum_{(m,n) \in E} a_{m,n} g(m) g(n) (f(n) - f(m)) \sin(\theta(h(n) - h(m))).
\end{equation}
\end{theorem}
Theorem \ref{thm:non_zero_expected_momentum_modulated_signals} can be interpreted as follows. Consider the edge signals $e_{g,h},e_f:E\rightarrow\mathbb{R}$ defined by
\begin{equation}
e_{g,h}(v,w)=g(v) g(w) \sin\big(\theta(h(w) - h(v))\big) , \quad  e_f(v,w)=f(v) - f(w).
\end{equation}
The right-hand-side of Eq.~(\ref{eq:mod_exp}) is  the  edge-space inner product $\ip{e_{g,h}}{e_f}$.
Hence, as long as we choose a modulating feature $h$ such that $e_{g,h}$ and $e_f$ are not orthogonal, the expected momentum of $D[\theta h]g$ will be nonzero.

\subsection{Dynamics of multi-channel signals and observables}
\label{Dynamics of Multi-Channel Signals and Observables}
We next derive the dynamics under the Schrödinger operator for multidimensional locations. 
\begin{theorem}[Expected multi-feature derivative]
\label{thm:derivative_expected_multi_feature}
Given the Schr\"odinger Laplacian $\Delta_f = -\sum_{j \in [K]}\nabla_{f_j}^2$, a normalized signal $g^{(0)}$, and $g^{(t)}=\mathcal{S}[t,f]g^{(0)}$, for every $k\in[K]$ we have
\begin{equation}
    \label{eq:EmultiFeature}
    \frac{\partial}{\partial t}\mathcal{E}_{X_{f_k}}(g^{(t)})=  
     2 \mathrm{Re}\left(\langle i\nabla_{f_k} g^{(t)}, W_{f_k} g^{(t)}\rangle\right) -\sum_{j\neq k} \ip{[i\nabla_{f_j}^2,X_{f_k}]g^{(t)}}{g^{(t)}}.
\end{equation}
\end{theorem}

Ideally, we would like the rate of change of the expected $X_{f_k}$ location to be a smoothed version of the expected $i\nabla_{f_k}$ momentum. However, we see that in Eq.~(\ref{eq:EmultiFeature}) there are additional cross terms. This leads to the following definition.

\begin{definition}[$\epsilon$-commuting features]
\label{def:epsilon-commuting features}
The feature locations $(f_1, f_2, \ldots, f_K)$ are called \emph{$\epsilon$-commuting} if for every pair $i\neq j\in [K]$, the matrix $E=([\nabla_{f_j}^2,X_{f_i}])_{i,j\in[K]}$ 
satisfies 
 $\|E\|_{\rm op} \leq \epsilon$.
\end{definition}

For a sequence of $\epsilon$ commuting features, the Schr\"odinger dynamics satisfies
\begin{equation}
\abs{\frac{\partial}{\partial t}\mathcal{E}_{X_{f_k}}(g^{(t)}) - 2 \mathrm{Re}\big(\langle i\nabla_f g^{(t)}, W_f g^{(t)}\rangle\big)} \leq (K-1)\epsilon.
\end{equation}
Hence, here as well we have the interpretation that for smooth enough signals, the rate of change of all expected locations is close to their corresponding expected momenta. see Theorem~\ref{thm:deviation_bounds_multi_feature} for details.

\subsection{Orthogonalizing the feature directions}

The signal $q:V\rightarrow\mathbb{R}^M$ in the raw data is not $\epsilon$-commuting in general. Hence, in Schr\"odinger GNNs, as a first step, we transform the feature $q$ to a sequence of features $f_1,\ldots,f_K$ which are $\epsilon$-commuting. For example, one can plug each node feature $q(n)$ into a simple MLP  or a linear transformation $\Theta$, to obtain $f(n)=\Theta(q(n))$. The transformation $\Theta$ is optimized with respect to the following target. 

\begin{definition}[Position-momentum optimization (PMO)]
\label{def:minimization}
Given a signal $q\in\mathbb{R}^{N\times M}$ and $\lambda>0$,  the  linear transformation $T \in \mathbb{R}^{M \times K}$,  mapping  $q$ to  $f = q T\in\mathbb{R}^{N\times K}$, is optimized with respect to
\begin{equation}
\min_{T\in \mathbb{R}^{M \times K}} \sum_{i \neq j}^K \|[\nabla_{f_j}^2, X_{f_i}]\|_{\rm op}^2 + \lambda\sum_{k=1}^K \left( \|\nabla_{f_k}\|_\infty-1\right)^2.
\end{equation}

\end{definition}

\subsection{Improving signal routing through modulation}
\label{Improving Signal Routing Through Modulation}

In typical situations, if the goal is to transmit the real-valued signal $g^{(0)}$ towards a feature location $r$, modulating the signal helps. Namely, given feature locations $h,f$ and time $t$, there is a nonzero choice of the modulation parameter $\theta$ such that the signal routing measure $\mathcal{P}_{X_f}(g^{(0)},\mathcal{S}[t,f]D[\theta h]g^{(0)},r)$ is lower than $\mathcal{P}_{X_f}(g^{(0)},\mathcal{S}[t,f]g^{(0)},r)$. We propose an analysis supporting this in Appendix \ref{Improving Signal Routing Through Modulation A}. Hence, since we assume that in order to solve the task there are special pairs of graph regions than need to communicate, modulations are an important ingredient that promote such long term interactions when propagating the signal via Schr\"odinger operators.

\begin{figure}
    \centering
    \includegraphics[width=1\linewidth]{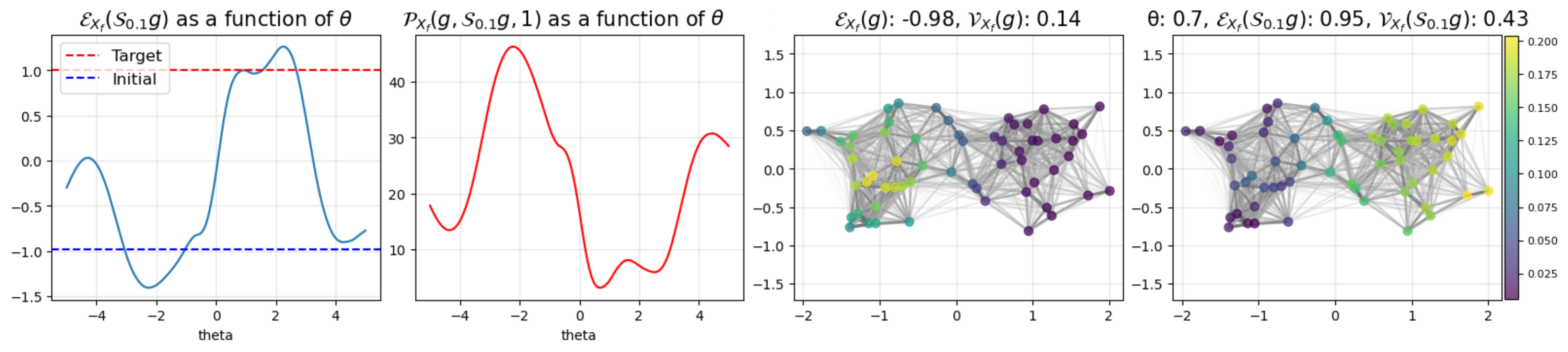}
    \caption{Signal transport under modulation for the phase-induced signal transport.
    The routing measure is represented as a function of modulation phase, and the expected feature location as a function of modulation phase.
    We next plot the initial node energy and the propagated energy after applying the selected phase. 
    }
    \label{fig:theta_transport}
\end{figure}

In Figure \ref{fig:theta_transport} we give an example of a graph, initial signal $g^{(0)}$ with $\mathcal{E}_{X_f}(g)=-0.98$,   modulating feature $h=f$, and desired location value $r=1$. We show that  choosing an appropriate modulation $\theta$ and propagating the signal using the Schr\"odinger operator to time $t=0.1$ improves the signal routing measure with respect to not modulating.
For further details of this example, refer to Appendix~\ref{app: Optimizing signal transport via modulation}.

\subsection{Schr\"odinger signal processing and GNNs}
\paragraph{Schr\"odinger signal processing.}
We define Schr\"odinger filters by considering linear combinations of the evolutions of the modulated signal with different modulations and times.  
Let $f:V\rightarrow\mathbb{R}^K$ be location features and $D\in\mathbb{N}$ be the output feature dimension. To use linear algebra notations, let us now treat signals and location features as vectors in $\mathbb{C}^{N\times J}$ and $\mathbb{R}^{N\times K}$ respectively. A Schr\"odinger filter $\Psi$ is parameterized by 
$(t_{m}\in\mathbb{R}, \theta_{m}\in\mathbb{R},\mathbf{W}^{(m)}\in\mathbb{C}^{J\times D},\mathbf{T}^{(m)}\in\mathbb{R}^{K\times 1})_{m\in [M]}$, and maps signals $g\in \mathbb{C}^{N\times J}$ to 
\begin{equation}
\Psi(g)= \sum_{m=1}^M\mathcal{S}[t_{m},f] D[\theta_{m} f \cdot \mathbf{T}^{(m)}] g \cdot \mathbf{W}^{(m)}.
\end{equation}

\paragraph{Schr\"odinger GNNs.}
The application of a Schr\"odinger GNN is a two-step procedure. First, the input signal is optimized via Position-Momentum Optimization (PMO) (Definition \ref{def:minimization}) to obtain the location features $f$. Then a layered architecture with Schr\"odinger filters and activation functions is applied on the locations features and the input signal. For nonlinear activations, we apply standard activations (e.g., ReLU) separately to the real and imaginary parts: $\sigma(z) = \mathrm{ReLU}(\mathrm{Re}(z)) + i \cdot \mathrm{ReLU}(\mathrm{Im}(z))$ or we use the absolute value $\sigma(z) =|z|$. 
See Appendix \ref{subsec:arch_details} for implementation details and for computational complexity analysis.

\begin{figure}[t]
\centering
\includegraphics[width=1\linewidth]{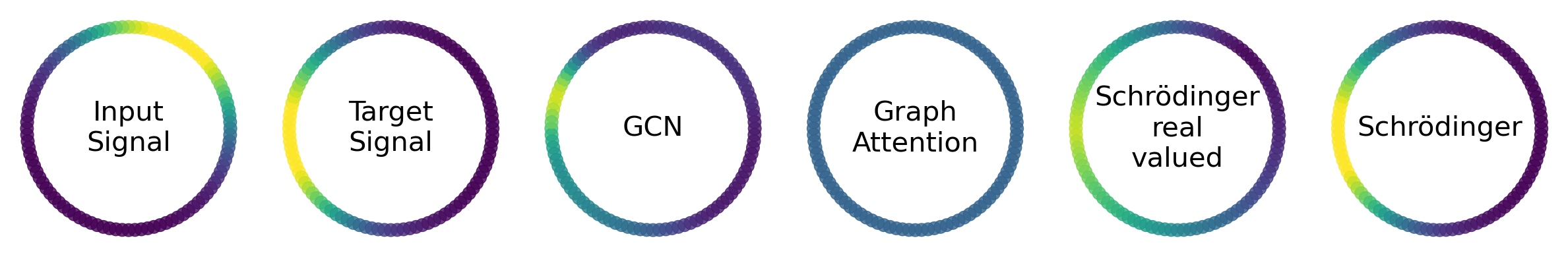}
\caption{
Signal propagation on a ring graph. 
From left to right: the noisy input Gaussian signal, the target signal obtained by translating the Gaussian by a fixed angular displacement, and the outputs predicted by GCN, GAT, the real-valued Schr\"odinger GNN ablation, and the full Schr\"odinger GNN. 
Only the full Schr\"odinger GNN accurately transports the input signal to the target location.
}
\label{fig:ring_gaussian_outputs}
\end{figure}

\paragraph{Sensitivity analysis.} In oversquashing analysis, one approach for modeling the sensitivity of the GNN in a given vertex to the signal received from another vertex is via the gradients of the MPNN \citep{topping2022understandingoversquashingbottlenecksgraphs,BlackEtAl2023ER, di2023over}. In Appendix \ref{Sensitivity Analysis of Signal Routing via Schr\"odinger Operators} we conduct a similar analysis for our setting. We focus on a single Schr\"odinger GNN layer. We first suppose that the modulation feature $h$, phase $\theta$, and time $t$ were chosen so that  each signal chunk is  transmitted to a meaningful destination location $r$ (appropriate for the task). Under this assumption, one would like to show that the derivative of the output signal about the location $r$ with respect to the location of the input chunk is not small, as $r$ is seen as the "right" destination for solving the task. We show that this derivative is exactly 1. This shows that all of the signal information reached its destination. In contrast to oversquashing analyses, 
we do not want high sensitivity of the signal at locations other than the destination $r$ with respect to the 
 input chunk, as transmitting information between such location pairs is not seen as informative for the task.

\section{Experiments}
\label{sec:experiments}

We evaluate Schrödinger GNN on a range of graph learning tasks and investigate its signal propagation properties. We first show the signal propagation experiment on a ring graph, followed by evaluations on graph classification, heterophilous node classification, and long-range graph benchmarks. We further examine the learned signal propagation through a window-based analysis of signal flow. Additional experiments in Appendix~\ref{app:exp_detail_addtional_exp} include ablation studies of unitary Schrödinger propagation, complex modulation, and PMO, as well as a sensitivity analysis of the learned feature locations.
Code available at: \url{https://github.com/edenagar/SGCN}.

\subsection{Toy example: signal propagation on a ring}
\label{sec:ring_graph_exp}

\begin{wraptable}{r}{50mm}
\centering
 \vspace{-5mm}
 \caption{
 Test loss for the ring signal transport task. 
The Schr\"odinger GNN achieves the lowest error.
 }
 \resizebox{0.32\textwidth}{!}{
 \begin{tabular}{lr}
 \toprule
 Model & Test Loss $(\downarrow)$\\
\midrule
GCN  & 0.6644\scriptsize{$\pm$0.0720}   \\
GAT & 0.6050\scriptsize{$\pm$0.0052} \\
Schr\"odinger real & 0.9334\scriptsize{$\pm$0.0514} \\
Schr\"odinger & \cellcolor{blue!20}\textbf{3e-04\scriptsize{$\pm$2e-04}} \\
\bottomrule
 \end{tabular}
 }
 \label{tab:ring_transport_results}
\end{wraptable}
We first demonstrate the ability of Schr\"odinger GNN to direct the propagation of the signal through a regression task on a ring graph. 
Consider a ring graph discretizing the unit circle, and the location feature $x=\cos(\theta)$, where $\theta$ is the angle. 
Each input signal is a noisy Gaussian on the ring with random mean $\mu$ and variance $\sigma^2$, and the target is obtained by translating the same Gaussian by a prescribed angular displacement $d$.
The task is to map the input signal to the target signal. 
Figure~\ref{fig:ring_gaussian_outputs} compares the predicted signal produced by Schr\"odinger GNN against GCN~\cite{kipf2017semisupervisedclassificationgraphconvolutional}, GAT~\cite{velickovic2018graphattentionnetworks}, and a real-valued Schr\"odinger GNN. 
Table~\ref{tab:ring_transport_results} reports the corresponding test losses. 
These results show that only the Schr\"odinger GNN, with a modulated input signal, accurately learns the prescribed transport.

\subsection{Graph classification on TU dataset under architecture-matched evaluation}

We next evaluate our approach on graph classification tasks from the TU dataset \cite{morris2020tudatasetcollectionbenchmarkdatasets}, using ENZYMES, IMDB, MUTAG, and PROTEINS. 
To isolate the effect of the propagation mechanism from differences in model depth or capacity, we use an architecture-matched evaluation protocol. 
We compare Schr\"odinger GNN with GIN \cite{xu2019howpowerfularegraphneuralnetworks}, GCN \cite{kipf2017semisupervisedclassificationgraphconvolutional}, GAT \cite{velickovic2018graphattentionnetworks}, UniGCN \cite{kiani2024unitaryconvolutionslearninggraphs}, adaptive unitary, Schr\"odinger, and their PMO variants. 
In this setting, every method is implemented with the same macro-architecture: six graph convolution or message-passing layers followed by the same graph-level readout and a final linear prediction layer. 
We match the number of trainable parameters across methods. 
The results are reported in Table~\ref{tab:architecture_matched}. 
We see that our method achieves the best performance on all four datasets, suggesting that the Schr\"odinger propagation is effective and PMO further improves its performance.

\begin{table*}[h]
\centering
 \caption{
Architecture- and parameter-matched graph classification results on TU datasets.
 }
 \resizebox{0.75\textwidth}{!}{
 \begin{tabular}{lcccc}
\toprule
 & ENZYMES & IMDB & MUTAG & PROTEINS \\
\midrule
GIN & 31.93\scriptsize{$\pm$3.16} & 69.22\scriptsize{$\pm$3.14} & 78.19\scriptsize{$\pm$5.57} & 71.88\scriptsize{$\pm$3.08} \\
GCN  & 31.66\scriptsize{$\pm$5.35} & 50.60\scriptsize{$\pm$4.10} & 73.24\scriptsize{$\pm$6.27} & 71.41\scriptsize{$\pm$3.04} \\
GAT  & 31.13\scriptsize{$\pm$3.48} & 49.54\scriptsize{$\pm$2.54} & 75.21\scriptsize{$\pm$6.41} & 72.31\scriptsize{$\pm$3.28} \\
UniGCN  & 40.30\scriptsize{$\pm$6.63} & 65.42\scriptsize{$\pm$2.80} & 75.74\scriptsize{$\pm$6.67} & 69.19\scriptsize{$\pm$3.01} \\
\midrule
Adaptive Unitary & 41.60\scriptsize{$\pm$5.18} & 65.46\scriptsize{$\pm$2.48} & 75.53\scriptsize{$\pm$5.95} & 71.79\scriptsize{$\pm$3.33} \\
Adaptive Unitary PMO & 41.83\scriptsize{$\pm$4.44} & 66.27\scriptsize{$\pm$3.01} & 75.62\scriptsize{$\pm$6.24} & 71.77\scriptsize{$\pm$2.84}\\
Schr\"odinger & 43.50\scriptsize{$\pm$4.89} & 65.86\scriptsize{$\pm$2.83} & 75.42\scriptsize{$\pm$6.11} & 71.57\scriptsize{$\pm$2.56} \\
Schr\"odinger PMO & \cellcolor{blue!20}\textbf{43.70\scriptsize{$\pm$3.37}} & \cellcolor{blue!20}\textbf{69.60\scriptsize{$\pm$2.85}} & \cellcolor{blue!20}\textbf{79.25\scriptsize{$\pm$6.19}} & \cellcolor{blue!20}\textbf{72.68\scriptsize{$\pm$3.05}} \\
\bottomrule
 \end{tabular}
}
 \label{tab:architecture_matched}
\end{table*}

\subsection{Heterophilous node classification}

We also assess Schr\"odinger GNN on heterophilous node classification benchmarks from  \cite{platonov2023critical}, including Roman-empire,
Amazon-ratings, Minesweeper, Tolokers, and Questions.
The competing baselines include GCN \cite{kipf2017semisupervisedclassificationgraphconvolutional}, SAGE \cite{hamilton2017inductive}, GAT \cite{velickovic2018graphattentionnetworks}, GT \cite{dwivedi2021generalization}, UniGCN \cite{kiani2024unitaryconvolutionslearninggraphs}, and Lie UniGCN. 
Table~\ref{tab:heterophilous_results} shows the classification results. 
Schr\"odinger GNN  achieves the best performance on three out of five datasets, and on the remaining two, it obtains the second best, indicating that Schr\"odinger propagation is also well-suited for heterophilous graphs.

\begin{table*}[h]
\centering
\caption{ 
Heterophilous node classification performance.
Results are reported as mean $\pm$ standard deviation. 
Results of competing methods marked with $\dag$ are taken from \cite{platonov2023critical, kiani2024unitaryconvolutionslearninggraphs}. 
}
 \resizebox{0.85\textwidth}{!}{
 \begin{tabular}{lccccccccccr}
\toprule  
& Roman-Empire & Amazon-Rating & Minesweeper & Tolokers & Questions \\ 
\cmidrule{2-6}
& Test AP $\uparrow$ & Test AP $\uparrow$ & ROC AUC $\uparrow$ & ROC AUC $\uparrow$ & ROC AUC $\uparrow$ \\ 
\midrule
GCN$^\dag$ & 73.69\scriptsize{$\pm$0.74} & 48.70\scriptsize{$\pm$0.63} & 89.75\scriptsize{$\pm$0.52} & 83.64\scriptsize{$\pm$0.67} & 76.09\scriptsize{$\pm$1.27} \\
SAGE$^\dag$ & 85.74\scriptsize{$\pm$0.67} & 53.63\scriptsize{$\pm$0.39} & 93.51\scriptsize{$\pm$0.57} & 82.43\scriptsize{$\pm$0.44} & 76.44\scriptsize{$\pm$0.62} \\
GAT$^\dag$ & 80.87\scriptsize{$\pm$0.30} & 49.09\scriptsize{$\pm$0.63} & 92.01\scriptsize{$\pm$0.68} & 83.70\scriptsize{$\pm$0.47} & 77.43\scriptsize{$\pm$1.20} \\
\midrule
GT$^\dag$ & 86.51\scriptsize{$\pm$0.73} & 51.17\scriptsize{$\pm$0.66} & 91.85\scriptsize{$\pm$0.76} & 83.23\scriptsize{$\pm$0.64} & 77.95\scriptsize{$\pm$0.68}\\
\midrule
UniGCN$^\dag$ & \cellcolor{orange!20}\underline{87.21\scriptsize{$\pm$0.76}} & \cellcolor{blue!20}\textbf{55.34\scriptsize{$\pm$0.74}} & 94.27\scriptsize{$\pm$0.58} & 84.83\scriptsize{$\pm$0.68} & 79.21\scriptsize{$\pm$0.79} \\
Lie UniGCN$^\dag$ & 85.50\scriptsize{$\pm$0.22} & 52.35\scriptsize{$\pm$0.26} & \cellcolor{orange!20}\underline{96.11\scriptsize{$\pm$0.10}} & \cellcolor{blue!20}\textbf{85.18\scriptsize{$\pm$0.43}} & \cellcolor{orange!20}\underline{80.01\scriptsize{$\pm$0.43}} \\
\midrule
Schr\"odinger & \cellcolor{blue!20}\textbf{88.56\scriptsize{$\pm$0.71}} & \cellcolor{orange!20}\underline{54.26\scriptsize{$\pm$0.40}} & \cellcolor{blue!20}\textbf{96.31\scriptsize{$\pm$0.49}} & \cellcolor{orange!20}\underline{85.09\scriptsize{$\pm$0.17}} & \cellcolor{blue!20}\textbf{80.14\scriptsize{$\pm$0.20}}\\
\bottomrule
 \end{tabular}
}
 \label{tab:heterophilous_results}
\end{table*}

\subsection{Long range graph benchmark}

Finally, we evaluate our approach on the long range graph benchmark (LRGB) \cite{dwivedi2022long}.
We consider four tasks: (i) Peptides-Func, a graph classification task, (ii) Peptides-Struct, a graph regression task, and (iii) PascalVOC-SP and (iv) COCO-SP, which formulate semantic image segmentation as a node-classification task on superpixel graphs.
These benchmarks are designed to test whether a model can capture long-range dependencies beyond local neighborhood aggregation.
We compare our method with GCN \cite{kipf2017semisupervisedclassificationgraphconvolutional}, GINE \cite{xu2019howpowerfularegraphneuralnetworks}, GatedGCN \cite{bresson2018residualgatedgraphconvnets}, GUMP \cite{qiu2024graphunitarymessagepassing}, GPS \cite{rampavsek2022recipe}, DRew \cite{gutteridge2023drew}, Exphormer \cite{shirzad2023exphormersparsetransformersgraphs}, GRIT \cite{ma2023graphinductivebiasestransformers}, Graph ViT \cite{he2023generalizationvitmlpmixergraphs}, CRAWL \cite{tonshoff2021walking}, UniGCN \cite{kiani2024unitaryconvolutionslearninggraphs}, and Lie UniGCN. 
The results are reported in Table~\ref{tab:long_range_results}. 
Schr\"odinger GNN performs strongly across LRGB tasks, where it outperforms all the baselines on the Peptides-Func dataset, and it achieves the second-best among the other three tasks. 
This suggests that Schr\"odinger propagation is effective at learning long-range signals.

\begin{table*}[h]
\centering
 \caption{
Comparison of Schr\"odinger GNN with competing GNN architectures on the LRGB dataset.
Results of competing methods marked with $\dag$ are taken from \cite{tonshoff2023wheredidgapgoreassessinglongrangegraph, kiani2024unitaryconvolutionslearninggraphs}.
 }
 \resizebox{0.85\textwidth}{!}{
 \begin{tabular}{lccccccccccr}
\toprule  
& Peptides-Func & Peptides-Struct & PascalVOC-SP  & COCO-SP \\ 
\cmidrule{2-5}
& Test AP $\uparrow$ & Test MAE $\downarrow$ & Test F1 $\uparrow$ & Test F1 $\uparrow$ \\ 
\midrule
GCN$^\dag$ & 0.6860\scriptsize{$\pm$0.0050} & 0.2460\scriptsize{$\pm$0.0007} & 0.1338\scriptsize{$\pm$0.0007} & 0.2078\scriptsize{$\pm$0.0031} \\
GINE$^\dag$ & 0.6621\scriptsize{$\pm$0.0067} & 0.2473\scriptsize{$\pm$0.0017} & 0.2125\scriptsize{$\pm$0.0009} & 0.2718\scriptsize{$\pm$0.0054} \\
GatedGCN$^\dag$ & 0.6765\scriptsize{$\pm$0.0047} & 0.2477\scriptsize{$\pm$0.0009} & 0.2922\scriptsize{$\pm$0.0018} & 0.3880\scriptsize{$\pm$0.0040} \\
GUMP$^\dag$ & 0.6843\scriptsize{$\pm$0.0037} & 0.2564\scriptsize{$\pm$0.0023} & - & - \\ 
\midrule
GPS$^\dag$ & 0.6534\scriptsize{$\pm$0.0091} & 0.2509\scriptsize{$\pm$0.0014} & \cellcolor{blue!20}\textbf{0.3884\scriptsize{$\pm$0.0055}} & 0.4440\scriptsize{$\pm$0.0065} \\
DRew$^\dag$ & 0.7150\scriptsize{$\pm$0.0044} & 0.2536\scriptsize{$\pm$0.0015} & - & 0.3314\scriptsize{$\pm$0.0024} \\
Exphormer$^\dag$ & 0.6527\scriptsize{$\pm$0.0043} & 0.2481\scriptsize{$\pm$0.0007} & 0.3430\scriptsize{$\pm$0.0008} & 0.3960\scriptsize{$\pm$0.0027} \\
GRIT$^\dag$ & 0.6988\scriptsize{$\pm$0.0082} & 0.2460\scriptsize{$\pm$0.0012} & - & - \\
Graph ViT$^\dag$ & 0.6942\scriptsize{$\pm$0.0075} & 0.2449\scriptsize{$\pm$0.0016} & - & - \\
CRAWL$^\dag$ & 0.7074\scriptsize{$\pm$0.0032} & 0.2506\scriptsize{$\pm$0.0022} & - & \cellcolor{blue!20}\textbf{0.4588\scriptsize{$\pm$0.0079}} \\
\midrule
UniGCN$^\dag$ & 0.7072\scriptsize{$\pm$0.0035} & \cellcolor{blue!20}\textbf{0.2425\scriptsize{$\pm$0.0009}} & 0.2852\scriptsize{$\pm$0.0016} & 0.3516\scriptsize{$\pm$0.0070} \\
Lie UniGCN$^\dag$ & \cellcolor{orange!20}\underline{0.7173\scriptsize{$\pm$0.0061}} & 0.2460\scriptsize{$\pm$0.0011} & 0.3153\scriptsize{$\pm$0.0035} & 0.4005\scriptsize{$\pm$0.0067} \\
\midrule 
Schr\"odinger & \cellcolor{blue!20}\textbf{0.7207\scriptsize{$\pm$0.0099}} & \cellcolor{orange!20}\underline{0.2439\scriptsize{$\pm$0.0011}} & \cellcolor{orange!20}\underline{0.3507\scriptsize{$\pm$0.0019}} & \cellcolor{orange!20}\underline{0.4259\scriptsize{$\pm$0.0034}} \\
\bottomrule
 \end{tabular}
}
 \label{tab:long_range_results}
\end{table*}

\subsection{Diagnosis of signal flow with windows}

\begin{wrapfigure}[18]{r}{0.6\textwidth}
     \includegraphics[width=0.55\textwidth]{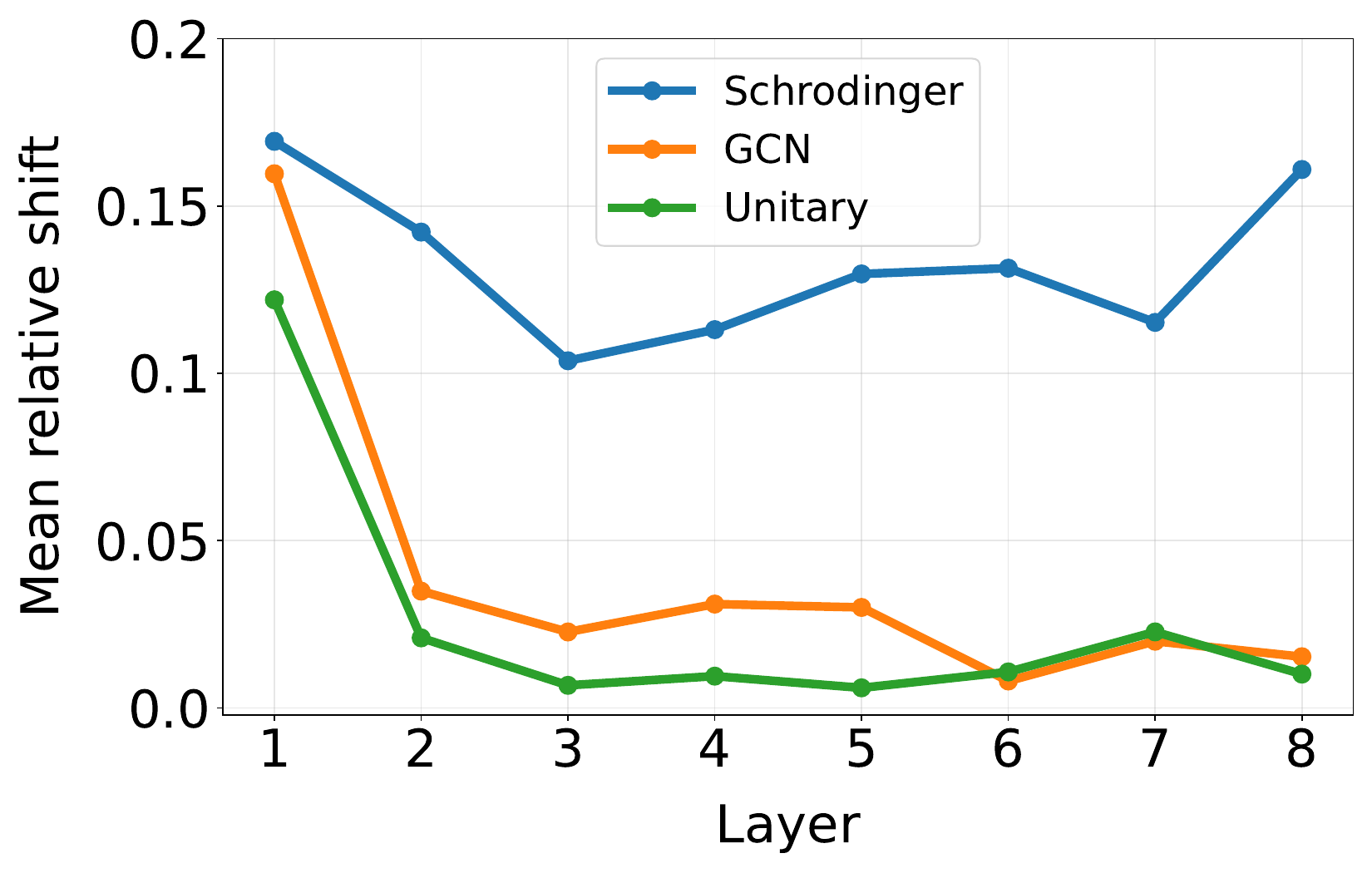}
    \caption{Diagnosis of signal flow on PascalVOC-SP. 
    The Schr\"odinger GNN maintains a large relative shift across layers, indicating that signal  transportation is achieved by the modulation in  Schr\"odinger layers.
    }
    \label{fig:pascal_mean_relative_shift}
\end{wrapfigure}To directly diagnose whether a trained model transports information across the graph, we measure whether the trained layer moves the signal along the learned location features. 
For each graph, we split the node signal into localized windows according to the values of the location features. 
Each  window isolates the part of the signal concentrated around a small region in feature space. We then pass each windowed signal through one trained layer and compare its average location before and after the layer. 
The shift of the window is measured as the change in its average location along the corresponding feature coordinate. To make this quantity comparable across graphs and features, we normalize the shift by the standard deviation of that location feature. 
We compute this relative shift for every layer, average it over the PascalVOC-SP dataset \cite{dwivedi2022long}, and repeat the same process for evaluating our Schr\"odinger GNN, GCN \cite{kipf2017semisupervisedclassificationgraphconvolutional}, and unitary GCN \cite{kiani2024unitaryconvolutionslearninggraphs}.
Figure~\ref{fig:pascal_mean_relative_shift} shows the mean relative shift across layers. 
Our method demonstrates a clear and sustained shift throughout the network. 
The Schr\"odinger GNN, with unitary propagation and feature-dependent modulation, enables the network to create nonzero effective momentum and route signal mass across the graph. Thus, it confirms that our model learns layer-wise 
signal flow.

\section{Summary}
\label{sec:summary}

We presented a new approach for defining and analyzing  signal propagation across graphs. The approach directly models where the information of the signal is, how well concentrated it is, and how well it is routed between regions in the graph. We presented Schr\"odinger GNN, a graph neural network that is able to route the information of the signal along any direction in the graph. We showed that standard GNNs do not have this capability. One limitation of Schr\"odinger filters with respect to simple polynomial filters is that applying the Schr\"odinger operator on a signal involves approximating the exponential of the GSO, which involves applying the GSO several times.

\section*{Acknowledgments}

RL was supported by a grant from the United States-Israel Binational Science Foundation (BSF), Jerusalem, Israel, and the United States National Science Foundation (NSF), (NSF-BSF, grant No. 2024660), and by the Israel Science Foundation (ISF grant No. 1937/23).
YEL acknowledges funding from the Alexander von Humboldt Foundation, the Munich Center for Machine Learning (MCML), the Schmidt Futures Israeli Women’s Postdoctoral Award, and the Viterbi Fellowship from the Faculty of Electrical and Computer Engineering at the Technion.

\bibliographystyle{unsrt}
\bibliography{bib}

@inproceedings{banerjee2022oversquashing,
  title={Oversquashing in gnns through the lens of information contraction and graph expansion},
  author={Banerjee, Pradeep Kr and Karhadkar, Kedar and Wang, Yu Guang and Alon, Uri and Mont{\'u}far, Guido},
  booktitle={2022 58th Annual Allerton Conference on Communication, Control, and Computing (Allerton)},
  pages={1--8},
  year={2022},
  organization={IEEE}
}

@article{scarselli2009graph,
  title   = {The Graph Neural Network Model},
  author  = {Scarselli, Franco and Gori, Marco and Tsoi, Ah Chung and Hagenbuchner, Markus and Monfardini, Gabriele},
  journal = {IEEE Transactions on Neural Networks},
  volume  = {20},
  number  = {1},
  pages   = {61--80},
  year    = {2009}
}

@article{mishayev2025short,
  title={Short-range oversquashing},
  author={Mishayev, Yaaqov and Sverdlov, Yonatan and Amir, Tal and Dym, Nadav},
  journal={arXiv preprint arXiv:2511.20406},
  year={2025}
}

@article{sverdlov2024fsw,
  title={FSW-GNN: A bi-lipschitz WL-equivalent graph neural network},
  author={Sverdlov, Yonatan and Davidson, Yair and Dym, Nadav and Amir, Tal},
  journal={arXiv preprint arXiv:2410.09118},
  year={2024}
}

@inproceedings{
lin2026adaptive,
title={Adaptive Canonicalization with Application to Invariant Anisotropic Geometric Networks},
author={Ya-Wei Eileen Lin and Ron Levie},
booktitle={The Fourteenth International Conference on Learning Representations},
year={2026}
}

@inproceedings{gravina2025oversquashing,
  title={On oversquashing in graph neural networks through the lens of dynamical systems},
  author={Gravina, Alessio and Eliasof, Moshe and Gallicchio, Claudio and Bacciu, Davide and Sch{\"o}nlieb, Carola-Bibiane},
  booktitle={Proceedings of the AAAI Conference on Artificial Intelligence},
  volume={39},
  number={16},
  pages={16906--16914},
  year={2025}
}

@inproceedings{
hariri2025return,
title={Return of ChebNet: Understanding and Improving an Overlooked {GNN} on Long Range Tasks},
author={Ali Hariri and Alvaro Arroyo and Alessio Gravina and Moshe Eliasof and Carola-Bibiane Sch{\"o}nlieb and Davide Bacciu and Xiaowen Dong and Kamyar Azizzadenesheli and Pierre Vandergheynst},
booktitle={The Thirty-ninth Annual Conference on Neural Information Processing Systems},
year={2025}
}

@article{lin2024equivariant,
  title={Equivariant machine learning on graphs with nonlinear spectral filters},
  author={Lin, Ya-Wei Eileen and Talmon, Ronen and Levie, Ron},
  journal={Advances in Neural Information Processing Systems},
  volume={37},
  pages={128182--128226},
  year={2024}
}

@inproceedings{
AlonYahav2021Bottleneck,
title={On the Bottleneck of Graph Neural Networks and its Practical Implications},
author={Uri Alon and Eran Yahav},
booktitle={International Conference on Learning Representations},
year={2021}
}

@inproceedings{
topping2022understandingoversquashingbottlenecksgraphs,
title={Understanding over-squashing and bottlenecks on graphs via curvature},
author={Jake Topping and Francesco Di Giovanni and Benjamin Paul Chamberlain and Xiaowen Dong and Michael M. Bronstein},
booktitle={International Conference on Learning Representations},
year={2022}
}

@inproceedings{di2023over,
  title={On over-squashing in message passing neural networks: The impact of width, depth, and topology},
  author={Di Giovanni, Francesco and Giusti, Lorenzo and Barbero, Federico and Luise, Giulia and Lio, Pietro and Bronstein, Michael M},
  booktitle={International conference on machine learning},
  pages={7865--7885},
  year={2023},
  organization={PMLR}
}

@inproceedings{BlackEtAl2023ER,
  title={Understanding oversquashing in {GNN}s through the lens of effective resistance},
  author={Black, Mitchell and Wan, Zhengchao and Nayyeri, Amir and Wang, Yusu},
  booktitle={International Conference on Machine Learning},
  pages={2528--2547},
  year={2023},
  organization={PMLR}
}

@article{tuysuz2021hybrid,
  title={Hybrid quantum classical graph neural networks for particle track reconstruction},
  author={T{\"u}ys{\"u}z, Cenk and Rieger, Carla and Novotny, Kristiane and Demirk{\"o}z, Bilge and Dobos, Daniel and Potamianos, Karolos and Vallecorsa, Sofia and Vlimant, Jean-Roch and Forster, Richard},
  journal={Quantum Machine Intelligence},
  volume={3},
  number={2},
  pages={29},
  year={2021},
  publisher={Springer}
}

@inproceedings{GiraldoEtAl2023Tradeoff,
  title={On the trade-off between over-smoothing and over-squashing in deep graph neural networks},
  author={Giraldo, Jhony H and Skianis, Konstantinos and Bouwmans, Thierry and Malliaros, Fragkiskos D},
  booktitle={Proceedings of the 32nd ACM international conference on information and knowledge management},
  pages={566--576},
  year={2023}
}

@article{venegas2012quantum,
  title={Quantum walks: a comprehensive review},
  author={Venegas-Andraca, Salvador El{\'\i}as},
  journal={Quantum Information Processing},
  volume={11},
  number={5},
  pages={1015--1106},
  year={2012},
  publisher={Springer}
}

@article{kempe2003quantum,
  title={Quantum random walks: an introductory overview},
  author={Kempe, Julia},
  journal={Contemporary Physics},
  volume={44},
  number={4},
  pages={307--327},
  year={2003},
  publisher={Taylor \& Francis}
}

@inproceedings{aharonov2001quantum,
  title={Quantum walks on graphs},
  author={Aharonov, Dorit and Ambainis, Andris and Kempe, Julia and Vazirani, Umesh},
  booktitle={Proceedings of the thirty-third annual ACM symposium on Theory of computing},
  pages={50--59},
  year={2001}
}

@article{farhi1998quantum,
  title={Quantum computation and decision trees},
  author={Farhi, Edward and Gutmann, Sam},
  journal={Physical Review A},
  volume={58},
  number={2},
  pages={915},
  year={1998},
  publisher={APS}
}

@article{bai2021learning,
  title={Learning graph convolutional networks based on quantum vertex information propagation},
  author={Bai, Lu and Jiao, Yuhang and Cui, Lixin and Rossi, Luca and Wang, Yue and Yu, Philip S and Hancock, Edwin R},
  journal={IEEE Transactions on Knowledge and Data Engineering},
  volume={35},
  number={2},
  pages={1747--1760},
  year={2021},
  publisher={IEEE}
}

@inproceedings{GravinaEtAl2025SWAN,
  title={On oversquashing in graph neural networks through the lens of dynamical systems},
  author={Gravina, Alessio and Eliasof, Moshe and Gallicchio, Claudio and Bacciu, Davide and Sch{\"o}nlieb, Carola-Bibiane},
  booktitle={Proceedings of the AAAI Conference on Artificial Intelligence},
  volume={39},
  number={16},
  pages={16906--16914},
  year={2025}
}

@article{levie2014adjoint,
  title={Adjoint translation, adjoint observable and uncertainty principles},
  author={Levie, Ron and Stark, H-G and Lieb, Florian and Sochen, Nir},
  journal={Advances in computational mathematics},
  volume={40},
  number={3},
  pages={609--627},
  year={2014},
  publisher={Springer}
}

@article{levie2019cayleynets,
  title={Cayleynets: Graph convolutional neural networks with complex rational spectral filters},
  author={Levie, Ron and Monti, Federico and Bresson, Xavier and Bronstein, Michael M},
  journal={IEEE Transactions on Signal Processing},
  volume={67},
  number={1},
  pages={97--109},
  year={2018},
  publisher={IEEE}
}

@article{bassey2021survey,
  title={A survey of complex-valued neural networks},
  author={Bassey, Joshua and Qian, Lijun and Li, Xianfang},
  journal={arXiv preprint arXiv:2101.12249},
  year={2021}
}

@article{defferrard2017convolutionalneuralnetworksgraphs,
  title={Convolutional neural networks on graphs with fast localized spectral filtering},
  author={Defferrard, Micha{\"e}l and Bresson, Xavier and Vandergheynst, Pierre},
  journal={Advances in Neural Information Processing Systems},
  volume={29},
  year={2016}
}

@article{qiu2024graphunitarymessagepassing,
  title={Graph unitary message passing},
  author={Qiu, Haiquan and Bian, Yatao and Yao, Quanming},
  journal={arXiv preprint arXiv:2403.11199},
  year={2024}
}

@inproceedings{zhao2019buildingefficientdeepneural,
  title={Building efficient deep neural networks with unitary group convolutions},
  author={Zhao, Ritchie and Hu, Yuwei and Dotzel, Jordan and Sa, Christopher De and Zhang, Zhiru},
  booktitle={Proceedings of the IEEE/CVF Conference on Computer Vision and Pattern Recognition},
  pages={11303--11312},
  year={2019}
}

@inproceedings{
kiani2024unitaryconvolutionslearninggraphs,
title={Unitary Convolutions for Learning on Graphs and Groups},
author={Bobak Kiani and Lukas Fesser and Melanie Weber},
booktitle={The Thirty-eighth Annual Conference on Neural Information Processing Systems},
year={2024}
}

@article{kipf2017semisupervisedclassificationgraphconvolutional,
  title={Semi-supervised classification with graph convolutional networks},
  author={Kipf, Thomas N and Welling, Max},
  journal={arXiv preprint arXiv:1609.02907},
  year={2016}
}

@book{nielsen2010quantum,
  title={Quantum computation and quantum information},
  author={Nielsen, Michael A and Chuang, Isaac L},
  year={2010},
  publisher={Cambridge university press}
}

@inproceedings{
oono2021graphneuralnetworksexponentially,
title={Graph Neural Networks Exponentially Lose Expressive Power for Node Classification},
author={Kenta Oono and Taiji Suzuki},
booktitle={International Conference on Learning Representations},
year={2020}
}

@inproceedings{
zhao2020pairnorm,
title={PairNorm: Tackling Oversmoothing in {GNN}s},
author={Lingxiao Zhao and Leman Akoglu},
booktitle={International Conference on Learning Representations},
year={2020}
}

@inproceedings{
rong2020dropedgedeepgraphconvolutional,
title={DropEdge: Towards Deep Graph Convolutional Networks on Node Classification},
author={Yu Rong and Wenbing Huang and Tingyang Xu and Junzhou Huang},
booktitle={International Conference on Learning Representations},
year={2020}
}

@inproceedings{chen2020simpledeepgraphconvolutional,
  title={Simple and deep graph convolutional networks},
  author={Chen, Ming and Wei, Zhewei and Huang, Zengfeng and Ding, Bolin and Li, Yaliang},
  booktitle={International conference on machine learning},
  pages={1725--1735},
  year={2020},
  organization={PMLR}
}

@article{sandryhaila2013discretesignalprocessinggraphs,
  title={Discrete signal processing on graphs: Frequency analysis},
  author={Sandryhaila, Aliaksei and Moura, Jose MF},
  journal={IEEE Transactions on signal processing},
  volume={62},
  number={12},
  pages={3042--3054},
  year={2014},
  publisher={IEEE}
}

@article{morris2020tudatasetcollectionbenchmarkdatasets,
  title={{TUD}ataset: A collection of benchmark datasets for learning with graphs},
  author={Morris, Christopher and Kriege, Nils M and Bause, Franka and Kersting, Kristian and Mutzel, Petra and Neumann, Marion},
  journal={arXiv preprint arXiv:2007.08663},
  year={2020}
}

@article{levie2018uncertaintyprinciplesoptimallysparse,
  title={Uncertainty principles and optimally sparse wavelet transforms},
  author={Levie, Ron and Sochen, Nir},
  journal={Applied and Computational Harmonic Analysis},
  volume={48},
  number={3},
  pages={811--867},
  year={2020},
  publisher={Elsevier}
}

@article{levie2021waveletplanchereltheoryapplication,
  title={A Wavelet Plancherel Theory with Application to Multipliers and Sparse Approximations},
  author={Levie, Ron and Sochen, Nir},
  journal={Numerical Functional Analysis and Optimization},
  volume={43},
  number={11},
  pages={1303--1400},
  year={2022},
  publisher={Taylor \& Francis}
}

@article{Halvdansson_2023,
   title={Existence of uncertainty minimizers for the continuous wavelet transform},
   volume={296},
   ISSN={1522-2616},
   DOI={10.1002/mana.202100466},
   number={3},
   journal={Mathematische Nachrichten},
   publisher={Wiley},
   author={Halvdansson, Simon and Olsen, Jan‐Fredrik and Sochen, Nir and Levie, Ron},
   year={2023},
   month=jan, pages={1156–1172} 
}

@inproceedings{banerjee2022oversquashinggnnslensinformation,
  title={Oversquashing in {GNN}s through the lens of information contraction and graph expansion},
  author={Banerjee, Pradeep Kr and Karhadkar, Kedar and Wang, Yu Guang and Alon, Uri and Mont{\'u}far, Guido},
  booktitle={2022 58th Annual Allerton Conference on Communication, Control, and Computing (Allerton)},
  pages={1--8},
  year={2022},
  organization={IEEE}
}

@inproceedings{gilmer2017neuralmessagepassingquantum,
  title={Neural message passing for quantum chemistry},
  author={Gilmer, Justin and Schoenholz, Samuel S and Riley, Patrick F and Vinyals, Oriol and Dahl, George E},
  booktitle={International conference on machine learning},
  pages={1263--1272},
  year={2017},
  organization={PMLR}
}

@article{dwivedi2022long,
  title={Long range graph benchmark},
  author={Dwivedi, Vijay Prakash and Ramp{\'a}{\v{s}}ek, Ladislav and Galkin, Michael and Parviz, Ali and Wolf, Guy and Luu, Anh Tuan and Beaini, Dominique},
  journal={Advances in Neural Information Processing Systems},
  volume={35},
  pages={22326--22340},
  year={2022}
}

@misc{bresson2018residualgatedgraphconvnets,
      title={Residual Gated Graph ConvNets}, 
      author={Xavier Bresson and Thomas Laurent},
      year={2018},
      eprint={1711.07553},
      archivePrefix={arXiv},
      primaryClass={cs.LG}
}

@article{kingma2017adammethodstochasticoptimization,
  title={Adam: A method for stochastic optimization},
  author={Kingma, Diederik P and Ba, Jimmy},
  journal={arXiv preprint arXiv:1412.6980},
  year={2014}
}

@inproceedings{shirzad2023exphormersparsetransformersgraphs,
  title={Exphormer: Sparse transformers for graphs},
  author={Shirzad, Hamed and Velingker, Ameya and Venkatachalam, Balaji and Sutherland, Danica J and Sinop, Ali Kemal},
  booktitle={International Conference on Machine Learning},
  pages={31613--31632},
  year={2023},
  organization={PMLR}
}

@inproceedings{he2023generalizationvitmlpmixergraphs,
  title={A generalization of  {V}i{T}/{MLP}-{M}ixer to graphs},
  author={He, Xiaoxin and Hooi, Bryan and Laurent, Thomas and Perold, Adam and LeCun, Yann and Bresson, Xavier},
  booktitle={International conference on machine learning},
  pages={12724--12745},
  year={2023},
  organization={PMLR}
}

@inproceedings{ma2023graphinductivebiasestransformers,
  title={Graph inductive biases in transformers without message passing},
  author={Ma, Liheng and Lin, Chen and Lim, Derek and Romero-Soriano, Adriana and Dokania, Puneet K and Coates, Mark and Torr, Philip and Lim, Ser-Nam},
  booktitle={International Conference on Machine Learning},
  pages={23321--23337},
  year={2023},
  organization={PMLR}
}

@inproceedings{
xu2019howpowerfularegraphneuralnetworks,
title={How Powerful are Graph Neural Networks?},
author={Keyulu Xu and Weihua Hu and Jure Leskovec and Stefanie Jegelka},
booktitle={International Conference on Learning Representations},
year={2019}
}

@article{rampavsek2022recipe,
  title={Recipe for a general, powerful, scalable graph transformer},
  author={Ramp{\'a}{\v{s}}ek, Ladislav and Galkin, Michael and Dwivedi, Vijay Prakash and Luu, Anh Tuan and Wolf, Guy and Beaini, Dominique},
  journal={Advances in Neural Information Processing Systems},
  volume={35},
  pages={14501--14515},
  year={2022}
}

@inproceedings{gutteridge2023drew,
  title={{DR}ew: Dynamically rewired message passing with delay},
  author={Gutteridge, Benjamin and Dong, Xiaowen and Bronstein, Michael M and Di Giovanni, Francesco},
  booktitle={International Conference on Machine Learning},
  pages={12252--12267},
  year={2023},
  organization={PMLR}
}

@inproceedings{zhang2018shufflenet,
  title={Shufflenet: An extremely efficient convolutional neural network for mobile devices},
  author={Zhang, Xiangyu and Zhou, Xinyu and Lin, Mengxiao and Sun, Jian},
  booktitle={Proceedings of the IEEE conference on computer vision and pattern recognition},
  pages={6848--6856},
  year={2018}
}

@inproceedings{ding2017circnn,
  title={Circnn: accelerating and compressing deep neural networks using block-circulant weight matrices},
  author={Ding, Caiwen and Liao, Siyu and Wang, Yanzhi and Li, Zhe and Liu, Ning and Zhuo, Youwei and Wang, Chao and Qian, Xuehai and Bai, Yu and Yuan, Geng and others},
  booktitle={Proceedings of the 50th Annual IEEE/ACM International Symposium on Microarchitecture},
  pages={395--408},
  year={2017}
}

@article{ying2021transformers,
  title={Do transformers really perform badly for graph representation?},
  author={Ying, Chengxuan and Cai, Tianle and Luo, Shengjie and Zheng, Shuxin and Ke, Guolin and He, Di and Shen, Yanming and Liu, Tie-Yan},
  journal={Advances in neural information processing systems},
  volume={34},
  pages={28877--28888},
  year={2021}
}

@article{dwivedi2020generalization,
  title={A generalization of transformer networks to graphs},
  author={Dwivedi, Vijay Prakash and Bresson, Xavier},
  journal={arXiv preprint arXiv:2012.09699},
  year={2020}
}

@article{vaswani2017attention,
  title={Attention is all you need},
  author={Vaswani, Ashish and Shazeer, Noam and Parmar, Niki and Uszkoreit, Jakob and Jones, Llion and Gomez, Aidan N and Kaiser, {\L}ukasz and Polosukhin, Illia},
  journal={Advances in neural information processing systems},
  volume={30},
  year={2017}
}

@article{tonshoff2021walking,
  title={Walking out of the weisfeiler leman hierarchy: Graph learning beyond message passing},
  author={T{\"o}nshoff, Jan and Ritzert, Martin and Wolf, Hinrikus and Grohe, Martin},
  journal={arXiv preprint arXiv:2102.08786},
  year={2021}
}

@inproceedings{
velickovic2018graphattentionnetworks,
title={Graph Attention Networks},
author={Petar Veličković and Guillem Cucurull and Arantxa Casanova and Adriana Romero and Pietro Liò and Yoshua Bengio},
booktitle={International Conference on Learning Representations},
year={2018}
}

@book{bracewell1986fourier,
  title={The Fourier transform and its applications},
  author={Bracewell, Ronald Newbold and Bracewell, Ronald N},
  volume={31999},
  year={1986},
  publisher={McGraw-hill New York}
}

@article{
tonshoff2023wheredidgapgoreassessinglongrangegraph,
title={Where Did the Gap Go? Reassessing the Long-Range Graph Benchmark},
author={Jan T{\"o}nshoff and Martin Ritzert and Eran Rosenbluth and Martin Grohe},
journal={Transactions on Machine Learning Research},
issn={2835-8856},
year={2024}
}

@article{lecun1998gradient,
  title={Gradient-Based Learning Applied to Document Recognition},
  author={LeCun, Yann and Bottou, L{\'e}on and Bengio, Yoshua and Haffner, Patrick},
  journal={Proceedings of the IEEE},
  volume={86},
  number={11},
  pages={2278--2324},
  year={1998},
  publisher={IEEE}
}

@inproceedings{nguyen2023revisitingoversmoothingoversquashingusing,
  title={Revisiting over-smoothing and over-squashing using {O}llivier-{R}icci curvature},
  author={Nguyen, Khang and Hieu, Nong Minh and Nguyen, Vinh Duc and Ho, Nhat and Osher, Stanley and Nguyen, Tan Minh},
  booktitle={International Conference on Machine Learning},
  pages={25956--25979},
  year={2023},
  organization={PMLR}
}

@article{verdon2019quantum,
  title={Quantum Graph Neural Networks},
  author={Verdon, Guillaume and McCourt, Trevor and Luzhnica, Enxhell and Singh, Vikash and Leichenauer, Stefan and Hidary, Jack},
  journal={arXiv preprint arXiv:1909.12264},
  year={2019}
}

@misc{lecun1998mnist,
  title={The {MNIST} database of handwritten digits},
  author={LeCun, Yann and Cortes, Corinna and Burges, Christopher J.C.},
  year={1998}
}

@inproceedings{platonov2023critical,
  title={A Critical Look at the Evaluation of {GNN}s under Heterophily: Are We Really Making Progress?},
  author={Platonov, Oleg and Kuznedelev, Denis and Diskin, Michael and Babenko, Artem and Prokhorenkova, Liudmila},
  booktitle={International Conference on Learning Representations},
  year={2023}
}

@inproceedings{hamilton2017inductive,
  title={Inductive Representation Learning on Large Graphs},
  author={Hamilton, William L. and Ying, Rex and Leskovec, Jure},
  booktitle={Advances in Neural Information Processing Systems},
  volume={30},
  year={2017}
}

@article{dwivedi2021generalization,
  title={A generalization of transformer networks to graphs},
  author={Dwivedi, Vijay Prakash and Bresson, Xavier},
  journal={arXiv preprint arXiv:2012.09699},
  year={2020}
}

@article{wu2020comprehensive,
  title={A comprehensive survey on graph neural networks},
  author={Wu, Zonghan and Pan, Shirui and Chen, Fengwen and Long, Guodong and Zhang, Chengqi and Yu, Philip S},
  journal={IEEE transactions on neural networks and learning systems},
  volume={32},
  number={1},
  pages={4--24},
  year={2020},
  publisher={IEEE}
}

@article{zhou2020graph,
  title={Graph neural networks: A review of methods and applications},
  author={Zhou, Jie and Cui, Ganqu and Hu, Shengding and Zhang, Zhengyan and Yang, Cheng and Liu, Zhiyuan and Wang, Lifeng and Li, Changcheng and Sun, Maosong},
  journal={AI open},
  volume={1},
  pages={57--81},
  year={2020},
  publisher={Elsevier}
}

@article{duvenaud2015convolutional,
  title={Convolutional networks on graphs for learning molecular fingerprints},
  author={Duvenaud, David K and Maclaurin, Dougal and Iparraguirre, Jorge and Bombarell, Rafael and Hirzel, Timothy and Aspuru-Guzik, Al{\'a}n and Adams, Ryan P},
  journal={Advances in neural information processing systems},
  volume={28},
  year={2015}
}

@article{kearnes2016molecular,
  title={Molecular graph convolutions: moving beyond fingerprints},
  author={Kearnes, Steven and McCloskey, Kevin and Berndl, Marc and Pande, Vijay and Riley, Patrick},
  journal={Journal of computer-aided molecular design},
  volume={30},
  number={8},
  pages={595--608},
  year={2016},
  publisher={Springer}
}

@article{schutt2017schnet,
  title={Schnet: A continuous-filter convolutional neural network for modeling quantum interactions},
  author={Sch{\"u}tt, Kristof and Kindermans, Pieter-Jan and Sauceda Felix, Huziel Enoc and Chmiela, Stefan and Tkatchenko, Alexandre and M{\"u}ller, Klaus-Robert},
  journal={Advances in neural information processing systems},
  volume={30},
  year={2017}
}

@article{battaglia2016interaction,
  title={Interaction networks for learning about objects, relations and physics},
  author={Battaglia, Peter and Pascanu, Razvan and Lai, Matthew and Jimenez Rezende, Danilo and others},
  journal={Advances in neural information processing systems},
  volume={29},
  year={2016}
}

@inproceedings{sanchez2020learning,
  title={Learning to simulate complex physics with graph networks},
  author={Sanchez-Gonzalez, Alvaro and Godwin, Jonathan and Pfaff, Tobias and Ying, Rex and Leskovec, Jure and Battaglia, Peter},
  booktitle={International conference on machine learning},
  pages={8459--8468},
  year={2020},
  organization={PMLR}
}

@inproceedings{perozzi2014deepwalk,
  title={Deepwalk: Online learning of social representations},
  author={Perozzi, Bryan and Al-Rfou, Rami and Skiena, Steven},
  booktitle={Proceedings of the 20th ACM SIGKDD international conference on Knowledge discovery and data mining},
  pages={701--710},
  year={2014}
}

@inproceedings{wang2019neural,
  title={Neural graph collaborative filtering},
  author={Wang, Xiang and He, Xiangnan and Wang, Meng and Feng, Fuli and Chua, Tat-Seng},
  booktitle={Proceedings of the 42nd international ACM SIGIR conference on Research and development in Information Retrieval},
  pages={165--174},
  year={2019}
}

@inproceedings{he2020lightgcn,
  title={Lightgcn: Simplifying and powering graph convolution network for recommendation},
  author={He, Xiangnan and Deng, Kuan and Wang, Xiang and Li, Yan and Zhang, Yongdong and Wang, Meng},
  booktitle={Proceedings of the 43rd International ACM SIGIR conference on research and development in Information Retrieval},
  pages={639--648},
  year={2020}
}

@inproceedings{li2018deeper,
  title={Deeper insights into graph convolutional networks for semi-supervised learning},
  author={Li, Qimai and Han, Zhichao and Wu, Xiao-Ming},
  booktitle={Proceedings of the AAAI conference on artificial intelligence},
  volume={32},
  number={1},
  year={2018}
}

@article{nerin2024machine,
  title={Machine learning approaches in predicting allosteric sites},
  author={Ner{\'\i}n-Fonz, Francho and Cournia, Zoe},
  journal={Current Opinion in Structural Biology},
  volume={85},
  pages={102774},
  year={2024},
  publisher={Elsevier}
}

@article{tian2023passer,
  title={{PASS}er: fast and accurate prediction of protein allosteric sites},
  author={Tian, Hao and Xiao, Sian and Jiang, Xi and Tao, Peng},
  journal={Nucleic Acids Research},
  volume={51},
  number={W1},
  pages={W427--W431},
  year={2023},
  publisher={Oxford University Press}
}

@article{shuman2013emerging,
  title={The emerging field of signal processing on graphs: Extending high-dimensional data analysis to networks and other irregular domains},
  author={Shuman, David I and Narang, Sunil K and Frossard, Pascal and Ortega, Antonio and Vandergheynst, Pierre},
  journal={IEEE signal processing magazine},
  volume={30},
  number={3},
  pages={83--98},
  year={2013},
  publisher={IEEE}
}

\newpage

\appendix

\begin{center}
    \LARGE  \textbf{Appendix}    
\end{center}

\addcontentsline{toc}{section}{Appendices}

\startcontents[appendices]
\printcontents[appendices]{l}{1}{\setcounter{tocdepth}{2}}

\newpage

\section{Extended background and related work}

\subsection{Spectral graph neural networks}

The convolution layers in many graph neural architectures can be viewed as instances of spectral graph filtering, where the convolution on a graph is defined through the eigendecomposition of the graph shift operator.
Namely, the graph convolution uses the eigenvectors of a graph shift operator as graph Fourier modes.  
Our method belongs to this broad family of spectral GNN, since the operators used by our method can be interpreted as learnable frequency responses acting on graph signals.

\paragraph{Spectral graph filters.} 
A graph shift operator (GSO) is a self-adjoint matrix that reflects the graph’s connectivity, such as a normalized graph Laplacian, an unnormalized Laplacian, or a symmetrized adjacency matrix.  
Let $\mathcal{L}$ be such a GSO with eigenpairs $\{(\lambda_i,v_i)\}_{i=1}^N$, and denote $\mathcal{L} = \mV \bm{\Lambda} \mV^\top$, where $\mV=[v_1,\ldots,v_N]$ and  $
\bm{\Lambda}=\operatorname{diag}(\lambda_1,\ldots,\lambda_N)$.
The graph Fourier transform of a node signal $\mX\in\mathbb{R}^{N\times T}$ is given by $\mV^\top \mX$, and the inverse transform is given by multiplication by $\mV$.

For a $T$-channel node signal $\mX\in\mathbb{R}^{N\times T}$ and a matrix-valued frequency response $
q:\mathbb{R}\to\mathbb{R}^{d'\times T}$, and 
the corresponding spectral graph filter is given by 
\begin{equation*}
\label{eq:spectral_filter_multichannel}
q(\mathcal{L})\mX
:=
\sum_{i=1}^N
v_i v_i^\top \mX\, q(\lambda_i)^\top .
\end{equation*}
This applies the graph analog of the convolution theorem \citep{bracewell1986fourier}: each graph Fourier mode $v_i$ is preserved in the node domain, while the channel features associated with that mode are mixed by the matrix $q(\lambda_i)$ in the spectral domain.  
Thus, $q(\lambda_i)$ determines how information at graph frequency $\lambda_i$ is amplified, suppressed, or mixed across channels.
In the scalar case, where $T=d'=1$ and $f:\mathbb{R}\to\mathbb{R}$, the filter reduces to the usual functional-calculus operator
\[
f(\mathcal{L})\mX
=
\sum_{i=1}^N f(\lambda_i)\,v_i v_i^\top \mX
=
\mV f(\Lambda) \mV^\top \mX.
\]
Hence, a scalar spectral filter acts diagonally in the graph Fourier basis: it multiplies the coefficient of the $i$-th graph Fourier mode by $f(\lambda_i)$.

\paragraph{Parameterized spectral filters.}
In a neural network, the frequency response $q$ is usually not fixed and it can be parameterized by learnable weights.  
A common choice is a polynomial filter $ q_\theta(\lambda)=\sum_{k=0}^{K} \Theta_k \lambda^k$,
where $\Theta_k\in\mathbb{R}^{d'\times T}$ are trainable channel-mixing matrices.  Then, it gives the node-domain expression $ q_\theta(\mathcal{L})\mX = \sum_{k=0}^{K}
\mathcal{L}^k \mX\Theta_k^\top$.

This form avoids explicitly computing the eigendecomposition of $\mathcal{L}$ and also gives a locality interpretation, i.e., if $\mathcal{L}$ is sparse and respects graph connectivity, then $\mathcal{L}^k\mX$ aggregates information along paths of length up to $k$.  
Chebyshev spectral networks \citep{defferrard2017convolutionalneuralnetworksgraphs} use a stable polynomial parameterization based on Chebyshev polynomials, while the graph convolutional network of \cite{kipf2017semisupervisedclassificationgraphconvolutional} can be viewed as a first-order spectral approximation with a particular normalization.  
More expressive spectral families, such as rational or Cayley filters, replace polynomial responses by richer frequency responses while retaining the same spectral-filtering principle \citep{levie2019cayleynets}.
Nonlinear spectral filters \cite{lin2024equivariant, lin2026adaptive} were further introduced to be equivariant to graph functional shifts.

\paragraph{Spectral GNN layers.}
A spectral graph neural network composes spectral filters with pointwise nonlinearities.  
Given hidden node features $\bm{H}^{(\ell)}\in\mathbb{R}^{N\times d_\ell}$ at layer $\ell$, a spectral GNN layer is given by  $ \bm{H}^{(\ell+1)} = \sigma\left( q_{\theta_\ell}(\mathcal{L})\bm{H}^{(\ell)} + \bm{1} b_\ell^\top \right)$, where $q_{\theta_\ell}:\mathbb{R}\to\mathbb{R}^{d_{\ell+1}\times d_\ell}$ is a trainable matrix-valued frequency response, $b_\ell\in\mathbb{R}^{d_{\ell+1}}$ is a bias vector, and $\sigma$ is applied pointwise.  
Stacking these layers then produces a nonlinear architecture whose linear components are graph-frequency filters.

This perspective places our method within the broad family of spectral GNNs: the learnable operators can be interpreted as frequency-selective transformations of graph signals, with the GSO determining the graph Fourier basis and the parameterized response determining how each spectral component is transformed.  
Consequently, the architecture inherits the standard spectral interpretation of graph convolution while allowing flexible channel mixing and nonlinear feature extraction across layers.

\subsection{Comparison with related unitary neural architectures}

In this section, we position our Schrödinger graph signal processing framework in relation to recent unitary neural architectures, including unitary convolutions for graphs and groups \cite{kiani2024unitaryconvolutionslearninggraphs}, graph unitary message passing \cite{qiu2024graphunitarymessagepassing}, and unitary group convolutions \cite{zhao2019buildingefficientdeepneural}, and clarify how our approach differs from these methods.

The main distinguishing difference of our method is that it is unitary. Rather, our contribution lies in the formulation of feature observables, feature-derivative Schr\"odinger dynamics, phase-induced feature momentum, and routing measure. 
Existing unitary methods mainly use unitarity to preserve norms, stabilize gradients, avoid oversmoothing
or oversquashing, or improve efficient channel mixing. Our method uses unitarity as part of a broader physical and geometric framework for measuring and controlling signal localization on graphs.

\subsubsection{Comparison with unitary convolutions on graphs and groups}

Our Schr\"odinger graph signal processing framework is related to, but different from, the unitary convolution framework \cite{kiani2024unitaryconvolutionslearninggraphs} but they use different components, optimize for different notions of signal propagation, and provide different theoretical guarantees.

\paragraph{Common principle: unitary propagation.}
\cite{kiani2024unitaryconvolutionslearninggraphs} propose replacing standard graph convolutions by unitary or orthogonal convolutions.
Their separable unitary graph convolution is $F_{\mathrm{UniConv}}(\mX) = \exp(i t \mA) \mX \mU$ with $\mU \mU^\dagger = \mI$, where $\mA$ is the graph adjacency matrix, $t\in\mathbb{R}$ is a learnable or tunable propagation time, and $\mU$ is a unitary channel-mixing matrix.
They also introduce a Lie unitary/orthogonal graph convolution: $F_{\mathrm{Lie}}(\mX) = \exp(g_{\mathrm{conv}})(\mX)$ with $ g_{\mathrm{conv}}(\mX)=\mA\mX\mW$ and $\mW+\mW^\dagger=0$ so that the vectorized generator belongs to a Lie algebra and its exponential is unitary or orthogonal.

Our method also uses unitary evolution, but the unitary operator is not generated directly by the
adjacency matrix. Instead, we construct feature-dependent graph derivatives $ (\nabla_{f_k})_{n,m} =a_{n,m}(f_k(n)-f_k(m))$, define the Schr\"odinger Laplacian $\Delta_f = -\sum_{k=1}^K \nabla_{f_k}^2$, and propagate signals by $\mathcal{S}[t,f] = \exp(-it\Delta_f)$. 

Thus, while both methods use exponentials of structure-dependent generators, our framework builds on a second-order feature-derivative operator built from learned or optimized feature locations.

\paragraph{Main conceptual difference. }
Unitary convolutions \cite{kiani2024unitaryconvolutionslearninggraphs} are primarily a stability mechanism for graph and group convolutions. Their goal is to make deep equivariant networks more stable by using norm-preserving and invertible layers. Specifically, their theory focuses on avoiding oversmoothing and avoiding
vanishing or exploding gradients.

By contrast, our method analyzes information flow in GNNs beyond oversquashing. It introduces an observable-based framework for measuring, controlling, and routing signal localization. We define feature-location
observables $X_{f_k}=\operatorname{diag}(f_k)$ and and track the expected feature location and variance of a signal: $\mathcal{E}_{X_{f_k}}(g) = \langle X_{f_k}g,g\rangle$ and $\mathcal{V}_{X_{f_k}}(g) =
\|(X_{f_k}-\mathcal{E}_{X_{f_k}}(g)I)g\|_2^2$. 
This leads to the signal routing measure $\mathcal{P}_{M}(g^{(0)},g^{(t)},r) = \frac{\langle (M-rI)^2g^{(t)},g^{(t)}\rangle} {\mathcal{V}_{M}(g^{(0)})}$, which explicitly quantifies whether the final signal is concentrated near a target feature value $r$.
This target-directed localization objective is absent from unitary convolutions.

\paragraph{Role of complex phases.} In \cite{kiani2024unitaryconvolutionslearninggraphs}, complexification is used to make operators such as $\exp(i t \mA)$ unitary and wave-like. 
In our framework, complex phases serve an additional geometric role: they create momentum in feature space. For a real signal $g$, the expected feature momentum $mathcal{E}_{i\nabla_f}(g) = \langle i\nabla_f g,g\rangle$ is zero. We therefore introduce feature modulation $D[\theta h] = \operatorname{diag}(e^{i\theta h})$, which can create nonzero expected momentum.
Thus, our method enables directed transport toward a
target feature location.

\subsubsection{Comparison with graph unitary message passing}

Graph unitary message passing (GUMP) \cite{qiu2024graphunitarymessagepassing} proposes to improve graph learning by making the adjacency operator used in message passing unitary.
Although both GUMP and our Schr\"odinger graph signal processing framework use unitarity, they use it in different places and for different purposes.

\paragraph{GUMP: unitary adjacency for stable message passing.} 
GUMP starts from the observation that standard message passing repeatedly applies powers of a graph adjacency operator: in $\mH^{(k)} = \sigma(\mA\mH^{(k-1)}\mW_k)$, the dependence of $\mH^{(L)}$ on the input contains powers of $\mA$.  
If the eigenvalues of $\mA$ have magnitude smaller or larger than one, these powers may decay or grow exponentially with depth. GUMP therefore replaces the standard adjacency matrix by a unitary adjacency matrix $\mU[\mA_w]$, giving $\mH^{(k)} = \sigma\left(\mU[\mA_w]\mH^{(k-1)}\mW_k\right)$ and $\mU[\mA_w]^\dagger \mU[\mA_w]=\mI$. 
Their main objective is to avoid graph-induced training instability by preventing the graph propagation operator from causing exponential contraction or expansion of the Jacobian.

In contrast, our method does not make the adjacency matrix itself unitary. We keep the graph adjacency as a structural object and construct feature-dependent derivative operators, then define the Schr\"odinger Laplacian and propagate by the unitary Schr\"odinger operator. 

Therefore, in GUMP, the unitary object is an adjacency-like message-passing matrix. In our method, the unitary object is a Schr\"odinger evolution generated by a self-adjoint feature-derivative Laplacian.

\paragraph{Graph transformation vs. feature-geometry construction.}
One technical challenge in GUMP is that a general graph does not naturally have a sparse unitary adjacency matrix with the same support as the original adjacency.
GUMP addresses this by transforming the graph. For an undirected graph $G=(V,E)$, it first constructs a directed graph $G'$ by replacing each undirected edge with two directed edges, and then passes to the line graph $G \longmapsto G' \longmapsto L(G')$. The transformed line graph $L(G')$ has $2|E|$ vertices, corresponding to directed edges of the original graph. GUMP then computes a weighted adjacency matrix $A_w$ on $L(G')$, projects it to a unitary matrix $\mU[\mA_w]$, performs message passing on $L(G')$.

Our framework does not require this line-graph transformation. The signal remains on the original vertex set $V$, and the graph operator is constructed directly from the original adjacency weights and the learned feature locations. Therefore, our method keeps the signal domain fixed. 

\paragraph{Unitary projection vs. exponential of a self-adjoint operator.}
GUMP obtains a unitary adjacency matrix by projecting a learned weighted adjacency matrix onto the unitary group. In its polar-projection form, this can be written as $   \mU[A_w] = \arg\min_{\mU^\dagger \mU=\mI}\|\mA_w-\mU\|_F^2 = \mA_w(\mA_w^\dagger \mA_w)^{-1/2}$. 

Our method does not project an arbitrary learned adjacency matrix to the unitary group. Instead, unitarity
comes from the spectral theorem: since $\Delta_f$ is self-adjoint, $\mathcal{S}[t,f] = e^{-it\Delta_f}$ is unitary. 
We believe this distinction is important as GUMP enforces unitarity by projecting a graph message-passing matrix, and our method derives unitarity from a Schr\"odinger-type Hamiltonian constructed from feature derivatives.

\subsubsection{Comparison with unitary group convolutions}

We next compare our method with the unitary group convolution framework \cite{zhao2019buildingefficientdeepneural}. Unitary group
convolutions are designed for efficient convolutional neural networks, especially CNNs with grouped channel convolutions. Our method is designed for graph signal propagation and feature-space routing.

\paragraph{UGConv: unitary transforms around channel group convolutions.} 
A  convolutional layer with $M$ input channels and $N$ output channels can be represented by $ y^{(j)} = \sum_{i=1}^M x^{(i)} * W^{(i,j)}$ for $1\leq j\leq N$. 
A group convolution divides channels into $G$ disjoint groups $\widetilde y^{(g,j)} = \sum_{i=1}^{M/G} \widetilde x^{(g,i)} * \widetilde W^{(g,i,j)}$ for $ 1\leq j\leq N/G$. This reduces parameters and floating-point operations, but it also removes many cross-group channel connections. UGConv addresses this by sandwiching the group convolution between two unitary channel transforms $P$ and $Q$:
\begin{align*}
    \widetilde X_k &= P X_k,
    \qquad \forall k, \\
    \widetilde y^{(g,j)}
    &=
    \sum_{i=1}^{M/G}
    \widetilde x^{(g,i)}
    *
    \widetilde W^{(g,i,j)}, \\
    Y_\ell &= Q\widetilde Y_\ell,
    \qquad \forall \ell.
\end{align*}
Here $P\in\mathbb{C}^{M\times M}$ and $Q\in\mathbb{C}^{N\times N}$ are unitary transforms applied across channels.

Our Schr\"odinger method uses unitarity in a different domain. The unitary operator $\mathcal{S}[t,f]=e^{-it\Delta_f}$ acts across graph nodes, and $\Delta_f$ is built from the graph structure and feature locations. 
Channel mixing is handled separately by matrices $\mW^{(m)}$, but the main unitary propagation is a graph-domain evolution, not a channel-domain basis change.

\paragraph{Different meaning of ``group convolution.''}
In \cite{zhao2019buildingefficientdeepneural}, ``group convolution'' refers to grouped channel convolution in CNNs: the input and output channels are partitioned into groups, and convolution is performed independently in each group. It is not a graph convolution and does not use a graph adjacency matrix. It is also not primarily about group-equivariant convolution over symmetry groups.

\paragraph{Efficiency-oriented unitarity versus physics-inspired unitarity.} 
UGConv uses unitary transforms because they preserve inner products and allow information to mix across channel groups without destroying gradient magnitudes. The paper shows that ShuffleNet \cite{zhang2018shufflenet} can be interpreted as a UGConv with a permutation matrix, and block-circulant networks (e.g., \cite{ding2017circnn}) can be interpreted as UGConvs with Fourier transforms. It then proposes Hadamard transforms as an efficient dense channel-mixing alternative, where  Hadamard matrices have $\pm1$ entries, can be computed using additions and subtractions, and avoid the
complex arithmetic of Fourier transforms.

Our use of unitarity is closer to Schr\"odinger dynamics. Since $\Delta_f$ is self-adjoint, $\mathcal{S}[t,f]$ is unitary. Moreover, our complex phases are not introduced for computationally efficient channel mixing, they are introduced to create momentum along feature coordinates $D[\theta h]g =e^{i\theta h}\odot g$. 
Thus, UGConv uses unitary transforms as efficient channel-basis changes, while our method uses unitary
evolution and modulation as interpretable graph-signal transport mechanisms.

\subsection{Relation to quantum and quantum-inspired graph neural networks}

Our use of the term ``Schr\"odinger'' is mathematical rather than hardware-oriented. The proposed model is a classical graph signal processing and graph neural network architecture whose propagation operator is motivated by the Schr\"odinger equation. 
It is important to highlight this distinction as the phrase ``quantum graph neural network'' has been used in the literature for several different families of methods. 
In short, prior quantum and quantum-inspired GNNs primarily use quantum circuits, quantum walks, or tensor-network parameterizations to define new graph learning architectures. Our work instead develops a classical Schr\"odinger graph signal processing framework in which graph signals are routed through learned feature-location coordinates using observables, momentum, commutators, and phase modulation. This distinction clarifies both the relationship to the existing literature and the technical contribution of the proposed Schr\"odinger GNN.
We next position our method relative to these families and clarify the novelty of our construction.

\paragraph{Quantum-circuit graph neural networks.} 
A first line of work studies quantum graph neural networks as quantum neural-network ans\"atze implemented by parameterized quantum circuits \cite{verdon2019quantum, tuysuz2021hybrid}.
For example, \cite{verdon2019quantum} introduces graph-structured quantum neural network ans\"atze designed to represent quantum processes with graph structure, including quantum graph recurrent neural networks and quantum graph convolutional neural networks. In their setting, the graph typically specifies interactions among quantum subsystems or the structure of a quantum computation, and node or edge information is encoded into qubits, transformed by trainable quantum gates, and read out through quantum measurements. The central design questions in this line of work are  quantum-circuit expressivity, quantum state preparation, measurement, and compatibility with near-term or distributed quantum hardware.

Our method is different in both computational substrate and mathematical objective. We do not encode graph data into qubits, do not use parameterized quantum gates, and do not perform quantum measurements. 
All operations in our framework are standard finite-dimensional linear-algebra operations on classical node features. The graph signal $g:V\rightarrow \mathbb{C}^J$ is represented explicitly as a complex-valued array over the vertices, and the unitary operator $\mathcal{S}[t,f] = \exp(-it\Delta_f)$ is evaluated as a spectral graph filter. Thus, unitarity in our model is not used to obtain quantum computation, but to obtain norm-preserving, phase-sensitive signal propagation on a graph. The relevant inductive bias is signal transport in learned feature-location coordinates.

\paragraph{Quantum-walk and quantum-propagation graph convolutions.} 
Another related line of work uses quantum walks to define graph convolution operators \cite{farhi1998quantum, aharonov2001quantum, kempe2003quantum, venegas2012quantum}.  
A quantum walk is the quantum analog of a classical random walk: instead of evolving a real-valued probability distribution by a stochastic
transition matrix, it evolves a complex-valued amplitude vector by a unitary operator. The probability of visiting a vertex is then obtained from the squared modulus of the corresponding complex amplitude.

Several graph learning methods use such quantum-walk dynamics to construct feature aggregation weights. For example, quantum spatial graph convolutional
networks \cite{bai2021learning} compute an average mixing matrix $\mQ$ from a continuous-time quantum walk and use it as the propagation matrix in a graph convolution  $\mZ=\sigma(\mQ\mX\mW)$, where $\mX$ is the node feature matrix and $\mW$ is a trainable channel-mixing matrix. Thus, the quantum-walk component primarily replaces the usual adjacency- or Laplacian-based propagation matrix by a quantum-walk-derived visiting probability matrix.

Our method is related to this family only at the broad level that both use unitary dynamics as inspiration for graph propagation. The construction and purpose are different. Quantum-walk-based graph convolutions usually derive a real-valued propagation kernel from the visiting probabilities of a walk generated by a fixed graph operator such as $\mA$ or $\mL$. In contrast, we construct a feature-dependent Schr\"odinger Laplacian $
    \Delta_f=-\sum_{k=1}^K \nabla_{f_k}^2$, $(\nabla_{f_k})_{n,m}=a_{n,m}(f_k(n)-f_k(m))$,
and propagate complex graph signals by $
    \mathcal{S}[t,f]=e^{-it\Delta_f}$.
Moreover, our architecture includes feature modulation $D[\theta h]=\operatorname{diag}(e^{i\theta h})$,
which creates controllable feature momentum, and our analysis is formulated in terms of feature-location observables $X_{f_k}$, momentum operators
$i\nabla_{f_k}$, variances, routing measures, and commutators. These elements are not present in quantum-walk graph convolutions, whose main role is to define an alternative graph aggregation kernel.

It is important to note that the components proposed by our framework, including learned feature-location coordinates using observables, momentum, commutators, and phase modulation, are not present in previous QGNN or tensor-network GNN constructions. Its role is to make multidimensional feature-location propagation more interpretable by reducing cross-direction interference.

\subsection{Sparse signal routing versus dense attention}
\label{app:comparison_transformers}
Transformers provide a natural alternative mechanism for long-range communication. 
In a standard self-attention layer, every node or token can attend to every other node or token. 
Thus, in principle, a Transformer can route information from a source region to a specific distant destination in a single layer, since the attention weights define a task-dependent soft communication pattern between all pairs of locations 
\cite{vaswani2017attention}.
Thus, Transformer-based graph models are attractive for tasks requiring long-range interactions \cite{dwivedi2020generalization, ying2021transformers, rampavsek2022recipe}. 
However, dense self-attention achieves this flexibility by explicitly considering all node pairs. 
For a graph with $N$ nodes, full attention has quadratic time and memory complexity in the number of nodes, namely $O(N^2)$ pairwise attention scores per layer, up to channel and head dimensions. 
This all-pairs communication bypasses graph bottlenecks, but it is also expensive and can be unnecessary when the relevant information flow is structured by the graph structure or by a small number of meaningful feature directions.
Schr\"odinger GNNs address the same conceptual problem, routing information between distant regions, from a different perspective. 
Rather than constructing an explicit all-pairs communication matrix, our method defines feature-location observables and uses complex modulation to induce momentum along learned formal directions. 
The Schr\"odinger operator then propagates the signal coherently along these directions. 
Thus, routing is represented as a geometry-preserving flow of localized signal content, instead of as a dense set of pairwise attention coefficients. In addition, our method is $O(|E|)$ for fixed architectural hyperparameters. 
This contrasts with the $O(N^2)$ pairwise cost of dense Transformer attention.

\section{Computational realization of Schr\"odinger GNNs}
\label{subsec:arch_details}

In this appendix, we explain how the Schr\"odinger graph signal processing introduced in Section~\ref{sec:Schrodinger_gsp} is instantiated in practice.
One important note is that the model does not require explicit construction of a dense matrix exponential.
Instead, it applies sparse feature-dependent operators to node features and uses a truncated series to approximate the action of the Schr\"odinger propagator.

\subsection{Approximating the unitary propagation}

In the implementation, we approximate $\mathcal S[t,f]=e^{-it\Delta_f}$ on feature tensors rather than forming the dense matrix exponential explicitly.
For any operator $B$, we have
\begin{equation*}
    e^{B}=\sum_{r=0}^{\infty}\frac{B^{r}}{r!},
\end{equation*}
thus, the order-$R$ truncated Taylor approximation is given by 
\begin{equation*}
    e^{-it\Delta_f}X
\approx
\sum_{r=0}^{R}\frac{(-it\Delta_f)^{r}}{r!}\,X.
\end{equation*}
A practical advantage of this formulation is that one can apply $\Delta_f$ without explicitly constructing the matrix $\Delta_f$. Specifically, for the multi-feature operator, we have 
\begin{equation*}
    \Delta_f X = -\sum_{k=1}^{K}\nabla_{f_k}\bigl(\nabla_{f_k}X\bigr).
\end{equation*}
Each application of $\nabla_{f_k}$ is an edge-local sparse operation, so one evaluates $\Delta_f X$ by a sequence of sparse passes over the graph. This is more efficient for the actual implementation than a dense matrix-exponential routine. 

\paragraph{Approximation error introduced by the truncated Taylor expansion.}
We remark that the implementation applies the action of $\mathcal{S}[t,f]=e^{-it\Delta_f}$ through a truncated Taylor series, and  it does not construct a dense exponential. We evaluate the approximation on trained propagation operators at the truncation orders and find that: (i) with the order-10 approximation $\tilde{\mathcal{S}}_{10}$ of the unitary Schrödinger operator $\mathcal{S}$, the Taylor remainder satisfies $||\tilde{\mathcal{S}}_{10} - \mathcal{S} ||_{\mathrm{op}} \leq 4.44 \times 10^{-3}$, and $||\tilde{\mathcal{S}}_{10}^{*}\tilde{\mathcal{S}}_{10} - I||_{\mathrm{op}} \leq 8.90 \times 10^{-3}$, where the singular values of the approximate propagation operator lie in [0.99999794, 1.00000206], and (ii) with the order-15 approximation $\tilde{\mathcal{S}}_{15}$, the Taylor remainder satisfies $||\tilde{\mathcal{S}}_{15} - S ||_{\mathrm{op}} \leq 2.06 \times 10^{-6}$, and $||\tilde{\mathcal{S}}_{15}^{*}\tilde{\mathcal{S}}_{15} - I||_{\mathrm{op}} \leq 4.12 \times 10^{-6}$, where the singular values of the approximate propagation operator lie in [0.99999794, 1.00000206]. Hence, near-unitarity, routing observables, and propagation stability are preserved to within the reported numerical tolerances after Taylor truncation.

\subsection{Layer-level realization of the Schr\"odinger GNN}

Let $\mF\in\mathbb R^{N\times K_f}$ be the matrix of formal-location features, with $\mF_{:,k}=f_k$, and let $\mX\in\mathbb C^{N\times J}$ be the current complex signal matrix. A Schr\"odinger filter with $M$ filter terms outputs $\mY\in\mathbb C^{N\times D}$ by 
\begin{equation*}
    \mY
=
\sum_{m=1}^{M}
\mathcal S[t_m,\mathbf{f}]
\,D[\theta_m \mF\mT^{(m)}]
\,\mX\,\mW^{(m)}.
\end{equation*}
Here $t_m,\theta_m\in\mathbb R$ are scalar propagation and modulation parameters, $\mT^{(m)}\in\mathbb R^{K_f\times 1}$ chooses a direction in the formal-location feature space, $D[\theta_m \mF\mT^{(m)}]$ is the diagonal phase-modulation operator with diagonal entries $\exp\bigl(i\theta_m(\mF\mT^{(m)})_n\bigr)$ for $ n\in[N]$, and $\mW^{(m)}\in\mathbb C^{J\times D}$ mixes feature channels. The order of operations is: phase-modulate the signal, propagate it with the unitary Schr\"odinger operator, and then mix channels. In such construction, the modulation creates directional momentum, while $\mathcal S[t_m,\mathbf{f}]$ transports the signal along the formal locations. 

\subsection{Complex input modulation}

The raw input node-feature matrix is denoted by $\mQ\in\mathbb R^{N\times d_{\rm in}}$. Before applying Schr\"odinger propagation, we map $\mQ$ to a complex signal matrix by an amplitude--phase factorization,
\begin{equation*}
    \widetilde{\mX} = (\mQ\mB)\odot \exp\bigl(i\,\mQ\mP\bigr) \in\mathbb C^{N\times d},
\end{equation*}
where $\mB,\mP\in\mathbb R^{d_{\rm in}\times d}$ are learned linear maps and $\odot$ denotes entrywise multiplication. The amplitude term $\mQ\mB$ carries the real-valued content, while the phase term $\exp(i\mQ\mP)$ creates oscillations that can induce nonzero expected momentum under the Schr\"odinger dynamics. This conversion is not a dimensionality-reducing projection by itself, but it represents each learned channel by a complex amplitude whose modulus and phase are both trainable functions of the original features.

\subsection{Position-momentum optimization}
When Position-momentum optimization (PMO) is used, it is run before training the main prediction network as a preprocessing step. 

Given the raw candidate formal-location features $\mQ\in\mathbb R^{N\times M}$, PMO learns a linear map $\mT\in\mathbb R^{M\times K_f}$ and defines $\mF=\mQ\mT$.
The objective in Definition~\ref{def:minimization} penalizes cross-commutators such as $[\nabla_{f_j}^2,X_{f_i}]$ and includes the stated regularization terms. In practice, we optimize $\mT$ with Adam over the training graphs, using mini-batches when the dataset contains multiple graphs. After this preprocessing stage, $\mT$ is fixed, and the resulting formal-location matrix $\mF$ is used by the Schr\"odinger GNN during training and inference. Thus, PMO is an additional preprocessing cost, not an additional matrix-exponential cost at every layer.

\subsection{Nonlinearities and dropout}
After the modulation layer, the network operates on complex-valued node features.
We apply the nonlinearities componentwise \cite{bassey2021survey}, given by 
\begin{equation*}
    \sigma(z)=\mathrm{ReLU}(\mathrm{Re}(z)) + i\,\mathrm{ReLU}(\mathrm{Im}(z)),
\end{equation*}
and we found that it works well in practice. We also explore magnitude $\sigma(z)=|z|$. 
The componentwise choice keeps the representation in the complex domain, and  the magnitude is used at the final stage when the downstream task requires real-valued outputs.

\subsection{Computational complexity}

Let $R$ denote the Taylor truncation order, let $K$ be the number of feature-location directions used to build $\Delta_f$, and let $C$ be the hidden feature dimension. One application of $\nabla_{f_k}$ to a feature tensor costs $O(|E|C)$, since it is an edge-local sparse operation. Therefore, one application of $\nabla_{f_k}^{2}$ also costs $O(|E|C)$ up to a constant factor, and one application of $\Delta_f = -\sum_{k=1}^{K}\nabla_{f_k}^{2}$ costs $O(K|E|C)$. An order-$R$ truncated Taylor approximation costs $O(RK|E|C)$ per layer. For fixed $R$ and $K$, this remains linear in the number of edges.

The memory cost is dominated by the feature tensors, the sparse graph structure, and the stored feature-location channels. Note that we do not compute the dense matrix exponential, which keeps memory usage comparable to other sparse message-passing architectures. In our experiments, we found that values around $R\approx 10\text{--}15$ gave a good balance between accuracy and cost.
The memory complexity is $O((|V| + |E|)C)$, similar to standard GNNs, as we do not explicitly construct the dense matrix exponential.

\subsection{Propagation time initialization}
Each output channel has a learnable real propagation-time parameter. If a layer has $C$ output channels, we initialize 
\begin{equation*}
    t_j\sim \operatorname{Uniform}(0,1.5),
\qquad j=1,\ldots,C.
\end{equation*}
Small values of $t_j$ produce local filters, while larger values allow a channel to aggregate information from farther regions through the Schr\"odinger propagation. This initialization gives different channels different initial transport ranges. When learnable propagation times are disabled, we use a fixed scalar time value instead.

\section{Theoretical results and proofs}\
\label{app:proof}

We present the theoretical results and proof for Schr\"odinger graph signal processing in Section \ref{sec:Schrodinger_gsp}.
To support these proofs, we also include several additional propositions and lemmas, which are numbered separately for clarity.
We restate the claim of each statement for convenience.

\subsection{Background in functional calculus}

We recall the finite-dimensional functional calculus for self-adjoint operators. 
Let $\mM \in \CC^{N \times N}$ be self-adjoint and let $\mM = \mV \bm{\Lambda}\mV^*$ be its eigendecomposition, where $\mV$ is a unitary matrix and $\bm{\Lambda}$ is a diagonal matrix with elements $\operatorname{diag}(\lambda_1, \ldots, \lambda_N)$.

For any function $y: \R \rightarrow \CC$ defined on the spectrum $\{\lambda_1, \ldots, \lambda_N\}$, we have $y(\mM) = \mV y(\bm{\Lambda})\mV^*$, where $y(\bm{\Lambda}) = \operatorname{diag}(y(\lambda_1), \ldots, y(\lambda_N))$. 
If $\vv_j$ is an orthonormal eigenbasis of $\mM$ with $\mM \vv_j =\lambda_j \vv_j$, then 
$y(\mM)\vv_j = y(\lambda_j)\vv_j$. 
Thus, $y(\mM)$ is diagonal in an orthonormal basis, and hence is a normal operator. 
In addition, if $y$ maps to $\R$ into the unit circle $\{z\in \mathbb{C}: |z|=1\}$, then $y(\mM)$ is unitary.

\begin{lemma}
\label{lem:func_calc_commute}
    Let $\mM,\mR\in\mathbb{C}^{N\times N}$ be commuting self-adjoint operators (i.e., $\mM\mR=\mR\mM$). Then, for every function $f:\mathbb{R}\rightarrow\mathbb{C}$, $M$ commutes with $y(\mR)$.
\end{lemma}

\subsection{Proof of Theorem \ref{thm:expected_momentum_invariant}}

\textbf{Theorem \ref{thm:expected_momentum_invariant}} (Constant expected momentum)\textbf{.}
\textit{Let $g^{(0)}:V\rightarrow\mathbb{C}$ be a normalized signal and $f:V\rightarrow\mathbb{R}$ a feature location. Then, for every $t\in\mathbb{R}$, 
\[
\mathcal{E}_{i\nabla_{f}}(g^{(t)})=\mathcal{E}_{i\nabla_{f}}(g^{(0)}).
\]}

\begin{proof}
Let $\Delta_f = - \nabla_f^2$ and $\mathcal{S}[t,f]=e^{-it\Delta_f}$. Since $\Delta_f$ is self-adjoint, and $e^{-it(\cdot)}$ maps to the complex unit circle, $\mathcal{S}[t,f]$ is unitary. Moreover, since $\Delta_f$ commutes with $\nabla_f$, by Lemma \ref{lem:func_calc_commute}  $\nabla_f$ commutes with $\mathcal{S}[t,f]$. 
Hence, 
\begin{equation*}
    \begin{aligned}
        \mathcal{E}_{i\nabla_{f}}(g^{(t)})
        & = \ip{i \nabla_{f} \mathcal{S}[t,f] g^{(0)}}{ \mathcal{S}[t,f] g^{(0)}} \\
        & = \ip{  \mathcal{S}[t,f] i\nabla_{f} g^{(0)}}{ \mathcal{S}[t,f] g^{(0)}} \\
        & = \ip{i \nabla_{f} g^{(0)} }{ g^{(0)}} = \mathcal{E}_{i\nabla_{f}}(g^{(0)}). 
    \end{aligned}
\end{equation*}

\end{proof}

\subsection{Proof of Theorem \ref{thm:derivative_expected_feature}}
\label{app:proof_of_thm:derivative_expected_feature}

Before proving Theorem \ref{thm:derivative_expected_feature}, we develop some analysis for the smoothing operator.

\begin{lemma}
[Smoothing operator as commutator]
\label{lem:smooth_comm}
\begin{equation*}
    W_f = [X_f,\nabla_f] \;=\; X_f \nabla_f - \nabla_f X_f \;=\; -[\nabla_f,X_f].
\end{equation*}
\end{lemma}

\begin{proof}

For any signal $g$ and vertex $v$, we have
\begin{equation*}
    \begin{aligned}
        ([X_f,\nabla_f]g)(v) 
        & = (X_f \nabla_f g)(v) - (\nabla_f X_f g)(v)\\
        &= f(v)\sum_{w \in V} a_{v,w}(f(v) - f(w)) g(w) - \sum_{w\in V} a_{v,w}(f(v) - f(w)) f(w) g(w)\\
&= \sum_{w\in V} a_{v,w}(f(v) - f(w))^2 g(w) \\
&= (W_f g)(v).
    \end{aligned}
\end{equation*}
Hence, $[X_f,\nabla_f]=W_f$, and equivalently $[\nabla_f,X_f]=-W_f$. 

\end{proof}

\begin{lemma}
[Commutator of Schr\"odinger Laplacian and location observable]
\label{lemma:commutator_expansion}
For the Schr\"odinger Laplacian $\Delta = -\nabla_f^2$ and location observable $X_f$, we have
\[
[\Delta, X_f] = \nabla_f W_f + W_f \nabla_f
\]
where $W_f = [\nabla_f, X_f]$ is the $f$-smoothing operator.
\end{lemma}

\begin{proof}
Using the product rule for commutators $[AB, C] = A[B,C] + [A,C]B$, and by Lemma \ref{lem:smooth_comm}, we have
\begin{align*}
[\Delta, X_f] 
&= [-\nabla_f^2, X_f] \\
&= -[\nabla_f \nabla_f, X_f]\\
&= -\nabla_f[\nabla_f, X_f] - [\nabla_f, X_f]\nabla_f\\
& = -\nabla_f(-W_f) - (-W_f)\nabla_f\\
&= \nabla_f W_f + W_f \nabla_f
\end{align*}
\end{proof}

Now we prove Theorem \ref{thm:derivative_expected_feature}.

\textbf{Theorem \ref{thm:derivative_expected_feature}} (Expected feature location derivative under Schr\"odinger dynamics)\textbf{.}
\textit{Let $g^{(0)}:V\rightarrow\mathbb{C}$ be a normalized signal and $f:V\rightarrow\mathbb{R}$ a feature location. Let $g^{(t)}=\mathcal{S}[t,f]g^{(0)}$ for every $t\in\mathbb{R}$. Then, 
\begin{equation*}
\frac{\partial}{\partial t}\mathcal{E}_{X_f}(g^{(t)})=2 \mathrm{Re}\left(\langle i\nabla_f g^{(t)}, W_f g^{(t)}\rangle\right).  
\end{equation*}
}

\begin{proof}
\fix{Note that by definition of $\mathcal{E}_{X_f}(g^{(t)}) = \ip{X_fg^{(t)}}{g^{(t)}}$, we have 
\begin{equation*}
    \frac{\partial}{\partial t}\mathcal{E}_{X_f}(g^{(t)}) =  \frac{\partial}{\partial t}\ip{X_fg^{(t)}}{g^{(t)}} 
\end{equation*}
and since $g^{(t)}=\mathcal{S}[t,f]g^{(0)}$, by the chain rule 
}, 
we have 
\begin{equation*}
\begin{aligned}
    \frac{\partial}{\partial t}\mathcal{E}_{X_f}(g^{(t)})
    & = \ip{X_f(-i\Delta_f)g^{(t)}}{g^{(t)}} + \ip{X_fg^{(t)}}{(-i\Delta_f)g^{(t)}}\\ 
    & = \ip{i[\Delta_f,X_f]g^{(t)}}{g^{(t)}}
\end{aligned}
\end{equation*}
Substituting $\Delta_f=-\nabla_f^2$ and using Lemma~\ref{lemma:commutator_expansion} yields
\begin{align*}
 \frac{\partial}{\partial t}\mathcal{E}_{X_f}(g^{(t)}) 
 &= \langle i\nabla_f W_f g^{(t)}, g^{(t)}\rangle + \langle W_f i\nabla_f g^{(t)}, g^{(t)}\rangle\\
&= \langle W_f g_t,i\nabla_f g_t\rangle + \langle i\nabla_f g^{(t)},W_f g^{(t)}\rangle \\
&= \overline{\langle i\nabla_f g^{(t)},W_f g^{(t)}\rangle} + \langle i\nabla_f g^{(t)},W_f g^{(t)}\rangle \\
&= 2\,\mathrm{Re}\left(\langle i\nabla_f g^{(t)}, W_f g^{(t)}\rangle\right).
\end{align*}
\end{proof}

The previous theorem still contains the smoothed signal $W_fg^{(t)}$. The next definition isolates the regime in which $W_fg^{(t)}$ is close to $g^{(t)}$, so that the derivative of the expected feature is close to the exact momentum.

\begin{definition}[$\epsilon$-$f$ regular signal]
\label{app:epsilon_regular}
Let $\gG = (V, E)$ be a graph, $f: V \to \mathbb{R}$ be a feature location, and $W_f$  the  corresponding $f$-smoothing operator. A signal $g: V \to \mathbb{C}$ is called $\epsilon$-$f$ regular if there exists a signal $e_g:V \to \mathbb{R}$ such that
$W_f g = g + e_g$ with $\|e_g\|_2 \leq \epsilon$. 
\end{definition}

The next theorem quantifies the extent to which  the the exact expected momentum approximates the rate of change of the expected feature location,  when the signal is $\epsilon$-$f$ regular.

\begin{theorem}[Expected feature location dynamics for $\epsilon$-$f$ regular signals]
\label{thm:deviation_bounds_single_feature} 
Let $g^{(0)}: V \to \mathbb{C}$ be normalized  and $g^{(t)}=\mathcal{S}[t,f]g^{(0)}$. 
If $g^{(t)}$ is $\epsilon$-$f$ regular, then
\[
\left| \frac{\partial}{\partial t}\mathcal{E}_{X_f}(g^{(t)}) - 2 \mathcal{E}_{i\nabla_f}(g^{(0)}) \right| \leq 2\epsilon \|\nabla_f\|_{\rm op}. 
\]
\end{theorem}
\begin{proof}
Recall from Theorem \ref{thm:derivative_expected_feature} that
\begin{equation}
    \label{eq:e-f-_bound0}
    \frac{\partial}{\partial t}\mathcal{E}_{X_f}(g^{(t)})=2 \mathrm{Re}\left(\langle i\nabla_f g^{(t)}, W_f g^{(t)}\rangle\right).
\end{equation}
By the $\epsilon$--$f$ regularity assumption there exists $e_{g^{(t)}}$ with $\lVert e_{g^{(t)}}\rVert_2\le \epsilon$ such that
$W_f g^{(t)} = g^{(t)} + e_{g^{(t)}}$. Using this identity together with  (\ref{eq:e-f-_bound0}) gives
\begin{align*}
\left|\frac{\partial}{\partial t}\mathcal{E}_{X_f}(g^{(t)}) - 2\,\mathcal{E}_{i\nabla_f}(g^{(t)})  \right| 
&= \left| 2 \,{\rm Re}\Big(\langle i\nabla_f  g^{(t)},e_{g^{(t)}}\rangle\Big) \right| \\
&\le 2\, \| i\nabla_f g^{(t)}\|_2\,\| e_{g^{(t)}}\|_2 \\
&\le 2\epsilon \,\|\nabla_f\|_\mathrm{op}\,\|g^{(t)}\|_2 \\
& = 2\epsilon \,\|\nabla_f\|_\mathrm{op}.
\end{align*}
Hence, by Theorem~\ref{thm:expected_momentum_invariant}, we have 
\begin{equation*}
    \left| \frac{\partial}{\partial t}\mathcal{E}_{X_f}(g^{(t)}) - 2\mathcal{E}_{i\nabla_f}(g^{(0)}) \right| \leq 2\epsilon \|\nabla_f\|_{\mathrm{op}}.
\end{equation*}
\end{proof}

\subsection{Dynamics of the variance}
\label{Dynamics of the Variance}

Next, we derive the dynamics of the variance.
\begin{theorem}[Time derivative of variance]
\label{thm:var_evolution}
Let $g:V\rightarrow\mathbb{C}$ be a signal and $f:V\rightarrow\mathbb{R}$ a feature location, and $\Delta_f = -\nabla_f^2$. The first-order derivative of variance with respect to time $t\in\mathbb{R}$ is 
\[
\frac{\partial}{\partial t}\mathcal{V}_{X_f}(g^{(t)}) = \mathcal{E}_{i[\Delta_f,X_f^2]}(g^{(t)}) -4\mathcal{E}_{X_f}(g^{(t)})\mathrm{Re}\big(\langle i\nabla_f g^{(t)}, W_f g^{(t)}\rangle\big).
\]
\end{theorem}
This mirrors the classical Schr\"odinger equation dynamics where variance evolution depends on both the commutator $[\Delta,X^2]$ 
and the coupling between position and momentum. See Appendix \ref{sec:Schrodinger_classic} for the classical correspondence.

\begin{proof}[Proof of Theorem \ref{thm:var_evolution}]
    By the definition of variance, we have 
    \begin{equation*}
        \mathcal{V}_{X_f}(g^{(t)})=\mathcal{E}_{X_f^2}(g^{(t)})-\mathcal{E}_{X_f}(g^{(t)})^2. 
    \end{equation*}
    Differentiating with respect to $t$ then gives
    \begin{equation*}
        \frac{\partial}{\partial t}\mathcal{V}_{X_f}(g^{(t)})
=
\frac{\partial}{\partial t}\mathcal{E}_{X_f^2}(g^{(t)})
-
2\mathcal{E}_{X_f}(g^{(t)})\frac{\partial}{\partial t}\mathcal{E}_{X_f}(g^{(t)}).
    \end{equation*}
From the time evolution of the expected feature for every observable, \update{and by definition of $\mathcal{E}_{X_f^2}(g^{(t)}) = \ip{X_f^2g^{(t)}}{g^{(t)}}$, we have $\frac{\partial}{\partial t}\mathcal{E}_{X_f^2}(g^{(t)}) =  \frac{\partial}{\partial t}\ip{X_f^2g^{(t)}}{g^{(t)}} $.
Now, by the chain rule and the fact that $g^{(t)}=\mathcal{S}[t,f]g^{(0)}$, we have 
\begin{equation*}
\begin{aligned}
    \frac{\partial}{\partial t}\mathcal{E}_{X_f^2}(g^{(t)})
    & = \ip{X_f^2(-i\Delta_f)g^{(t)}}{g^{(t)}} + \ip{X_fg^{(t)}}{(-i\Delta_f)g^{(t)}}\\ 
    & = \ip{i[\Delta_f,X_f^2]g^{(t)}}{g^{(t)}} \\
    & =\mathcal{E}_{i[\Delta_f,X_f^2]}(g^{(t)})
\end{aligned}
\end{equation*}
}

In addition, by Theorem~\ref{thm:derivative_expected_feature}, we have
\begin{equation*}
    \frac{\partial}{\partial t}\mathcal{E}_{X_f}(g^{(t)})=
2\,\mathrm{Re}\left(\langle i\nabla_f g^{(t)},W_fg^{(t)}\rangle\right).
\end{equation*}
By substituting both expressions, we obtain
\begin{equation*}
    \frac{\partial}{\partial t}\mathcal{V}_{X_f}(g^{(t)}) = \mathcal{E}_{i[\Delta_f,X_f^2]}(g^{(t)}) -4\mathcal{E}_{X_f}(g^{(t)})\mathrm{Re}\big(\langle i\nabla_f g^{(t)}, W_f g^{(t)}\rangle\big).
\end{equation*}
\end{proof}

\subsection{Proof of Theorem \ref{thm:non_zero_expected_momentum_modulated_signals}}

\textbf{Theorem \ref{thm:non_zero_expected_momentum_modulated_signals}} (Expected momentum of modulated signal)\textbf{.}
\textit{Given a  normalized signal $g:V\rightarrow\mathbb{R}$,  feature locations $f,h:V\rightarrow\mathbb{R}$, and a  phase $\theta\in\mathbb{R}$, the expected momentum of  $D[\theta h]g$  satisfies
\begin{equation*}
\mathcal{E}_{i\nabla_f}(D[\theta h]g) = \sum_{(m,n) \in E} a_{m,n} g(m) g(n) (f(n) - f(m)) \sin(\theta(h(n) - h(m))).
\end{equation*}
}

\begin{proof}
For the modulated signal ${D_{\theta h}}g(v) = g(v) e^{i\theta h(v)}$, we have
\[
(\nabla_f {D_{\theta h}}g) (m)= \sum_{n \in V} a_{m,n}\bigl(f(m)-f(n)\bigr) g(n)e^{i\theta h(n)}.
\]
Therefore
\begin{align*}
\mathcal{E}_{i\nabla_f}(D_{\theta h}g) 
&= \langle i\nabla_f D_{\theta h}g, D_{\theta h}g \rangle \\
&= i \sum_{m \in V} \sum_{n \in V} a_{m,n} g(m) g(n)e^{i\theta(h(n) - h(m))} (f(m) - f(n)).
\end{align*}
Because the graph is undirected, we may group the two orientations of each edge. Namely, 

\begin{align*}
& \quad \quad \mathcal{E}_{i\nabla_f}(D_{\theta h}g) \\
& = \frac{i}{2} \sum_{(m,n) \in E} a_{m,n} g(m) g(n) \left[e^{i\theta(h(n) - h(m))}(f(m) - f(n)) + e^{i\theta(h(m) - h(n))}(f(n) - f(m))\right] \\
&= \frac{i}{2} \sum_{(m,n) \in E} a_{m,n} g(m) g(n) (f(n) - f(m))\left[e^{-i\theta(h(n) - h(m))} - e^{i\theta(h(n) - h(m))}\right] \\
&= \frac{i}{2} \sum_{(m,n) \in E} a_{m,n} g(m) g(n) \bigl(f(n) - f(m))\bigl(-2i)\sin(\theta(h(n) - h(m))) \\
&= \sum_{(m,n) \in E} a_{m,n} g(m) g(n) (f(n) - f(m))\sin(\theta(h(n) - h(m))).
\end{align*}

\end{proof}

\subsection{Proof of Theorem \ref{thm:derivative_expected_multi_feature}}

\textbf{Theorem \ref{thm:derivative_expected_multi_feature}} (Expected multi-feature derivative)\textbf{.}
\textit{Given the Schr\"odinger Laplacian $\Delta_f = -\sum_{k \in [K]}\nabla_{f_k}^2$, a normalized signal $g^{(0)}$, and $g^{(t)}=\mathcal{S}[t,f]g^{(0)}$, we have
\begin{equation*}
    \frac{\partial}{\partial t}\mathcal{E}_{X_{f_k}}(g^{(t)})=  
     2 \mathrm{Re}\left(\langle i\nabla_{f_k} g^{(t)}, W_{f_k} g^{(t)}\rangle\right) -\sum_{j\neq k} \ip{[i\nabla_{f_j}^2,X_{f_k}]g^{(t)}}{g^{(t)}}.
\end{equation*}
}

\update{
\begin{proof}
    By definition of $\mathcal{E}_{X_{f_k}}(g^{(t)}) = \ip{X_{f_k}g^{(t)}}{g^{(t)}}$, we have 
\begin{equation*}
    \frac{\partial}{\partial t}\mathcal{E}_{X_{f_k}}(g^{(t)}) =  \frac{\partial}{\partial t}\ip{X_fg^{(t)}}{g^{(t)}} 
\end{equation*}
and since $g^{(t)}=\mathcal{S}[t,f]g^{(0)}$, by the chain rule, we have 
\begin{equation*}
\begin{aligned}
    \frac{\partial}{\partial t}\mathcal{E}_{X_{f_k}}(g^{(t)})
    & = \ip{X_{f_k}(-i\Delta_{f})g^{(t)}}{g^{(t)}} + \ip{X_{f_k}g^{(t)}}{(-i\Delta_{f})g^{(t)}}\\ 
    & = \ip{i[\Delta_f,X_{f_k}]g^{(t)}}{g^{(t)}}.
\end{aligned}
\end{equation*}
Since $\Delta_f=-\sum_{j=1}^K \nabla_{f_j}^2$, we have
    \begin{equation*}
        \frac{\partial}{\partial t}\mathcal{E}_{X_{f_k}}(g^{(t)}) = -\sum_{j=1}^K \langle [i\nabla_{f_j}^2,X_{f_k}]g^{(t)},g^{(t)}\rangle.
    \end{equation*}
    We split the sum into the $j=k$ term and the cross terms.
    When $j=k$, using Lemma~\ref{lemma:commutator_expansion} gives 
    \begin{align*}
 \frac{\partial}{\partial t}\mathcal{E}_{X_{f_k}}(g^{(t)}) 
 &= \langle i\nabla_{f_k} W_f g^{(t)}, g^{(t)}\rangle + \langle W_{f_k} i\nabla_{f_k} g^{(t)}, g^{(t)}\rangle\\
&= \langle W_{f_k} g_t,i\nabla_{f_k} g_t\rangle + \langle i\nabla_{f_k} g^{(t)},W_{f_k} g^{(t)}\rangle \\
&= \overline{\langle i\nabla_{f_k} g^{(t)},W_{f_k} g^{(t)}\rangle} + \langle i\nabla_{f_k} g^{(t)},W_{f_k} g^{(t)}\rangle \\
&= 2\,\mathrm{Re}\left(\langle i\nabla_{f_k} g^{(t)}, W_{f_k} g^{(t)}\rangle\right).
\end{align*}
Therefore, we have
    \begin{equation*}
        \ip{i[\Delta_{f_k},X_{f_k}]g^{(t)}}{g^{(t)}}  = 2 \mathrm{Re}\left(\langle i\nabla_{f_k} g^{(t)}, W_{f_k} g^{(t)}\rangle\right).
    \end{equation*}
     Keeping the remaining indices explicit gives
    \begin{equation*}
        \frac{\partial}{\partial t}\mathcal{E}_{X_{f_k}}(g^{(t)})=  
     2 \mathrm{Re}\left(\langle i\nabla_{f_k} g^{(t)}, W_{f_k} g^{(t)}\rangle\right) -\sum_{j\neq k} \ip{[i\nabla_{f_j}^2,X_{f_k}]g^{(t)}}{g^{(t)}}.
    \end{equation*}
\end{proof}
}

The cross terms in Theorem~\ref{thm:derivative_expected_multi_feature} vanish when $i\nabla_{f_j}^2$ commute with $X_{f_k}$ for $j\neq k$.  
The next result quantifies the deviation from the single feature dynamics when commutation holds only approximately and the signal is $\epsilon$-$f_k$ regular.

\begin{theorem}[Dynamics of multi-channel expected feature location for $\epsilon$-$f$ regular signals]
\label{thm:deviation_bounds_multi_feature}
Let $\Delta_f=-\sum_{j\in[K]}\nabla_{f_j}^2$,  $g^{(t)}=\mathcal{S}[t,f]g=e^{-it\Delta_f}g^{(0)}$ with $\|g^{(0)}\|_2=1$, and  $k\in[K]$. Suppose that $g^{(t)}$ is $\epsilon$-$f_k$ regular and that the feature locations $\{f_1,\ldots,f_K\}$ form a $\delta$-commuting in the sense of Definition~\ref{def:epsilon-commuting features}.
Then,
\begin{equation*}
    \left| \frac{\partial}{\partial t}\mathcal{E}_{X_{f_k}}(g^{(t)}) - 2\mathcal{E}_{i\nabla_{f_k}}(g^{(t)}) \right|
\leq
2\epsilon \|\nabla_{f_k}\|_{\rm op}
+
(K-1)\delta.
\end{equation*}
\end{theorem}
\begin{proof}
    By Theorem~\ref{thm:derivative_expected_multi_feature}, we have 
    \begin{equation*}
    \frac{\partial}{\partial t}\mathcal{E}_{X_{f_k}}(g^{(t)})= 2 \mathrm{Re}\left(\langle i\nabla_{f_k} g^{(t)}, W_{f_k} g^{(t)}\rangle\right) - \sum_{j\neq k} \ip{[i\nabla_{f_j}^2,X_{f_k}]g^{(t)}}{g^{(t)}}.
\end{equation*}
\update{
Let $A_k(t):=2 \mathrm{Re}\big(\langle i\nabla_{f_k} g^{(t)}, W_{f_k} g^{(t)}\rangle\big)$ and $C_{k,j}(t):=\ip{[i\nabla_{f_j}^2,X_{f_k}]g^{(t)}}{g^{(t)}}$ for $j\neq k$, we have 
\begin{equation*}
    \frac{\partial}{\partial t}\mathcal{E}_{X_{f_k}}(g^{(t)}) = A_k(t)-\sum_{j\neq k} C_{k,j}(t),
\end{equation*}
and therefore, 
\begin{equation*}
    \left| \frac{\partial}{\partial t}\mathcal{E}_{X_{f_k}}(g^{(t)}) - 2\mathcal{E}_{i\nabla_{f_k}}(g^{(t)}) \right| \le \left|A_k(t)-2\mathcal{E}_{i\nabla_{f_k}}(g^{(t)})\right| +\sum_{j\neq k}|C_{k,j}(t)|.
\end{equation*}
Because $g^{(t)}$ is $\epsilon$-$f_k$ regular, there exists a signal $e_t$ such that $W_{f_k}g^{(t)}=g^{(t)}+e_t$ and $\|e_t\|_2\le \epsilon$. 
Since $i\nabla_{f_k}$ is self-adjoint and $\mathcal{E}_{i\nabla_{f_k}}(g^{(t)})=\langle i\nabla_{f_k}g^{(t)},g^{(t)}\rangle$, we have 
\begin{align*}
A_k(t)-2\mathcal{E}_{i\nabla_{f_k}}(g^{(t)})
&=2\mathrm{Re}\big(\langle i\nabla_{f_k}g^{(t)},W_{f_k}g^{(t)}\rangle\big)
-2\langle i\nabla_{f_k}g^{(t)},g^{(t)}\rangle \\
&=2\mathrm{Re}\big(\langle i\nabla_{f_k}g^{(t)},W_{f_k}g^{(t)}-g^{(t)}\rangle\big) \\
&=2\mathrm{Re}\big(\langle i\nabla_{f_k}g^{(t)},e_t\rangle\big).
\end{align*}
Therefore, 
\begin{equation*}
    \left|A_k(t)-2\mathcal{E}_{i\nabla_{f_k}}(g^{(t)})\right|
\le 2\|i\nabla_{f_k}g^{(t)}\|_2\,\|e_t\|_2
\le 2\epsilon\|\nabla_{f_k}\|_{\rm op}.
\end{equation*}
For each $j\neq k$, the $\delta$-commuting assumption gives $\|[\nabla_{f_j}^2,X_{f_k}]\|_{op}\le \delta$. Hence, 
\begin{equation*}
    |C_{k,j}(t)|
\le \|[i\nabla_{f_j}^2,X_{f_k}]\|_{\rm op}\,\|g^{(t)}\|_2^2
= \|[\nabla_{f_j}^2,X_{f_k}]\|_{\rm op}
\le \delta.
\end{equation*}
By summing over $j\neq k$, then we have $\sum_{j\neq k}|C_{k,j}(t)|\le (K-1)\delta$. 
}
\end{proof}

\subsection{Improving signal routing through modulation}
\label{Improving Signal Routing Through Modulation A}

Here, we show that in typical situations modulating real-valued signals improves their signal routing measure. Consider the following setting. We have a multilayer network where at each layer $l$ we have a real-valued signal $g^{(l)}$ that we are allowed to modulate by choosing the free parameter $\theta_l\in\mathbb{R}$ in  $D[\theta_l h]g^{(l)}$. We then propagate the signal via $\mathcal{S}[dt,f]D[\theta_l h]g^{(l)}$ for some small time step $dt$, and lastly apply a modulus nonlinearity to define the signal at the next layer $g^{(l+1)}=\abs{\mathcal{S}[dt,f]D[\theta_l h]g^{(l)}}$. Here, we can interpret $g^{(l)}$ as the signal at time $ldt$, and the input to the network $g^{(0)}$ as the signal at time $0$.

Suppose that we would like to route the signal to the feature location $r$, i.e., we would like  $\mathcal{P}_{X_f}(g^{(0)},D[\theta_l h]g^{(l)},r)$ to decrease in $l$ by choosing appropriate $\theta_l$. In this setting, since $dt$ is small, we can linearize the propagation of $g^{(l)}$ about $t=0$, and obtain
\begin{align*}
\mathcal{P}_{X_f}(g^{(0)},g^{(l+1)},r) 
&= \mathcal{P}_{X_f}(g^{(0)},g^{(l)},r)  + \frac{\partial}{\partial t}\mathcal{P}_{X_f}(g^{(0)},\mathcal{S}[t,f]D[\theta_l h]g^{(l)},r)\Big|_{t=0}dt + O(dt^2),
\end{align*}
where the first term in the right-hand-side follows the fact that $\mathcal{P}_{X_f}(g^{(0)},D[\theta_l h]g^{(l)},r)$ does not depend on $\theta_l$.
We would now like to know if modulating the signal at layer $l$ improves the routing measure at layer $l+1$. For that, it is enough to show that the derivative of $\mathcal{P}_{X_f}(g^{(0)},g^{(l+1)},r)$ with respect to $\theta_l$ is nonzero at $\theta_l=0$. Observe that
\[
\frac{\partial}{\partial \theta_l}\mathcal{P}_{X_f}(g^{(0)},g^{(l+1)},r)= \frac{\partial}{\partial \theta_l} \frac{\partial}{\partial t}\mathcal{P}_{X_f}(g^{(0)},\mathcal{S}[t,f]D[\theta_l h]g^{(l)},r)\Big|_{t=0} +O(dt^2).
\]
Hence, our goal is to show that
\[
\mathcal{D}:=\frac{\partial}{\partial \theta_l} \frac{\partial}{\partial t}\mathcal{P}_{X_f}(g^{(0)},\mathcal{S}[t,f]D[\theta_l h]g^{(l)},r)\Big|_{t,\theta_l=0}
\]
is nonzero in general. As long as this is true,  $\theta_l=0$ is not the minimizer of $\mathcal{P}_{X_f}(g^{(0)},g^{(l+1)},r)$, so one can always choose a better modulation than $\theta_l=0$.

We now simplify the notation and give a formula for 
\[
\mathcal{D}:=\frac{\partial}{\partial \theta}\frac{\partial}{\partial t}\mathcal{P}_{X_f}(g,\mathcal{S}[t,f]D[\theta h]g,r)\Big|_{t=\theta=0}.
\] 

\begin{claim}
[Mixed derivative of the signal routing measure]\label{claim:srm_derivative}
\begin{align*}
&\frac{\partial}{\partial \theta}\frac{\partial}{\partial t}\mathcal{P}_{X_f}(g,\mathcal{S}[t,f]D[\theta h]g,r)\Big|_{t=\theta=0} = \frac{\langle [X_h,[\Delta,X_f^2]]g,g\rangle \;-\; 4r\,\mathrm{Re}\,\langle [X_h,W_f\nabla_f] g, g\rangle}{\mathcal{V}_{X_f}(g)}
\end{align*}
% $?}
% $. }

%  
\end{claim}
We see that when $h$ is constant, i.e.,  there is no modulation, $\mathcal{D}$ is zero.

\begin{proof}
Let $\phi_\theta=D[\theta h]g=e^{i\theta X_h}g$. 
Since both $D[\theta h]$ and $X_f$ are diagonal, they commute. 
Consequently, 
$\mathcal{E}_{X_f}(\phi_\theta)=\mathcal{E}_{X_f}(g)$ and $\mathcal{V}_{X_f}(\phi_\theta)=\mathcal{V}_{X_f}(g)$ for every  $\theta$.
Using Lemma~\ref{thm:energy_flow_measure}, Theorem~\ref{thm:derivative_expected_feature}, and Theorem~\ref{thm:var_evolution}, we obtain
\begin{equation*}
    \frac{\partial}{\partial t}\mathcal{P}_{X_f}(g,\mathcal{S}[t,f]\phi_\theta,r)\Big|_{t=0}
=
\frac{\mathcal{E}_{i[\Delta_f,X_f^2]}(\phi_\theta)-4r\,\mathrm{Re}\big(\langle i\nabla_f\phi_\theta,W_f\phi_\theta\rangle\big)}{\mathcal{V}_{X_f}(g)}.
\end{equation*}
We now differentiate with respect to $\theta$.
For any operator $A$ and $\phi_\theta=e^{i\theta X_h}g$, we have
\begin{equation*}
    \frac{d}{d\theta}\langle A\phi_\theta,\phi_\theta\rangle
=
i\langle [A,X_h]\phi_\theta,\phi_\theta\rangle.
\end{equation*}
By applying this with $A=i[\Delta_f,X_f^2]$ and evaluating at $\theta=0$, we obtain 
\begin{equation*}
    \frac{d}{d\theta}\mathcal{E}_{i[\Delta_f,X_f^2]}(\phi_\theta)\Big|_{\theta=0}
=
\langle [X_h,[\Delta_f,X_f^2]]g,g\rangle.
\end{equation*}
For the second term, we can write 
\begin{equation*}
    \langle i\nabla_f\phi_\theta,W_f\phi_\theta\rangle
=
\langle iW_f\nabla_f\phi_\theta,\phi_\theta\rangle.
\end{equation*}
We can use the same  differentiation rule with $A=iW_f\nabla_f$ and get 
\begin{align*}
\frac{d}{d\theta}\mathrm{Re}\big(\langle i\nabla_f\phi_\theta,W_f\phi_\theta\rangle\big)\Big|_{\theta=0}
&=
\mathrm{Re}\left(i\langle [iW_f\nabla_f,X_h]g,g\rangle\right) \\
&=
\mathrm{Re}\big(\langle [X_h,W_f\nabla_f]g,g\rangle\big).
\end{align*}
Substituting the two derivatives into the previous formula completes the proof.
\end{proof}

\subsection{Sensitivity analysis of signal routing via Schr\"odinger operators}
\label{Sensitivity Analysis of Signal Routing via Schr\"odinger Operators}

In this section, we propose a simple sensitivity analysis for Schr\"odinger Operators in our observable setting, to mirror the sensitivity analysis presented in oversquashing papers \citep{topping2022understandingoversquashingbottlenecksgraphs,BlackEtAl2023ER, di2023over}.

Suppose we have an input nonnegative-valued signal  at some layer, which is not localized on one graph region, and we partition it into chunks. To keep the notations simple, we focus on one chunk, and denote it by $g^{(0)}$. Since the original signal has roughly the same amount of energy at all locations (it is  not localized), we can roughly think about $g^{(0)}$ both as the signal's chunk and as a localization window about $\mathcal{E}_{X_{f_{k}}}(g^{(0)})$, with variance $\mathcal{V}_{X_{f_{k}}}(g^{(0)})$.

Suppose that we modulate the signal to $g^{(0)}_{\theta} = D[\theta h]g^{(0)}$, apply a Sch\"odinger operator $\mathcal{S}[t,f]$ to obtain the output signal $\Phi(g^{(0)})= \mathcal{S}[t,f]g^{(0)}_{\theta}$. We lastly apply a modulus nonlinearity, to obtain $\Psi(g^{(0)})=\abs{\Phi(g^{(0)})}$.
Similarly to the input signal, we think of $\Psi(g^{(0)})$ both as the output signal and as a localization window about $\mathcal{E}_{X_{f_{k}}}(\Psi(g^{(0)}))$, with variance $\mathcal{V}_{X_{f_{k}}}(\Psi(g^{(0)}))$. For this interpretation, we suppose that the variance of $\Psi(g^{(0)})$ is not too high.

We would like to know how much the values of $\Psi(g^{(0)})$ in the graph region about $\mathcal{E}_{X_{f_{k}}}\big(\Phi(g^{(0)})\big)$ are sensitive to changes in the input values $g^{(0)}$ about  $\mathcal{E}_{X_{f_{k}}}(g^{(0)})$. If the sensitivity is high, it indicates that much of the information from region $\mathcal{E}_{X_{f_{k}}}(g^{(0)})$ reached region $\mathcal{E}_{X_{f_{k}}}\big(\Psi(g^{(0)})\big)$. 
We formalize this sensitivity via the differential of $\Phi$, which is also the standard approach in oversquashing analyses \citep{topping2022understandingoversquashingbottlenecksgraphs,BlackEtAl2023ER, di2023over}.

\paragraph{Differentials.}

To keep the analysis self-contained, we recall basic definitions related to differentials.

Given a function $\psi:\RR^D\rightarrow \RR^K$, and a point in the domain $x\in\RR^D$, if there exists a linear operator $\mathcal{D}_x[\psi]:\RR^D\rightarrow\RR^K$ which satisfies for every $y\in\RR^D$ 
\[\psi(x+dy) = \psi(x) + d \mathcal{D}_x[\psi]y +o(d),\]
this operator is called the \emph{differential} of $\psi$ at $x$. Here, the limit in the $o$ notation is over $d>0$. In view of the above formula,  the function $y\mapsto \psi(x) + \mathcal{D}_x[\psi]y$ is seen as the linearization of $\psi$ about $x$.

In typical calculus formulations, we compute partial derivatives. These are connected to differentials as follows. Denote the $n$'th standard basis element  by $\delta^n=(0,\ldots, 0, 1,0\ldots,0)$, i.e., $\delta^n$ has 1 at entry $n$ and 0 elsewhere. Given $m\in[D]$, consider the linear functional $\lceil\delta^m\rceil:\RR^D\rightarrow\RR$ which operates on input vectors $v=(v_d)_{d=1}^D$ via $\lceil\delta^m\rceil(v) = v_m$. Now, the partial derivative of the $m$'th output coordinate of $\psi$ with respect to the $n$'th  input coordinate is exactly $\lceil\delta^m\rceil\mathcal{D}_x[\psi] (\delta^n)$.

In general, we need not restrict ourselves to probing the differential using only standard basis elements. Instead, given a general vector $v\in\RR^K$ one can consider a general linear functional $\lceil v\rceil$ defined by
\[\lceil v\rceil(w) = \ip{w}{v}\]
for probing the output of the differential. Moreover, one can consider a general vector $w\in\RR^D$ for probing the input. Now, the quantity
$\lceil v\rceil\mathcal{D}_x[\psi] (w)$ is interpreted as the partial derivative of the $v$  component of the output of $\psi$ with respect to the $w$ component of the input.

\paragraph{Derivation of the routing sensitivity}

In our formulation, we are not interested in showing that the sensitivity between all pairs of nodes is high. Rather, we just want to show that the sensitivity between the starting region and the destination  region is high. Indeed, our premise is that good GNNs should have the ability to route the signal information between each region of the graph to a \emph{specific} corresponding destination region, relevant for the task. We suppose that the destination of region  $\mathcal{E}_{X_{f_{k}}}(g^{(0)})$  is $\mathcal{E}_{X_{f_{k}}}\big(\Psi(g^{(0)})\big)$. Namely, we suppose that the correct modulation feature and time $t$ were already chosen for the task of interest. This greatly simplifies the analysis with respect to oversquashing analyses, and gives us exact sensitivity results rather than loose inequalities.

By linearity of $\Phi$, we have
\[\lceil \Phi(g^{(0)}) \rceil \mathcal{D}_{g^{(0)}}[\Phi] (g^{(0)})= 1.\]
In words, scaling up the input of a linear operator scales up the output by the same value.

Now, by the fact that $\Psi = \abs{\Phi}$ we also have 
\[\lceil \Psi(g^{(0)}) \rceil \mathcal{D}_{g^{(0)}}[\Psi] (g^{(0)})= 1.\]
Indeed, denote $\nu=\Psi(g^{(0)})$, and observe that
\[\ip{\abs{\nu}}{\abs{\nu}} = \sum_n \abs{\nu_n}\abs{\nu_n} = \sum_n \nu_n\overline{\nu_n}= \ip{\nu}{\nu}.\]
As a result, the sensitivity of $\Psi$ probed at $\Psi(g^{(0)})$ is equal to the sensitivity of $\Phi$ probed at $\Phi(g^{(0)})$.
This simple observation shows that all of the energy in the window $g^{(0)}$ is mapped to the window $\Psi(g^{(0)})$ under $\Psi$.

\section{Tutorial: Schr\"odinger dynamics in classical quantum mechanics}\label{sec:Schrodinger_classic}

Our method is inspired by classical quantum mechanics.
In this section, we briefly review it. 
Throughout this subsection we take $f(x)=x$, let $X_f$ be the position operator $(X_f g)(x)=x\,g(x)$, and choose the continuum analogue 
\begin{equation*}
    \nabla_f = -\frac{\partial}{\partial x},\quad \Delta=-\nabla_f^2=-\frac{\partial^2}{\partial x^2}, 
\end{equation*}
With this convention, we have
\begin{equation*}
    \mathcal{E}_{i\nabla_f}(g)=\left\langle i\left(-\frac{\partial}{\partial x}\right)g,\;g\right\rangle
= -i\int_{\mathbb R} g'(x)\overline{g(x)}dx.
\end{equation*}
Note that if $g$ is real-valued, then $\mathcal{E}_{i\nabla_f}(g)=0$.
The free Schr\"odinger evolution is $g^{(t)}=e^{-it\Delta}g$.
\begin{theorem}[Heisenberg equation of motion for expected values]\label{thm:heisenberg_motion}
For every observable $A$, we have 
\begin{equation*}
    \frac{\partial}{\partial t}\mathcal{E}_{A}(g^{(t)})= i\langle [\Delta, A]g^{(t)}, g^{(t)} \rangle.
\end{equation*}
\end{theorem}
\begin{proof}
    By definition, we have 
    \begin{equation*}
        \begin{aligned}
            \frac{\partial}{\partial t}\mathcal{E}_{A}(g^{(t)})  &= \frac{\partial}{\partial t}\langle Ag^{(t)},g^{(t)}\rangle \\
&= \langle A(-i\Delta)g^{(t)},g^{(t)}\rangle + \langle Ag^{(t)},(-i\Delta)g^{(t)}\rangle \\
&= i\langle [\Delta,A]g^{(t)},g^{(t)}\rangle.
        \end{aligned}
    \end{equation*}
\end{proof}
\begin{theorem}[Expected position evolution in the classical case]\label{thm:classical_position_evolution}
Let $X$ be multiplication by $x$ and let $\Delta=-\partial_x^2$.
Then,
\begin{equation*}
    \frac{\partial}{\partial t}\mathcal{E}_{X}(g^{(t)})=2\,\mathcal{E}_{i\nabla}(g^{(t)}).
\end{equation*}
Moreover, $\mathcal{E}_{i\nabla}(g_t)$ is constant in $t$, and therefore 
\begin{equation*}
    \mathcal{E}_{X}(g^{(t)})=\mathcal{E}_{X}(g^{(0)})+2t\,\mathcal{E}_{i\nabla}(g^{(0)}).
\end{equation*}
\end{theorem}
\begin{proof}
    By Theorem~\ref{thm:heisenberg_motion}, we have
    \begin{equation*}
        \frac{\partial}{\partial t}\mathcal{E}_{X}(g_t)= i\langle [\Delta, X]g_t, g_t \rangle.
    \end{equation*}
    For any smooth test function $u$, we have 
    \begin{equation*}
        [\Delta,X]u = -\partial_x^2(xu)+x\partial_x^2u = -2\,\partial_x u = 2i(i\partial_x u).
    \end{equation*}
    Therefore, $\frac{\partial}{\partial t}\mathcal{E}_{X}(g^{(t)})=2\,\mathcal{E}_{i\nabla}(g^{(t)})$. 
    Similarly, since
    \begin{equation*}
        [\Delta,i\nabla]=[-\partial_x^2,i\partial_x]=0,
    \end{equation*}
    we have $\mathcal{E}_{i\nabla}(g^{(t)})$ is conserved by Theorem~\ref{thm:heisenberg_motion}. Integrating in time gives the stated linear evolution of $\mathcal{E}_{X}(g^{(t)})$.
\end{proof}
For real-valued signals, the expected location remains constant under Schr\"odinger evolution, which motivates the need for modulation to achieve directional transport.
\begin{theorem}[Modulation generates classical transport]
Let $g,\widetilde{g}:\mathbb{R}\to\mathbb{R}$ be real-valued and define $g^{(0)}(x)=e^{i\theta \widetilde{g}(x)}g(x)$. Then
\begin{equation*}
    \mathcal{E}_{i\nabla}(g^{(0)})= -\theta\int \widetilde{g}'(x)\,|g(x)|^2\,dx,
\end{equation*}
and 
\begin{equation*}
    \mathcal{E}_{X}(g^{(t)})=\mathcal{E}_{X}(g)-2t\theta\int \widetilde{g}'(x)\,|g(x)|^2\,dx.
\end{equation*}
\begin{proof}
    We can directly compute 
    \begin{equation*}
        \begin{aligned}
            \mathcal{E}_{i\nabla}(g^{(0)}) 
                &= \int i\partial_x(e^{i\theta \widetilde{g}(x)}g(x))\,\overline{e^{i\theta \widetilde{g}(x)}g(x)}\,dx \\
                &= \int (ig'(x)-\theta \widetilde{g}'(x)g(x))g(x)\,dx \\
                &= -\theta\int \widetilde{g}'(x)|g(x)|^2\,dx,
        \end{aligned}
    \end{equation*}
    since $\int ig'(x)g(x)\,dx=0$ for real-valued $g$. The second identity then follows from Theorem~\ref{thm:classical_position_evolution}.
\end{proof}
\end{theorem}
\begin{theorem}[Time derivative of the position variance in the free Schr\"odinger case]
Let $g^{(t)}=e^{-it\Delta}g$ with $\Delta=-\partial_x^2$. Then
\begin{equation*}
    \frac{\partial}{\partial t}\mathcal{V}_{X}(g^{(t)})
=
\mathcal{E}_{i[\Delta,X^2]}(g^{(t)})-4\,\mathcal{E}_{X}(g^{(t)})\,\mathcal{E}_{i\nabla}(g^{(t)}).
\end{equation*}
Since $\mathcal{E}_{i\nabla}(g^{(t)})$ is constant, it can also be written as
\begin{equation*}
    \frac{\partial}{\partial t}\mathcal{V}_{X}(g^{(t)})
=
\mathcal{E}_{i[\Delta,X^2]}(g^{(t)})-4(\mathcal{E}_{X}(g^{(0)})+2t\,\mathcal{E}_{i\nabla}(g^{(0)}))\mathcal{E}_{i\nabla}(g^{(0)}).
\end{equation*}
\begin{proof}
    By definition, we have 
    \begin{equation*}
        \mathcal{V}_{X}(g^{(t)})=\mathcal{E}_{X^2}(g^{(t)})-\mathcal{E}_{X}(g^{(t)})^2.
    \end{equation*}
    By differentiating and using Theorem~\ref{thm:heisenberg_motion} with $A=X^2$, we obtain
    \begin{equation*}
        \frac{\partial}{\partial t}\mathcal{V}_{X}(g^{(t)})
=
\mathcal{E}_{i[\Delta,X^2]}(g^{(t)})-2\mathcal{E}_{X}(g^{(t)})\frac{\partial}{\partial t}\mathcal{E}_{X}(g^{(t)}).
    \end{equation*}
    By Theorem~\ref{thm:classical_position_evolution}, we have $\frac{\partial}{\partial t}\mathcal{E}_{X}(g^{(t)})=2\mathcal{E}_{i\nabla}(g^{(t)})$, which gives the first formula. The second is obtained by substituting the linear evolution of $\mathcal{E}_{X}(g^{(t)})$ and the conservation of $\mathcal{E}_{i\nabla}(g^{(t)})$.
\end{proof}
\end{theorem}

\paragraph{Relation of standard quantum-mechanical identities from our graph-specific constructions.}
Note that our contribution is to adapt the established formalism from quantum mechanic to graph signal processing and graph machine learning and to derive the resulting graph-specific observables, derivatives, Hamiltonians, commutator terms, and routing guarantees. 
Below, we distinguish standard quantum-mechanical identities from our graph-specific constructions. 

Specifically, Theorem \ref{thm:derivative_expected_feature} is obtained by applying the standard Heisenberg equation for expectation values, also known as the general Ehrenfest relation, $\frac{d}{dt}\langle A\rangle_t=i\langle [H,A]\rangle_t$ in units $\hbar=1$, to the graph Hamiltonian $H=\Delta_f$ and the feature-location observable $A=X_f$. Our Lemmas \ref{lem:smooth_comm}-\ref{lemma:commutator_expansion} then evaluate this commutator for the proposed feature derivative and obtain $\frac{d}{dt}\mathcal{E}_{X_f}(g^{(t)})=2\operatorname{Re}\langle
i\nabla_f g^{(t)},W_f g^{(t)} \rangle$. In the continuum free-particle setting $f(x)=x, \nabla_f=-\partial_x$, and $W_f[X_f,\nabla_f]=I$. The formula reduces to $\frac{d}{dt}\langle X\rangle_t = 2\langle P\rangle_t$, which is the free-particle Ehrenfest equation under the normalization $H=P^2$, corresponding to $\hbar=1$ and $m=\frac{1}{2}$. 
Importantly, we remark that:
\begin{itemize}
    \item Theorem \ref{thm:expected_momentum_invariant} is a graph free-momentum conservation, which is an instance of the Heisenberg/Ehrenfest equation when the momentum observable commutes with the Hamiltonian
    \item Theorem \ref{thm:derivative_expected_feature}  is a graph-specific application of the Heisenberg equation for expectation values and reduces to the standard Ehrenfest theorem for expected position in the continuum free-particle setting.
    \item  Theorem \ref{thm:non_zero_expected_momentum_modulated_signals} is a graph analogue of the momentum shift induced by a spatially varying phase.
    \item Theorem \ref{thm:derivative_expected_multi_feature} is a multidimensional graph Ehrenfest relation whose additional cross terms arise because the learned graph directions do not satisfy the exact Cartesian commutation relations.
    \item  Theorem \ref{thm:var_evolution} follows from applying the Heisenberg equation to the second position moment and differentiating the corresponding variance.
    \item  Theorems \ref{thm:deviation_bounds_single_feature} and \ref{thm:deviation_bounds_multi_feature} are graph-specific approximation bounds quantifying the deviations caused by graph smoothing and cross-feature non-commutation.
\end{itemize}

\section{Experimental details and additional experiments}
\label{app:exp_detail_addtional_exp}

All experimental runs were conducted on individual GPUs, specifically utilizing an NVIDIA L40S hardware. 

\subsection{Toy experiment: recovering orthogonal directions on a grid}

To illustrate the effectiveness of position-momentum optimization (PMO), we consider a two-dimensional grid graph whose original node features are the Cartesian coordinates $(x,y)$. We then replace this orthogonal pair by the correlated pair $(x,x+y)$ and run PMO on the resulting two channels. Since $(x,x+y)$ spans the same two-dimensional feature subspace as $(x,y)$, the optimization should recover two approximately orthogonal directions. Figure~\ref{fig:grid_orthogonality} compares the correlated input features with the transformed features returned by PMO. 

\begin{figure}[h]
\centering
\includegraphics[width=1\textwidth]{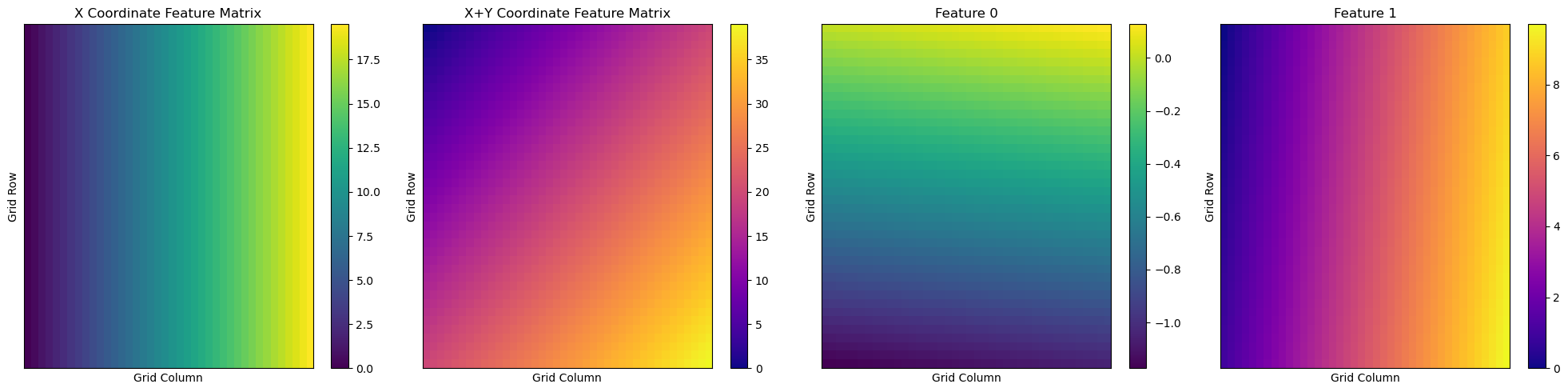}
\caption{
Grid-orthogonality toy experiment. The two left panels show the input features $x$ and $x+y$. The two right panels show the transformed features after PMO. The recovered directions align with two approximately orthogonal axes on the grid.
}
\label{fig:grid_orthogonality}
\end{figure}

\subsection{Optimizing signal transport via modulation}
\label{app: Optimizing signal transport via modulation}

In this subsection, we illustrate the effectiveness of signal transport under our modulation formulation. 
Specifically, we use a synthetic two-cluster graph to demonstrate that phase modulation can improve directed transport. We sample $N=60$ nodes from two Gaussian clouds in $\mathbb{R}^2$: $30$ nodes around $(-1,0)$ and $30$ nodes around $(1,0)$, each with standard deviation $0.5$ in both coordinates. Two nodes are connected by an undirected, unweighted edge whenever their Euclidean distance is smaller than $1.5$. The feature location is the $x$-coordinate, denoted by $f$, and we use the same feature for modulation, namely $h=f$.

We choose a nonnegative initial signal that is concentrated on the left cluster and normalize it to unit $\ell_2$ norm. The resulting signal is denoted by $g$. In the realization shown in Figure~\ref{fig:theta_transport}, this normalized signal satisfies $\mathcal{E}_{X_f}(g)=-0.98$, and the target location is set to $r=1$. 

For each modulation parameter $\theta\in[-5,5]$, we apply the modulation $D[\theta f]$, propagate by the Schr\"odinger operator $\mathcal S[0.1,f]$, and repeat this propagation three times. We track
\begin{align*}
\mathcal E_{X_f}(g)&=\sum_{n=1}^N f(n)|g(n)|^2,\\
\mathcal V_{X_f}(g)&=\mathcal E_{X_f^2}(g)-\mathcal E_{X_f}(g)^2,\\
\mathcal P_{X_f}(g^{(0)},g^{(t)},r)&=\frac{\mathcal V_{X_f}(g^{(t)})+\bigl(r-\mathcal E_{X_f}(g^{(t)})\bigr)^2}{\mathcal V_{X_f}(g^{(0)})}.
\end{align*}
We then evaluate $\mathcal{E}_{X_f}(g_\theta^{(0.1)})$ and $\mathcal{P}_{X_f}(g,g_\theta^{(0.1)},r)$. For visualization only, we renormalize $g_\theta^{(0.1)}$ to unit $\ell_2$ norm and plot its magnitude. Figure~\ref{fig:theta_transport} shows that an appropriate modulation parameter $\theta$ moves the signal toward the target location and lowers the routing measure relative to the unmodulated case.

\begin{figure}[t]
\centering
\includegraphics[width=0.7\textwidth]{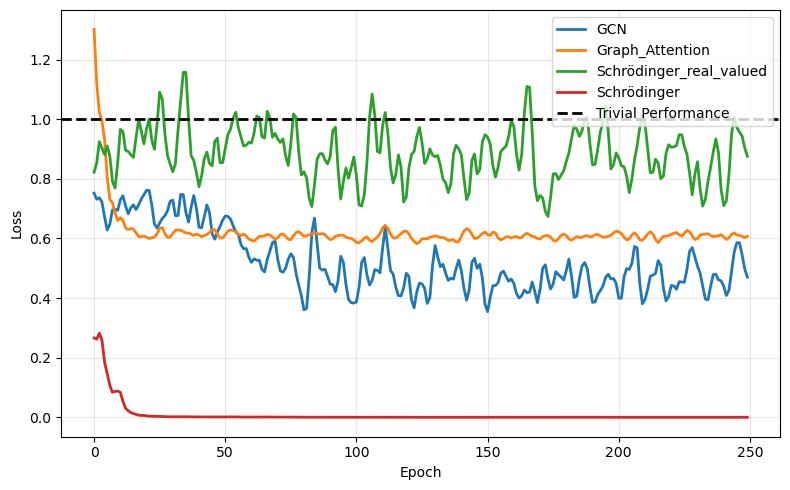}
\caption{Gaussian-translate learning curves. Lower is better. The modulated Schr\"odinger model converges to the lowest error among the compared models, outperforming the real-valued baselines and the non-modulated Schr\"odinger variant. The dashed line denotes the trivial predictor.}
\label{fig:dirichlet-curves}
\end{figure}

\subsection{Gaussian-translate toy experiment on a ring}

We study a synthetic translation task on a ring graph. Given an input signal $x$ on the ring, the model needs to predict the translated signal $y=S_d x$, where $S_d$ denotes circular shift by $d$ nodes. 
This task stresses whether a graph model can implement phase consistent transport on a simple topology.

\subsubsection{Data}

We use a ring graph with $N=100$ nodes. Let $\theta_n=-\pi+2\pi n/N$ denote the angular coordinate of node $n$. Each sample is generated by drawing a variance $\sigma^2\sim\mathcal{U}[0.5,1.5]$, sampling a Gaussian signal on the ring, adding Gaussian noise with standard deviation $10^{-3}$, applying a random circular shift, and normalizing the result to unit $\ell_2$ norm. The target is $y=S_d x$ with a fixed shift $d=35$. We use an 80/10/10 train/validation/test split and batch size $32$.

\subsubsection{Models}

We compare GCN and GAT baselines with Schr\"odinger variants that approximate a unitary graph flow through a truncated exponential. Let $A$ be the aggregation operator on the cycle and define the skew-Hermitian generator $H=i\,A$. Each Schr\"odinger layer applies a learnable linear map $W$ followed by the truncated series $
z\leftarrow \sum_{k=0}^{T}\frac{(\delta\mathcal{H})^k}{k!}\,Wz$ with $T=15$,
where $\delta$ is a learned per-channel scale. All Schr\"odinger models use feature normalization after each layer and a magnitude nonlinearity. The modulated variant additionally learns a phase direction $m=\mathrm{Linear}([x,\theta])$ and multiplies the features by $e^{i\epsilon m}$ with $\epsilon=25$.

\subsubsection{Training}

For a sample $(x_i,y_i)$, we minimize the $L_2$ loss $\|\hat y_i-y_i\|_2$. We train for $250$ epochs with Adam \citep{kingma2017adammethodstochasticoptimization}, using a base learning rate of $0.1$, a $10\times$ larger learning rate for the modulation parameters, and a ReduceLROnPlateau scheduler with factor $0.7$ and patience $10$. Figure~\ref{fig:dirichlet-curves} reports smoothed test curves for readability and the dashed line denotes a naive baseline.
In Table~\ref{tab:dirichlet-results}, we report the parameter counts for the Schr\"odinger GNN and the competing methods. 

\begin{table}[h]
\centering
\caption{Parameter counts for the Gaussian-translate models on a ring with $N=100$ and shift $d=35$. This table reports model size only, not predictive performance.}
\vspace{1 mm}
\begin{tabular}{lrrccc}
\toprule
Model & Params \\
\midrule
GCN & 2{,}136 \\
GAT & 6{,}193 \\
Schr\"odinger non modulated & 4{,}273 \\
Schr\"odinger & 4{,}275 \\
\bottomrule
\end{tabular}
\label{tab:dirichlet-results}
\end{table}

\subsection{MNIST graph classification}

We conduct an experiment on the classical MNIST dataset  \citep{lecun1998mnist} to evaluate our model's performance on a standard image classification task formulated as a graph problem. 

\paragraph{Dataset construction.} We convert each $28\times 28$ MNIST image into a graph with $N=784$ nodes, one node per pixel. We add undirected edges using 8-neighbor grid connectivity (Chebyshev radius $1$). Each node $v_i$ carries the feature vector $\mathbf{x}_i=[x_{\mathrm{norm}},y_{\mathrm{norm}},I]$, where $x_{\mathrm{norm}},y_{\mathrm{norm}}\in[0,1]$ are normalized pixel coordinates and $I\in[0,1]$ is the pixel intensity. We use the standard MNIST split of 60{,}000 training images and 10{,}000 test images \citep{lecun1998mnist}, and we report averages over five random seeds.

\paragraph{Baselines.}
We compare against five standard GNN baselines: GCN \citep{kipf2017semisupervisedclassificationgraphconvolutional}, GAT \citep{velickovic2018graphattentionnetworks}, GIN \citep{xu2019howpowerfularegraphneuralnetworks}, MPNN \citep{gilmer2017neuralmessagepassingquantum}, and ChebConv \citep{defferrard2017convolutionalneuralnetworksgraphs}. As an additional non-graph reference, we train a classical 2D CNN \citep{lecun1998gradient} directly on the raw $28\times 28$ images. For comparability, the CNN uses the same optimizer, learning rate, dropout, and overall width/depth scale whenever applicable.

\paragraph{Training hyperparameters.}
All graph models use hidden dimension $64$, depth $3$, batch size $16$, dropout $0.1$, and global mean pooling. We train for $200$ epochs with Adam and learning rate $\alpha=3\times 10^{-4}$.

\paragraph{Classification performance. }
Table~\ref{tab:mnist_results} reports the classification performance. We can see that our Schr\"odinger model achieves competitive performance. 

\begin{table}[h]
\centering
\caption{MNIST classification performance.}
\vspace{1 mm}
\begin{tabular}{lrrccc}
\toprule
Model & Test Accuracy $(\uparrow)$ \\
\midrule
GCN &  92.09\scriptsize{$\pm$0.28} \\
ChebConv & 95.72\scriptsize{$\pm$0.74} \\ 
GAT &  95.94\scriptsize{$\pm$0.71}\\
GIN & 98.33\scriptsize{$\pm$0.11}\\ 
MPNN & 98.95\scriptsize{$\pm$0.06}\\
CNN & 99.07\scriptsize{$\pm$0.07} \\ 
\midrule
Schr\"odinger & 99.13\scriptsize{$\pm$0.04}\\
\bottomrule
\end{tabular}
\label{tab:mnist_results}
\end{table}

\subsection{Graph classification on TU dataset}
In this subsection, we present the experimental setup for  the architecture-matched comparison on ENZYMES, IMDB-BINARY, MUTAG, and PROTEINS from the TU dataset collection \citep{morris2020tudatasetcollectionbenchmarkdatasets}. 
% dataset split: 50/25/25
The dataset is split into $50\%$ for training, $25\%$ for validation and $25\%$ for testing.
The corresponding results are reported in Table~\ref{tab:architecture_matched}.

\begin{table}[h]
\centering
\caption{
Statistics of the TU graph-classification datasets used in the architecture-matched comparison.}
\vspace{1 mm}
\begin{tabular}{lcccc}
\toprule
& {Enzymes} & {IMDB} & {MUTAG} & {PROTEINS} \\
\midrule
\# Graphs & 600 & 1000 & 188 & 1113 \\
\# Nodes (range) & 2 - 126 & 12 - 136 & 10 - 28 & 4 - 620 \\
\# Edges (range) & 2 - 298 & 52 - 2498 & 20 - 66 & 10 - 2098 \\
Avg \# Nodes & 32.63 & 19.77 & 17.93 & 39.06 \\
Avg \# Edges & 124.27 & 193.06 & 39.58 & 145.63 \\
\# Classes & 6 & 2 & 2 & 2 \\
Directed & False & False & False & False \\
ORC Mean & 0.13 & 0.58 & -0.27 & 0.17 \\
ORC Std & 0.15 & 0.19 & 0.05 & 0.20 \\
\bottomrule
\end{tabular}
\label{tab:tu_stats}
\end{table}

\paragraph{Architectural constraints.}
To compare architectures under a common capacity budget, we use a standardized backbone with six graph convolution layers followed by one linear classifier. We first define a reference parameter budget using the Unitary (UniGCN) architecture with a hidden dimension of $128$. For the remaining models (GAT, GCN, GIN, Adaptive Unitary, Schr\"odinger, and Schr\"odinger PMO), we then adjust the hidden dimension so that the total number of trainable parameters falls within $0.6\%$ of this reference budget. This isolates differences in the propagation operator as much as possible. For complex-valued models, each complex parameter is counted as two real parameters.

\paragraph{Hyperparameter search.}
For each model--dataset pair, we search over learning rate $\{0.0005,0.001,0.005,0.01\}$ and dropout $\{0,0.25,0.5\}$. The search space follows the settings used in prior evaluations~\citep{kiani2024unitaryconvolutionslearninggraphs,nguyen2023revisitingoversmoothingoversquashingusing}. The chosen configuration is the one with the best mean validation accuracy over $100$ runs.
The chosen values are summarized in Table~\ref{tab:arch_matched_hyperparams}.

\begin{table}[h]
    \centering
    \caption{Hyperparameters for the architecture-matched comparison.}
    \vspace{1 mm}
    \label{tab:arch_matched_hyperparams}
    \resizebox{0.9\textwidth}{!}{%
    \begin{tabular}{l|l|cccc}
    \toprule
    \textsc{Model} & \textsc{Hyperparameter} & \textsc{ENZYMES} & \textsc{IMDB} & \textsc{MUTAG} & \textsc{PROTEINS} \\
    \midrule
    \multirow{4}{*}{GCN} & 
    Learning Rate & $0.005$ & $0.005$ & $0.005$ & $0.001$ \\
    & Dropout & $0$ & $0$ & $0$ & $0$ \\
    & Hidden Dimension & $190$ & $190$ & $190$ & $190$ \\
    \midrule
    \multirow{4}{*}{GAT} & 
    Learning Rate & $0.001$ & $0.001$ & $0.0005$ & $0.005$ \\
    & Dropout & $0$ & $0.5$ & $0$ & $0.5$ \\
    & Hidden Dimension & $189$ & $189$ & $189$ & $189$ \\
    \midrule
    \multirow{4}{*}{Schr\"odinger} & 
    Learning Rate & $0.005$ & $0.0005$ & $0.005$ & $0.005$ \\
    & Dropout & $0.25$ & $0$ & $0.25$ & $0$ \\
    & Hidden Dimension & $117$ & $170$ & $170$ & $170$ \\
    \midrule
    \multirow{4}{*}{Schr\"odinger PMO} & 
    Learning Rate & $0.005$ & $0.001$ & $0.01$ & $0.005$ \\
    & Dropout & $0$ & $0$ & $0$ & $0$ \\
    & Hidden Dimension & $117$ & $117$ & $117$ & $117$ \\
    \midrule
    \multirow{4}{*}{Unitary} & 
    Learning Rate & $0.001$ & $0.001$ & $0.001$ & $0.0005$ \\
    & Dropout & $0$ & $0$ & $0$ & $0$ \\
    & Hidden Dimension & $128$ & $128$ & $128$ & $128$ \\
    \midrule
    \multirow{4}{*}{Adaptive Unitary} & 
    Learning Rate & $0.005$ & $0.0005$ & $0.005$ & $0.001$ \\
    & Dropout & $0$ & $0$ & $0$ & $0$ \\
    & Hidden Dimension & $127$ & $127$ & $127$ & $127$ \\
    \midrule
        \multirow{4}{*}{Adaptive Unitary PMO} & 
    Learning Rate & $0.001$ & $0.01$ & $0.001$ & $0.001$ \\
    & Dropout & $0.25$ & $0$ & $0$ & $0$ \\
    & Hidden Dimension & $127$ & $127$ & $127$ & $127$ \\
    \midrule
    \multirow{4}{*}{GIN} 
    & Learning Rate & $0.001$ & $0.005$ & $0.01$ & $0.0005$ \\
    & Dropout & $0$ & $0$ & $0$ & $0$ \\
    & Hidden Dimension & $190$ & $190$ & $190$ & $190$ \\
    \bottomrule
    \end{tabular}%
    }
    \end{table}

\paragraph{Runtime comparison.}
Table~\ref{tab:runtime_tu} reports the mean and standard deviation of the training time per epoch for each model on the TU datasets.

\begin{table}[h]
\centering
\caption{
Training time per epoch on TU datasets (seconds, mean $\pm$ std).
}
\vspace{1 mm}
\label{tab:runtime_tu}
\resizebox{0.8\textwidth}{!}{%
\begin{tabular}{lcccc}
\toprule
Model & ENZYMES & IMDB & MUTAG & PROTEINS \\
\midrule
GCN & $33.4 \pm 7.85$ & $27.5 \pm 5.66$ & $19.0 \pm 4.30$ & $33.9 \pm 2.02$ \\
GAT & $61.5 \pm 15.51$ & $42.6 \pm 0.52$ & $14.3 \pm 2.83$ & $65.7 \pm 8.13$ \\
GIN & $39.8 \pm 5.41$ & $32.9 \pm 9.12$ & $9.1 \pm 1.52$ & $43.2 \pm 10.40$ \\
Unitary & $216.7 \pm 6.25$ & $261.9 \pm 72.57$ & $60.4 \pm 14.47$ & $189.3 \pm 6.79$ \\
Adaptive Unitary & $200.1 \pm 27.03$ & $285.5 \pm 48.41$ & $47.5 \pm 12.85$ & $202.2 \pm 36.36$ \\
Adaptive Unitary PMO & $84.3 \pm 15.91$ & $142.9 \pm 7.83$ & $45.0 \pm 0.67$ & $158.0 \pm 0.82$ \\
Schr\"odinger & $172.8 \pm 12.89$ & $255.9 \pm 55.36$ & $44.4 \pm 13.06$ & $247.6 \pm 47.71$ \\
Schr\"odinger PMO & $173.5 \pm 25.11$ & $279.1 \pm 67.42$ & $68.6 \pm 8.54$ & $258.5 \pm 30.49$ \\
\bottomrule
\end{tabular}%
}
\end{table}

\paragraph{Ablation study on ENZYMES dataset}

We use the ENZYMES dataset to isolate the contribution of three ingredients in the Schr\"odinger framework: learnable propagation time, phase modulation, and PMO preprocessing. All models share the same backbone with hidden dimension $128$, learning rate $0.005$, and $300$ training epochs. Results are averaged over $100$ independent runs. The sequence in Table~\ref{tab:ablation_enzymes} reports a cumulative ablation from the UniGCN baseline to the full Schr\"odinger PMO model.

\begin{table}[h]
    \centering
    \caption{
    Ablation study on ENZYMES (test accuracy $\pm$ std). All models share the same backbone architecture and optimization settings.}
    \vspace{1 mm}
    \label{tab:ablation_enzymes}
    \begin{tabular}{lc}
    \toprule
    Model & Test Accuracy \\
    \midrule
    Unitary (UniGCN) & $37.33 \pm 8.25$ \\
    Adaptive Unitary & $41.56 \pm 5.67$ \\
    Schr\"odinger & $43.61 \pm 4.58$ \\
    Schr\"odinger PMO & $44.83 \pm 4.03$ \\
    \bottomrule
    \end{tabular}
\end{table}

\paragraph{Evaluation under 10-fold cross-validation.}
In Section \ref{sec:experiments},  we provide a controlled architecture-matching setup that holds depth and capacity fixed while changing the propagation operator.
Here,  we included the 10-fold cross-validation for TUDataset with layer $L\in\{2, 3, 4, 5, 6\}$ for GCN, GAT, GIN, UniGCN, and Schrödinger GNN. We report the best performance of each method.
These TU results in Table~\ref{tab:classification_results} support the competitiveness of Schrödinger propagation under a standard cross-validation protocol.

\begin{table}[h]
    \centering
    \caption{Classification accuracy (\%) on TU datasets under 10-fold cross-validation.}
    \vspace{1 mm}
    \label{tab:classification_results}
    \resizebox{0.8\textwidth}{!}{%
    \begin{tabular}{lcccc}
        \toprule
         & ENZYMES & IMDB & MUTAG & PROTEINS \\
        \midrule
        GIN & $53.28 \pm 5.31$ ($L=4$) & $73.34 \pm 4.02$ ($L=3$) & $86.49 \pm 4.95$ ($L=3$) & $74.69 \pm 3.64$ ($L=4$) \\
        GCN & $48.03 \pm 3.92$ ($L=3$) & $71.29 \pm 4.86$ ($L=2$) & $83.54 \pm 5.67$ ($L=3$) & $75.31 \pm 5.81$ ($L=3$) \\
        GAT & $47.65 \pm 5.08$ ($L=2$) & $68.76 \pm 3.42$ ($L=2$) & $79.70 \pm 9.45$ ($L=2$) & $76.14 \pm 3.92$ ($L=2$) \\
        UniGCN & $59.97 \pm 7.25$ ($L=6$) & $72.32 \pm 3.64$ ($L=4$) & $80.63 \pm 9.96$ ($L=4$) & $72.82 \pm 4.33$ ($L=6$) \\
        Schrödinger GNN (ours)
        & $63.87 \pm 5.11$ ($L=5$)
        & $74.45 \pm 3.85$ ($L=3$)
        & $88.24 \pm 6.24$ ($L=3$)
        & $77.56 \pm 4.02$ ($L=5$) \\
        \bottomrule
    \end{tabular}
    }
\end{table}

\subsection{Heterophilous node classification}

\paragraph{Dataset statistics.}
Table~\ref{tab:heterophilous_stats} summarizes the statistics of the heterophilous node-classification datasets.

\begin{table}[h]
\centering
\caption{
Statistics of the heterophilous node-classification datasets from \citep{platonov2023critical}.}
\vspace{1 mm}
\label{tab:heterophilous_stats}
\resizebox{\textwidth}{!}{%
\begin{tabular}{lccccc}
\toprule
& \textsc{Roman-Empire} & \textsc{Amazon-Ratings} & \textsc{Minesweeper} & \textsc{Tolokers} & \textsc{Questions} \\
\midrule
\#Nodes & 22,662 & 24,492 & 10,000 & 11,758 & 48,921 \\
\#Edges & 32,927 & 93,050 & 39,402 & 519,000 & 153,540 \\
\#Classes & 18 & 5 & 2 & 2 & 2 \\
Homophily & 0.05 & 0.38 & 0.68 & 0.59 & 0.84 \\
Metric & AP & AP & ROC AUC & ROC AUC & ROC AUC \\
\bottomrule
\end{tabular}%
}
\end{table}

\paragraph{Architecture and hyperparameters.}
We follow the experimental setup as in \cite{kiani2024unitaryconvolutionslearninggraphs}. 
For our Schr\"odinger model, we first run PMO on the input features for a $50$ batch size with learning rate $10^{-3}$, and 2000 epochs. 
We set the number of layers as $\{4, 6, 8, 10\}$ with hidden dimension 512. 
The dropout is set as $\{0.2, 0.5\}$. 
A final linear layer maps the node embeddings to class logits.
The performance is reported with the best results.

\subsection{Long range graph benchmark}

We follow the experimental setup and training scheme as in \cite{tonshoff2023wheredidgapgoreassessinglongrangegraph, kiani2024unitaryconvolutionslearninggraphs}.

\subsubsection{Peptides} 

\paragraph{Edge features.}The current unitary Schr\"odinger block does not natively consume edge attributes. For edge-attributed molecular graphs, we  use an edge-feature adapter before the Schr\"odinger propagation. In the Peptides runs, this adapter is GINE~\citep{xu2019howpowerfularegraphneuralnetworks}.

\paragraph{Resources, parameter count, and search space.}  In our implementation, Peptides epochs typically took less than $15$ seconds, and the dataset storage footprint was about $1$GB. Following the LRGB parameter-budget convention, models are kept around roughly $500$K trainable real-valued parameters.  We also follow the hyperparameter settings as in  \cite{tonshoff2021walking, kiani2024unitaryconvolutionslearninggraphs}.
 Each complex parameter is counted as two real-valued parameters. We use Adam~\citep{kingma2017adammethodstochasticoptimization} with an initial learning rate of $0.001$ and a cosine learning-rate scheduler. The sweep for the basic model varies the number of layers in $\{6,8, 10, 12, 14\}$, dropout in $\{0.1,0.15,0.2\}$, batch size 200, with maximal 4000 epochs, and hidden dimension in $\{135, 155, 175, 195, 200\}$.

\subsubsection{PascalVOC-SP and COCO-SP}
We evaluate on the PascalVOC-SP and COCO-SP datasets from the Long Range Graph Benchmark
(LRGB)~\citep{dwivedi2022long}. Both datasets are superpixel-based computer vision benchmarks
formulated as node classification tasks, and performance is reported using the official F1 metric.
We follow the experimental protocol of \cite{kiani2024unitaryconvolutionslearninggraphs} for these two datasets. 
We train with the Adam optimizer, set the
initial learning rate to $10^{-3}$, use a cosine learning-rate scheduler, and a batch size of 50.

For models that do not directly incorporate edge attributes in the message-passing layer, we follow
\cite{kiani2024unitaryconvolutionslearninggraphs} and prepend a single GatedGCN layer as an edge-feature aggregator, which
maps edge information into the node representations before applying the main network. We use the
same positional encoding as in the reference setup: no additional PE/SE is used for COCO-SP, while random-walk structural encodings (RWSE) \cite{rampavsek2022recipe} are used for PascalVOC-SP.

Hyperparameters were tuned by sweeping the dropout rate over $\{0.1, 0.15, 0.2\}$ and the number of layers over
$\{6, 8, 10, 12, 14\}$. For each depth, the hidden dimension was selected so that the model satisfied the LRGB parameter budget of approximately 500K trainable parameters. When applicable, complex-valued parameters were counted as two real-valued parameters. We trained COCO-SP models for up to 750 epochs and PascalVOC-SP models for up to 1000 epochs. Hyperparameters were selected based on validation performance, and we report the test-set F1 score of the selected configuration.

\subsubsection{Signal-flow diagnosis}
We describe the diagnostic used in Figure~\ref{fig:pascal_mean_relative_shift}. 
The purpose of this experiment is to measure whether the hidden signal of a trained model is transported across the graph by individual layers. 
All models are first trained on PascalVOC-SP using the standard training split and selected by validation performance. 
After training, all parameters are frozen and the diagnostic is run in evaluation mode, with dropout disabled. 
The diagnostic uses only hidden activations and location features.

Let $\gG=(V,E)$ be a PascalVOC-SP superpixel graph with $|V|=N$ nodes. 
Let $f=(f_1,\ldots,f_K):V\rightarrow\mathbb{R}^K$ denote the location features used for the diagnostic. 
In the PascalVOC-SP experiments, we use the two superpixel-center coordinates, $K=2$, corresponding to the normalized horizontal and vertical coordinates. 
The same location features are used for all compared models. 
For a hidden representation $H\in\mathbb{C}^{N\times d}$, or $H\in\mathbb{R}^{N\times d}$ for real-valued models, we compute its node energy and for each location coordinate $f_k$, we compute the expected location of $H$.
This graph-dependent standard deviation is used only for normalization of the reported shift.

\paragraph{Window construction.}
To measure local signal transport, we split the hidden signal into windows according to the location features. 
For each coordinate $f_k$, we construct a collection of nonnegative windows $
    \{w_{b,k}:V\rightarrow[0,1]\}_{b=1}^{B}$
localized at different values of $f_k$. 
These windows are normalized to form a partition of unity along each coordinate $
    \sum_{b=1}^{B} w_{b,k}(v)=1$ for all $ v\in V$.
For two-dimensional PascalVOC-SP coordinates, a window is indexed by $\alpha=(b_1,b_2)$ and is defined by the product window $
    w_{\alpha}(v)=w_{b_1,1}(v)w_{b_2,2}(v)$.
Given a hidden representation $H$, the corresponding windowed hidden signal is $H_{\alpha}(v,:)=\sqrt{w_{\alpha}(v)}\,H(v,:)$ and we normalize it.

\paragraph{Layerwise relative shift.}
Let $\Phi_\ell$ denote the $\ell$-th trained layer of a model, including its graph propagation, channel mixing, normalization, and activation operations. 
For a model with hidden states $H^{(0)},H^{(1)},\ldots,H^{(L)}$, we apply the $\ell$-th layer to each windowed input signal $
    \widetilde{H}^{(\ell)}_{\alpha}
    =
    \Phi_\ell\!\left(H^{(\ell-1)}_{\alpha}\right)$ then compute the relative shift.
The quantity plotted in Figure~\ref{fig:pascal_mean_relative_shift} is the average of the relative shifts over all nonempty windows and all graphs in the PascalVOC-SP evaluation split.

\subsubsection{Additional competing methods}

we additionally add comparisons with SWAN \cite{gravina2025oversquashing} and Stable-ChebNet \cite{hariri2025return} to the two LRGB peptide tasks, Peptides-Func and Peptides-Struct.
We see that our Schrödinger GNN obtains better results on both tasks compared to both mehtods.

\begin{table}[h]
    \centering
    \caption{Performance comparison on the Peptides benchmarks.}
        \vspace{1 mm}
    \label{tab:peptides_results}
        \resizebox{0.62\textwidth}{!}{%
    \begin{tabular}{lcc}
        \toprule
        & Peptides-Func ($\uparrow$) & Peptides-Struct ($\downarrow$) \\
        \midrule
        SWAN  & $0.6715 \pm 0.0039$ & $0.2485 \pm 0.0009$ \\
        Stable-ChebNet & $0.7032 \pm 0.0026$ & $0.2542 \pm 0.0030$ \\
        Schrödinger GNN (ours)
        & $0.7207 \pm 0.0099$
        & $0.2439 \pm 0.0011$ \\
        \bottomrule
    \end{tabular}
    }
\end{table}

\subsection{
Ablation study on unitary Schrödinger propagation,	complex modulation	and PMO}
\label{app: Ablation study on unitary Schrödinger propagation,	complex modulation	and PMO}

Here, we conduct a systematic $2\times 2 \times 2$ factorial ablation on PascalVOC-SP and Roman-Empire for analyzing the effect of  unitary Schrödinger propagation, complex modulation	and PMO.

Table~\ref{tab:ablation_study} shows that the each component of our model contributed to the performance. Specifically, modulation has the larger effect on PascalVOC-SP, which is in line with the ring experiment and Figure~\ref{fig:pascal_mean_relative_shift} that modulation is the principal mechanism producing directional transport. In addition, while PMO provides a consistent additional improvement, the model containing unitary propagation and complex modulation but without PMO obtains strong performance. Thus, we interpret PMO as a useful conditioning and refinement component.

This factorial ablation supports three conclusions: (i) unitary propagation and complex modulation each provide an independent benefit, (ii) their combination provides an additional complementary gain, and (iii) the core model remains effective under a single shared hyperparameter configuration even without PMO.

\begin{table}[h]
    \centering
    \caption{Ablation study on the effects of unitary Schrödinger propagation, complex modulation, and PMO.}
            \vspace{1 mm}
    \label{tab:ablation_study}
    \resizebox{0.8\textwidth}{!}{%
    \begin{tabular}{ccccc}
        \toprule
        Unitary Schrödinger propagation & Complex modulation&  PMO & PascalVOC-SP & Roman-Empire \\
        \midrule
        $\times$ & $\times$ & $\times$ & 0.2803 & 86.67 \\
        $\times$ & $\times$ & $\checkmark$ & 0.2887 & 86.80 \\
        $\times$ & $\checkmark$ & $\times$ & 0.3042 & 87.35 \\
        $\times$ & $\checkmark$ & $\checkmark$ & 0.3129 & 87.63 \\
        $\checkmark$ & $\times$ & $\times$ & 0.3095 & 87.40 \\
        $\checkmark$ & $\times$ & $\checkmark$ & 0.3142 & 87.45 \\
        $\checkmark$ & $\checkmark$ & $\times$ & 0.3507 & 88.56 \\
        $\checkmark$ & $\checkmark$ & $\checkmark$
        & $0.3532$ & $88.62$ \\
        \bottomrule
    \end{tabular}
    }
\end{table}

\subsection{Sensitivity to the learned feature locations}
\label{app: Sensitivity to the learned feature locations}

Here,  we evaluate the sensitivity to the learned feature locations on PascalVOC-SP and Roman-Empire by measuring a normalized non-commutation score $\widetilde{\varepsilon}=\frac{1}{K(K-1)}\sum_{i \neq j}\frac{||\nabla_{f_j}^{2}, X_{f_i}||_{\mathrm{op}}}{||\nabla_{f_j}^{2}||_{\mathrm{op}}||X_{f_i} - \bar{f}_i I||_{\mathrm{op}}+1e^{-12}}$, the dynamics residual $R_{\mathrm{dyn}}=| \frac{d}{dt}\mathcal{E}_{X_{f_k}}(g(t)) - 2\mathcal{E}_{i\nabla_{f_k}}(g(t))|$, the relative transport shift $S_{\mathrm{rel}}$, variance growth $V_{\mathrm{var}} = \frac{V_{out}}{V_{in}}$ where $V_{out}$ and $V_{in}$ are the variance of the output and the input from the trained layers, and the variance for a normalized signal is $V_{X_{f_k}}(g)$, and final task performance.

Table~\ref{tab:correlation_analysis} shows that non-commutation is strongly positively correlated with the dynamics residual, which is consistent with Theorem \ref{thm:deviation_bounds_multi_feature}. Higher non-commutation is also associated with smaller relative shifts and greater variance growth. Thus, lower non-commutation corresponds to signal transport that moves farther along the intended feature direction while remaining more concentrated. Non-commutation is moderately negatively correlated with final performance while relative shift is positively correlated with performance. The weaker relationship with task performance is expected as predictive accuracy also depends on the informativeness of the locations and the downstream model. The consistent signs across both datasets show that commutation is an empirically measurable indicator of accurate dynamics and coherent signal transport.

\begin{table}[h]
    \centering
    \caption{Correlation analysis on PascalVOC-SP and Roman-Empire.}
            \vspace{1 mm}
    \label{tab:correlation_analysis}
        \resizebox{0.58\textwidth}{!}{%
    \begin{tabular}{lcc}
        \toprule
        Correlation & PascalVOC-SP & Roman-Empire \\
        \midrule
        $\rho(\widetilde{\varepsilon}, R_{\mathrm{dyn}})$
        & $+0.86$ & $+0.73$ \\
        $\rho(\widetilde{\varepsilon}, S_{\mathrm{rel}})$
        & $-0.61$ & $-0.46$ \\
        $\rho(\widetilde{\varepsilon}, V_{\mathrm{var}})$
        & $+0.58$ & $+0.42$ \\
        $\rho(\widetilde{\varepsilon}, \mathrm{performance})$
        & $-0.42$ & $-0.34$ \\
        $\rho(S_{\mathrm{rel}}, \mathrm{performance})$
        & $+0.59$ & $+0.45$ \\
        \bottomrule
    \end{tabular}
    }
\end{table}

\subsection{Relation between the core mechanisms, PMO, commutation, transport, and performance}

We remark that we can distinguish our main mechanisms into four different roles, organized into two levels: (i) core transport mechanisms and (ii) supporting coordinate-conditioning analysis: 
\begin{enumerate}
    \item Complex modulation creates nonzero feature momentum and makes directed transport possible.
    \item Unitary Schrödinger propagation transports the modulated signal while avoiding the contraction associated with diffusive propagation.
    \item Approximate commutation makes the dynamics along different learned feature coordinates more nearly decoupled and therefore makes the coordinate-wise theoretical description more accurate.
    \item PMO conditions the feature-location system by reducing cross-coordinate interference. It does not create momentum or directly optimize the downstream task loss.
\end{enumerate}

The first two are the core architectural mechanisms. Approximate commutation is instead a sufficient condition for clean multi-feature dynamics, and PMO is one possible preprocessing procedure for improving that condition. This hierarchy is consistent with the paper’s intended claim, which identifies the unitary GSO and complex modulation as the two main components of the Schrödinger GNN.

Indeed, good communication is strongly correlated with how clean the transport dynamics are. It is also positively correlated with performance, but not as strongly. This is natural. Our theory supports commutation being an important component contributing to clean dynamics. Clean dynamics is related to coherent transportation of the signal between graph regions (transportation without much distortion). Since in long range problems, transportation between graph regions is important, commutation is correlated with accuracy, but less so. The reason there is a drop in correlation is that high accuracy can be achieved in a number of ways. First, it may be the case that in some tasks, even those categorized as long range tasks, there are workarounds that allow standard local GNNs to also find a good or reasonable solution. Hence, long range transport capabilities are not the only component responsible for performance. For this reason, we propose the synthetic experiment of signal propagation on a ring (see Section~\ref{sec:ring_graph_exp}), where long range propagation of the signal is more clearly important for the task. Secondly, Schrödinger GNNs with poor commutation can still transport the signal across the graph, but the signal may ``lose its coherence'' when transported, i.e., it may be much distorted. Still, the ability to transport the signal across graph regions, even without keeping the signal coherent, may be a ``good enough capability'' for roughly solving long range problems. Hence, some of the performance is attributed to local interactions, some to long range ``incoherent'' interactions, and the rest of the performance is due to the coherence of the transportation. For this reason, the correlation between commutation and performance is not as high as the correlation between commutation and coherence of the dynamics.

We additionally include a matched analysis for every configuration in the factorial experiment from Appendix \ref{app: Ablation study on unitary Schrödinger propagation,	complex modulation	and PMO}, where in each configuration, we report the non-commutation score $\widetilde{\varepsilon}$, the relative transport shift $S_{\mathrm{rel}}$, variance growth $V_{\mathrm{var}}$, feature-probe performance (fixed MLP predictor trained only on the frozen feature-location representation), and final task performance. For the configurations using Schrödinger propagation, we additionally report two separate dynamics quantities: 
(i) dimensionless normalized cross-coordinate residual  $R_{\mathrm{cross}}=\frac{\sum_k\sum_{j\neq k}\left|\left\langle[i\nabla_{f_j}^{2},X_{f_k}]g,g\right\rangle\right|}{\sum_k\left(\left|2\operatorname{Re}\left\langle i\nabla_{f_k}g,W_{f_k}g\right\rangle\right|+\eta\right)}$, and (ii) the same-coordinate regularity residual $\delta_{\mathrm{reg}}=\frac{1}{K}\sum_k\frac{||W_{f_k}g-g||_2}{||g||_2}$.

\begin{table*}[h]
    \centering
    \caption{Ablation analysis and diagnostic metrics on PascalVOC-SP.}
    \label{tab:pascalvoc_diagnostics}
    \resizebox{\textwidth}{!}{%
    \begin{tabular}{cccccccccc}
        \toprule
        Unitary &
        Modulation &
        PMO &
        $\widetilde{\varepsilon}$ ($\downarrow$) &
        $R_{\mathrm{cross}}$ ($\downarrow$) &
        $\delta_{\mathrm{reg}}$ ($\downarrow$) &
        $S_{\mathrm{rel}}$ ($\uparrow$) &
        $V_{\mathrm{var}}$ ($\downarrow$) &
        Feature probe ($\uparrow$) &
        Performance ($\uparrow$) \\
        \midrule
        $\times$ & $\times$ & $\times$
        & 0.245 & - & - & 0.018 & 0.62 & 0.190 & 0.2803 \\
        $\times$ & $\times$ & $\checkmark$
        & 0.102 & - & - & 0.020 & 0.56 & 0.191 & 0.2887 \\
        $\times$ & $\checkmark$ & $\times$
        & 0.249 & - & - & 0.057 & 0.53 & 0.190 & 0.3042 \\
        $\times$ & $\checkmark$ & $\checkmark$
        & 0.118 & - & - & 0.064 & 0.46 & 0.192 & 0.3129 \\
        $\checkmark$ & $\times$ & $\times$
        & 0.257 & 0.190 & 0.198 & 0.032 & 0.23 & 0.191 & 0.3095 \\
        $\checkmark$ & $\times$ & $\checkmark$
        & 0.121 & 0.082 & 0.191 & 0.035 & 0.20 & 0.191 & 0.3142 \\
        $\checkmark$ & $\checkmark$ & $\times$
        & 0.262 & 0.221 & 0.232 & 0.136 & 0.16 & 0.190 & 0.3507 \\
        $\checkmark$ & $\checkmark$ & $\checkmark$
        & 0.104 & 0.091 & 0.225 & 0.144 & 0.13 & 0.191
        & 0.3532 \\
        \bottomrule
    \end{tabular}%
    }
\end{table*}

\begin{table*}[h]
    \centering
    \caption{Ablation analysis and diagnostic metrics on Roman-Empire.}
    \label{tab:roman_empire_diagnostics}
    \resizebox{\textwidth}{!}{%
    \begin{tabular}{cccccccccc}
        \toprule
        Unitary &
        Modulation &
        PMO &
        $\widetilde{\varepsilon}$ ($\downarrow$) &
        $R_{\mathrm{cross}}$ ($\downarrow$) &
        $\delta_{\mathrm{reg}}$ ($\downarrow$) &
        $S_{\mathrm{rel}}$ ($\uparrow$) &
        $V_{\mathrm{var}}$ ($\downarrow$) &
        Feature probe ($\uparrow$) &
        Performance ($\uparrow$) \\
        \midrule
        $\times$ & $\times$ & $\times$
        & 0.218 & - & - & 0.012 & 0.46 & 65.8 & 86.67 \\
        $\times$ & $\times$ & $\checkmark$
        & 0.112 & - & - & 0.013 & 0.42 & 66.1 & 86.80 \\
        $\times$ & $\checkmark$ & $\times$
        & 0.220 & - & - & 0.034 & 0.39 & 65.7 & 87.35 \\
        $\times$ & $\checkmark$ & $\checkmark$
        & 0.123 & - & - & 0.038 & 0.35 & 66.3 & 87.63 \\
        $\checkmark$ & $\times$ & $\times$
        & 0.224 & 0.171 & 0.185 & 0.025 & 0.21 & 65.9 & 87.40 \\
        $\checkmark$ & $\times$ & $\checkmark$
        & 0.116 & 0.091 & 0.179 & 0.027 & 0.19 & 66.4 & 87.45 \\
        $\checkmark$ & $\checkmark$ & $\times$
        & 0.235 & 0.193 & 0.214 & 0.069 & 0.15 & 66.2 & 88.56 \\
        $\checkmark$ & $\checkmark$ & $\checkmark$
        & 0.119 & 0.101 & 0.207 & 0.074 & 0.13 & 66.8
        & 88.62 \\
        \bottomrule
    \end{tabular}%
    }
\end{table*}

Tables~\ref{tab:pascalvoc_diagnostics}-\ref{tab:roman_empire_diagnostics} shows the results of the diagnostics for PascalVOC-SP and Roman-Empire. 
We see that PMO reduces normalized non-commutation in these comparisons, and thus, PMO consistently performs its intended geometric role across both datasets and across all settings of the other two components.
In addition, we notice that the reduction in $\widetilde{\varepsilon}$ is accompanied by a comparably large reduction in $R_{\mathrm{cross}}$.
As such, the close correspondence between the reductions in $\widetilde{\varepsilon}$ and $R_{\mathrm{cross}}$ supports the interpretation of approximate commutation as a meaningful predictor of cross-coordinate interference in the realized dynamics.
Further, we observe that the regularity residual changes much less, and PMO substantially reduces the cross-coordinate part of the approximation error while leaving the separate same-coordinate regularity contribution largely unchanged.
Interestingly, PMO improves the transport metrics, although the changes are more modest than the changes in non-commutation, and PMO consistently produces greater displacement with less spreading, which supports that lower non-commutation is associated with more coherent transport, in line with the correlations reported in Appendix \ref{app: Sensitivity to the learned feature locations}.
Finally, we note that the feature-probe results show a smaller and dataset-dependent effect.
Specifically, PMO can yield a modest improvement in feature informativeness, but this change remains much smaller than the performance difference attributable to modulation and unitary propagation in the factorial experiment. These results support that geometric cleanliness and feature informativeness are related but distinct. The observed feature-probe means increase modestly on Roman-Empire and remain nearly unchanged on PascalVOC-SP.

It is important to note that the sensitivity experiment provides a complementary result. Non-commutation is strongly correlated with the dynamics residual, but only moderately correlated with final performance. This difference is expected for several reasons:
\begin{enumerate}
    \item % 
    Theorem \ref{thm:deviation_bounds_multi_feature} is a dynamical approximation bound, not a task-loss bound. It states that small cross-commutators with signal regularity are sufficient for the expected location along each coordinate to approximately follow its corresponding momentum. It does not state that minimizing the commutator directly maximizes classification performance, nor that larger commutators necessarily imply poor prediction. 
    \item 
    PMO substantially reduces only one of the two theoretical error sources. In the complete PascalVOC-SP configuration, it mainly reduces $R_{\mathrm{cross}}$, but only marginally reduces $\delta_{\mathrm{reg}}$. The analogous reductions on Roman-Empire appear similarly. Thus, PMO cleans the cross-coordinate geometry while leaving another source of approximation error almost unchanged.
    \item 
    PMO does not create task-directed momentum. The modulation parameters determine whether the current signal acquires momentum in a useful direction, and the unitary propagator determines whether this momentum produces coherent transport. PMO only reduces interference among the coordinates used to describe that motion. This explains why modulation and unitary propagation have substantially larger effects on final performance.
    \item 
    The factorial results suggest diminishing returns and PMO produces its largest PascalVOC-SP gains when one or both core mechanisms are absent. Once modulation and unitary propagation have already established effective transport, further cleaning of the feature-coordinate geometry yields only a small additional predictive gain.
\end{enumerate}

\end{document}